%% file: topK.tex
\documentclass[english,10pt]{article}
\usepackage[T1]{fontenc}
\usepackage[latin9]{inputenc}
\usepackage{geometry}
\usepackage{babel}
\usepackage{verbatim}
\usepackage{float}
\usepackage{bm}
\usepackage{amsmath}
\usepackage{amssymb}
\usepackage{graphicx}
\usepackage{breakurl}

\RequirePackage{natbib}
\RequirePackage[colorlinks,citecolor=blue,urlcolor=blue]{hyperref}
\makeatletter

\providecommand{\tabularnewline}{\\}
\floatstyle{ruled}
\newfloat{algorithm}{tbp}{loa}
\providecommand{\algorithmname}{Algorithm}
\floatname{algorithm}{\protect\algorithmname}

\usepackage{amsthm}\usepackage{dsfont}\usepackage{array}\usepackage{mathrsfs}\usepackage{comment}\onecolumn

\usepackage{color}\usepackage{babel}

\usepackage{enumitem}
\setlist[itemize]{leftmargin=1em}
\setlist[enumerate]{leftmargin=1em}
\allowdisplaybreaks

\usepackage{babel}
\usepackage[algo2e]{algorithm2e}
\usepackage{algorithmic}
\usepackage{arydshln}
\usepackage{enumitem}

\newcommand{\by}{\bm{y}}

\newcommand{\bP}{\bm{P}}

\newcommand{\cA}{\mathcal{A}}

\newcommand{\cE}{\mathcal{E}}

\newcommand{\cG}{\mathcal{G}}

\newcommand{\cL}{\mathcal{L}}

\newcommand{\cV}{\mathcal{V}}

\newcommand{\EE}{\mathbb{E}}

\newcommand{\PP}{\mathbb{P}}

\newcommand{\RR}{\mathbb{R}}

\newcommand{\btheta}{\bm{\theta}}

\newcommand{\bpi}{\bm{\pi}}

\DeclareMathOperator{\ind}{\mathds{1}}  % Indicator

\definecolor{yxc}{RGB}{255,0,0}
\definecolor{cm}{RGB}{0,0,200}
\definecolor{kzw}{RGB}{0,150,0}

\makeatother

\begin{document}
\theoremstyle{plain} \newtheorem{lemma}{\textbf{Lemma}} \newtheorem{prop}{\textbf{Proposition}}\newtheorem{theorem}{\textbf{Theorem}}\setcounter{theorem}{0}
\newtheorem{corollary}{\textbf{Corollary}} \newtheorem{assumption}{\textbf{Assumption}}
\newtheorem{example}{\textbf{Example}} \newtheorem{definition}{\textbf{Definition}}
\newtheorem{fact}{\textbf{Fact}} \theoremstyle{definition}

\theoremstyle{remark}\newtheorem{remark}{\textbf{Remark}}

\title{Spectral Method and Regularized MLE Are Both\\  Optimal
for Top-$K$ Ranking\footnotetext{Author names are sorted alphabetically.}}

\author
{
	Yuxin Chen\thanks{Department of Electrical Engineering, Princeton University, Princeton, NJ 08544, USA; Email:
		\texttt{yuxin.chen@princeton.edu}.}
	\qquad Jianqing Fan\thanks{Department of Operations Research and Financial Engineering, Princeton University, Princeton, NJ 08544, USA; Email:
		\texttt{\{jqfan, congm, kaizheng\}@princeton.edu}.}
	\qquad Cong Ma\footnotemark[2]
	\qquad Kaizheng Wang\footnotemark[2]
}

\date{August 2017; \quad Revised July 2018}

\maketitle
\begin{abstract}
This paper is concerned with the problem of top-$K$ ranking from
pairwise comparisons. Given a collection of $n$ items and a few pairwise comparisons across them, one wishes to identify the set of
$K$ items that receive the highest ranks. To tackle this problem,
we adopt the logistic parametric model --- the Bradley-Terry-Luce model,
where each item is assigned a latent preference score, and where the
outcome of each pairwise comparison depends solely on the relative
scores of the two items involved. Recent works have made significant
progress towards characterizing the performance (e.g.~the mean square
error for estimating the scores) of several classical methods, including
the spectral method and the maximum likelihood estimator (MLE). However,
where they stand regarding top-$K$ ranking remains unsettled.

We demonstrate that under a natural random sampling model, the spectral method alone, or the regularized MLE alone, is minimax optimal in terms of the sample complexity --- the
number of paired comparisons needed to ensure exact top-$K$ identification, for the fixed dynamic range regime. This is accomplished via optimal control
of the entrywise error of the score estimates. We complement our theoretical
studies by numerical experiments, confirming that both methods yield low entrywise errors for estimating the underlying scores. Our
theory is established via a novel leave-one-out trick, which
proves effective for analyzing both iterative and non-iterative
procedures. Along the way, we derive an elementary eigenvector perturbation
bound for probability transition matrices, which parallels the
Davis-Kahan $\sin\Theta$ theorem for symmetric matrices. This also allows us to close the gap between the $\ell_2$ error upper bound for the spectral method and the minimax lower limit.

\end{abstract}

\input{intro.tex}

\input{model.tex}
\input{extension.tex}

\input{discussion.tex}
\input{analysis_spectral.tex}

\input{analysis_MLE.tex}

\section*{Acknowledgements}

 Y.~Chen is supported in part by the grant  ARO W911NF-18-1-0303 and by the Princeton SEAS innovation award. J.~Fan is supported in part by NSF grants DMS-1662139 and DMS-1712591 and NIH grant 2R01-GM072611-13.

\appendix
\input{proof-extension.tex}

\input{proof-spectral.tex}

\input{proof-MLE.tex}

\section{Hoeffding's and Bernstein's inequalities}
This section collects two standard concentration inequalities
used throughout the paper, which can be easily found in textbooks such as \cite{boucheron2013concentration}. The proofs are omitted.

\begin{lemma}[Hoeffding's inequality]\label{lemma:hoeffding} Let
$\ensuremath{\{X_{i}\}_{1\leq i\leq n}}$ be a sequence of independent
random variables where $\ensuremath{X_{i}\in[a_{i},b_{i}]}$ for each
$1\leq i\leq n$, and $\ensuremath{S_{n}=\sum_{i=1}^{n}X_{i}}$. Then
\[
\mathbb{P}(|S_{n}-\mathbb{E}\left[S_{n}\right]|\geq t)\leq2e^{-2t^{2}/\sum_{i=1}^{n}(b_{i}-a_{i})^{2}}.
\]

\end{lemma}

The next lemma is about a user-friendly version of the Bernstein inequality.

\begin{lemma}[Bernstein's inequality]\label{lemma:bernstein}Consider
$n$ independent random variables $z_{l}\;\left(1\leq l\leq n\right)$,
each satisfying $\left|z_{l}\right|\leq B$. For any $a\geq2$, one
has
\[
\left|\sum_{l=1}^{n}z_{l}-\EE\left[\sum_{l=1}^{n}z_{l}\right]\right|\leq\sqrt{2a\log n\sum_{l=1}^{n}\EE\left[z_{l}^{2}\right]}+\frac{2a}{3}B\log n
\]
with probability at least $1-2n^{-a}$. \end{lemma}

%\end{proof}

{
\bibliographystyle{ims}
\bibliography{bibfile_rank}
}

\end{document}

%% file: intro.tex
\section{Introduction}

Imagine we have a large collection of $n$ items, and we are given
partially revealed comparisons between pairs of items. These paired
comparisons are collected in a non-adaptive fashion, and could be
highly noisy and incomplete. The aim is to aggregate these partial
preferences so as to identify the $K$ items that receive the highest
ranks. This problem, which is called \emph{top-$K$ rank aggregation},
finds applications in numerous contexts, including web search \citep{Dwork2001},
recommendation systems \citep{baltrunas2010group},
sports competition \citep{masse1997}, to name just a few. The challenge is
both statistical and computational: how can one achieve reliable top-$K$
ranking from a minimal number of pairwise comparisons,
while retaining computational efficiency?

\subsection{Popular approaches}

To address the aforementioned challenge, many prior approaches have
been put forward based on certain statistical models. Arguably one
of the most widely used parametric models is the Bradley-Terry-Luce
(BTL) model \citep{bradley1952rank,luce1959individual},
which assigns a latent preference score $\{w_{i}^{*}\}_{1\leq i\leq n}$
to each of the $n$ items. The BTL model posits that: the chance of
each item winning a paired comparison is determined by the relative
scores of the two items involved, or more precisely,
\begin{equation}
\mathbb{P}\left\{ \text{item }j\text{ is preferred over item }i\right\} =\frac{w_{j}^{*}}{w_{i}^{*}+w_{j}^{*}}\label{eq:BTL-informal}
\end{equation}
in each comparison of item $i$ against item $j$. The items are repeatedly
compared in pairs according to this parametric model. The task then
boils down to identifying the $K$ items with the highest preference
scores, given these pairwise comparisons.

Among the ranking algorithms tailored to the BTL model, 
the following two procedures have received particular attention, both of which
rank the items based on appropriate estimates of the latent preference
scores.

\begin{enumerate}
\item 
%\smallskip
\textbf{The spectral method.} By connecting the winning probability in
(\ref{eq:BTL-informal}) with the transition probability of a reversible
Markov chain, the spectral method attempts recovery of $\left\{ w_{i}^{*}\right\}$
via the leading left eigenvector of a sample transition matrix. This
procedure, also known as Rank Centrality
\citep{negahban2016rank}, bears  similarity to the PageRank algorithm.

%\smallskip\noindent
\item 
		\textbf{The maximum likelihood estimator (MLE).} This approach proceeds
by finding the score assignment that maximizes the likelihood function
\citep{ford1957solution}. When parameterized appropriately, solving
the MLE becomes a convex program, and hence is computationally feasible.
There are also important variants of the MLE that enforce additional
regularization.
%see \cite{negahban2016rank}.
\end{enumerate}

\smallskip
\noindent Details are postponed to Section \ref{sec:Algorithms}. In addition
to their remarkable practical applicability, these two ranking paradigms
are appealing in theory as well. For instance, both of them provably
achieve intriguing {\em $\ell_2$ accuracy} when estimating the latent
preference scores \citep{negahban2016rank}.

Nevertheless, the $\ell_{2}$ error for estimating the latent scores
merely serves as a ``meta-metric'' for the ranking task, which does
not necessarily reveal the accuracy of top-$K$ identification. In fact, given
that the $\ell_{2}$ loss only reflects the estimation error in some
average sense, it is certainly possible that an algorithm obtains
minimal $\ell_{2}$ estimation loss but incurs (relatively) large
errors when estimating the scores of the highest ranked items. Interestingly,
a recent work \cite{chen2015spectral} demonstrates that: a careful
combination of the spectral method and the coordinate-wise MLE is optimal for top-$K$ ranking. This leaves open the following
natural questions: \emph{where does the spectral alone, or the MLE
alone, stand in top-$K$ ranking?} \emph{Are they capable of attaining
exact top-$K$ recovery from minimal samples?}  These questions  form the primary objectives of our study.

As we will elaborate later, the spectral method part of the preceding questions was recently
explored by \citep{jang2016top}, for a regime where a relatively large fraction of item pairs have been compared. However, it remains unclear how well the spectral method can perform in a much broader --- and often much more
challenging --- regime, where the fraction of item pairs being compared may be vanishingly small.
Additionally, the ranking accuracy of the  MLE (and its variants) remains unknown.

\subsection{Main contributions}

The central focal point of the current paper is to assess the accuracy of both the spectral method and the regularized MLE in top-$K$
identification.
%\kzw{Does top-$K$ identification problem sound too specific? Shall we emphasize entrywise error bound instead and treat top-$K$ problem as an application?} Under a suitable random sampling model in which the
Assuming that the pairs of items being compared are randomly selected and that the preference scores fall within a {\em fixed} dynamic range,  our paper delivers a somewhat
surprising message:

{ \setlist{rightmargin=\leftmargin} \begin{itemize} \item[]\emph{Both
the spectral method and the regularized MLE achieve perfect identification
of top-$K$ ranked items under optimal sample complexity (up to some constant factor)! }

 \end{itemize} }

It is worth emphasizing that these two algorithms succeed even under the sparsest possible
regime, a scenario where only an exceedingly small fraction of pairs
of items have been compared. This calls for precise control of the
entrywise error --- as opposed to the $\ell_{2}$ loss --- for estimating
the scores. To this end, our theory is established upon a novel {\em leave-one-out
argument}, which might shed light on how to analyze the entrywise
error for more general optimization problems. 

As a byproduct of the
analysis, we derive an elementary eigenvector perturbation bound
for (asymmetric) probability transition matrices, which parallels Davis-Kahan's
$\sin\Theta$ theorem for symmetric matrices. This simple perturbation bound immediately leads to an improved $\ell_2$ error bound for the spectral method, which allows to close the gap between the theoretical performance of the spectral method and the minimax lower limit.

%Finally, we remark that both methods can be performed in time nearly proportional to reading all samples. In fact, the near-optimal computational cost of the spectral method has already been demonstrated in \cite{negahban2016rank}. When it comes to the regularized MLE, we will show that this method---which only exhibits a fairly weak level of strong
%convexity---can be computed by the gradient descent in nearly linear
%time. In summary, both methods provably achieve optimal efficiency from both  statistical and computational perspectives.

\subsection{Notation\label{subsec:Notation-pi}}

Before proceeding, we introduce a few notations that will be useful
throughout. To begin with, for any strictly positive probability vector
${\bm{\pi}\in\mathbb{R}^{n}}$,
%\footnote{We say that $\bm{\pi}$ is a probability vector if it has non-negative
%entries and satisfies $\bm{1}^{\top}\bm{\pi}=1$. }
we define the inner product space indexed by $\bm{\pi}$ as a vector space in $\RR^{n}$ endowed with the inner product
$\left\langle \bm{x},\bm{y}\right\rangle _{\bm{\pi}}=\sum\nolimits _{i=1}^{n}\pi_{i}x_{i}y_{i}$.
The corresponding vector norm and the induced matrix norm are defined respectively
as
$\|\bm{x}\|_{\bm{\pi}}=\sqrt{\left\langle \bm{x},\bm{x}\right\rangle _{\bm{\pi}}}$ and $\|\bm{A}\|_{\bm{\pi}}=\sup\nolimits_{\|\bm{x}\|_{\bm{\pi}}=1}\|\bm{x}^{\top}\bm{A}\|_{\bm{\pi}}$.

%If $0<\pi_{\min}\leq\pi_{i}\leq\pi_{\max}$ for all $1\leq i\leq n$, then one has the following elementary inequalities
%\begin{align*}
%\sqrt{\pi_{\min}}\|\bm{x}\|_{2}&\leq\|\bm{x}\|_{\bm{\pi}}\leq\sqrt{\pi_{\max}}\|\bm{x}\|_{2}\qquad\text{and}\\
%\qquad\sqrt{\frac{\pi_{\min}}{\pi_{\max}}}\|\bm{A}\|_{2}&\leq\|\bm{A}\|_{\bm{\pi}}\leq\sqrt{\frac{\pi_{\max}}{\pi_{\min}}}\|\bm{A}\|_{2}.
%\end{align*}

Additionally, the  notation $f(n)=O\left(g(n)\right)$ or
$f(n)\lesssim g(n)$ means  there is a constant $c>0$ such
that $\left|f(n)\right|\leq c|g(n)|$, $f(n)=\Omega\left(g(n)\right)$
or $f(n)\gtrsim g(n)$ means  there is a constant $c>0$ such
that $|f(n)|\geq c\left|g(n)\right|$, $f(n)=\Theta\left(g(n)\right)$
or $f(n)\asymp g(n)$ means that there exist constants $c_{1},c_{2}>0$
such that $c_{1}|g(n)|\leq|f(n)|\leq c_{2}|g(n)|$, and $f(n)=o(g(n))$
means  $\lim_{n\rightarrow\infty}\frac{f(n)}{g(n)}=0$.

Given a graph $\mathcal{G}$ with vertex set $\{1,2,\ldots, n\}$ and edge set $\mathcal{E}$, we denote by $\bm{L}_{\mathcal{G}} = \sum_{(i,j)\in\mathcal{E},i>j}(\bm{e}_{i}-\bm{e}_{j})(\bm{e}_{i}-\bm{e}_{j})^{\top}$ the (unnormalized) Laplacian matrix \citep{chung1997spectral} associated with it, where $\left\{\bm{e}_{i}\right\}_{1\leq i\leq n}$ are the standard basis vectors in $\RR^n$. For a matrix $\bm A \in \mathbb{R}^{n\times n}$ with $n$ real eigenvalues, we let $\lambda_1(\bm A) \geq \lambda_2(\bm A) \geq \cdots \geq  \lambda_n(\bm A)$ be the eigenvalues sorted in descending order.

%% file: model.tex
\section{Statistical models and main results}

\subsection{Problem setup\label{sec:Problem-formulation}}

We begin with a formal introduction of the Bradley-Terry-Luce parametric
model for binary comparisons.

%\begin{list}{}%
%{\leftmargin=2em \itemindent=-2em}

%\item 
\medskip\noindent	
\textbf{Preference scores}. As introduced earlier, we assume
the existence of a positive latent score vector
\begin{equation}
\bm{w}^{*}=[w_{1}^{*},\cdots,w_{n}^{*}]^{\top}\label{eq:defn-w}
\end{equation}
that comprises the underlying preference scores $\left\{ w_{i}^{*}>0\right\} _{1\leq i\leq n}$
assigned to each of the $n$ items. Alternatively, it is sometimes
more convenient to reparameterize the score vector by
\begin{equation}
\bm{\theta}^{*}=[\theta_{1}^{*},\cdots,\theta_{n}^{*}]^{\top},\qquad\text{where}\quad\theta_{i}^{*}=\log w_{i}^{*}.\label{eq:defn-theta}
\end{equation}
These scores are assumed to fall within a \emph{dynamic range} given by
\begin{equation}
w_{i}^{*}\in[w_{\min},w_{\max}],\qquad\text{or}\qquad\theta_{i}^{*}\in[\theta_{\min},\theta_{\max}]
\end{equation}
for all $1\leq i\leq n$ and for some $w_{\min}>0,$ $w_{\max}>0$, $\theta_{\min}=\log w_{\min}$,
	and $\theta_{\max}=\log w_{\max}$. We also introduce the {\em condition number} as
\begin{equation}
\kappa:= \frac{w_{\max}} {w_{\min}}.\label{eq:defn-kappa}
\end{equation}
Notably, the current paper primarily focuses on the case with a {\em fixed} dynamic range (i.e.~$\kappa$ is a fixed constant independent of $n$), although we will also discuss extensions to the large dynamic range regime in Section~\ref{sec:Extension}. Without loss of generality, it is assumed that
\begin{equation}
w_{\max}\geq w_{1}^{*}\geq w_{2}^{*}\geq\ldots\geq w_{n}^{*}\geq w_{\min},\label{eq:w-rank}
\end{equation}
meaning that items $1$ through $K$ are the desired top-$K$ ranked
items. %which is assumed to be a constant independent of $n$.

\medskip
\noindent\textbf{Comparison graph.} Let $\cG=(\cV,\cE)$ stand for a
comparison graph, where the vertex set $\mathcal{V}=\left\{ 1,2,\ldots,n\right\} $
represents the $n$ items of interest. The items $i$ and $j$ are
compared if and only if $(i,j)$ falls within the edge set $\cE$.
Unless otherwise noted, we assume that $\cG$ is drawn from the Erd\H{o}s\textendash R\'enyi
random graph $\cG_{n,p}$, such that an
edge between any pair of vertices is present independently with some
probability $p$. In words, $p$ captures the fraction of item pairs
being compared.
%In order to ensure that the comparison graph is connected (otherwise there is no basis to rank items belonging to two disconnected components of $\cG$), we assume throughout that
%\begin{equation}
%p\geq\frac{c_{0}\log n}{n}
%\end{equation}
%for some constant $c_{0}>1$.

\medskip
\noindent\textbf{Pairwise comparisons.} For each $(i,j)\in\cE$, we
obtain $L$ independent paired comparisons between items $i$ and
$j$. Let $y_{i,j}^{(l)}$ be the outcome of the $\ell$-th comparison,
which is independently drawn as
\begin{equation}
y_{i,j}^{(l)}\text{ }\overset{\text{ind.}}{=}\text{ }\begin{cases}
1,\quad & \text{with probability }\frac{w_{j}^{*}}{w_{i}^{*}+w_{j}^{*}}=\frac{e^{\theta_{j}^{*}}}{e^{\theta_{i}^{*}}+e^{\theta_{j}^{*}}},\\
0, & \text{else}.
\end{cases}\label{eq:pairwise-comparison-model}
\end{equation}
By convention, we set $y_{i,j}^{(l)}=1-y_{j,i}^{(l)}$ for all $(i,j)\in\mathcal{E}$
throughout the paper. This is also known as the \emph{logistic }pairwise
comparison model, due to its strong resemblance to logistic regression. It is self-evident that the sufficient
statistics under this model are given by
\begin{equation}
\bm{y}:=\left\{ y_{i,j}\mid(i,j)\in\mathcal{E}\right\} ,\qquad\text{where}\quad y_{i,j}:=\frac{1}{L}\sum\nolimits_{l=1}^{L}y_{i,j}^{(l)}.\label{eq:sufficient-stats}
\end{equation}
To simplify the notation, we shall also take
%\begin{equation}
	$$y_{i,j}^{*}:=\frac{w_{j}^{*}}{w_{i}^{*}+w_{j}^{*}}=\frac{e^{\theta_{j}^{*}}}{e^{\theta_{i}^{*}}+e^{\theta_{j}^{*}}}.$$
%\label{eq:defn-y-star}
%\end{equation}

\medskip
\noindent\textbf{Goal. }The goal is to identify the \emph{set} of top-$K$
ranked items --- that is, the \emph{set} of $K$ items that enjoy
the largest preference scores --- from the pairwise comparison data $\bm{y}$.

%\end{list}

\subsection{Algorithms\label{sec:Algorithms}}

\subsubsection{The spectral method: Rank Centrality}

The spectral ranking algorithm, or Rank Centrality \citep{negahban2016rank}, is motivated by
the connection between the pairwise comparisons and a random walk
over a directed graph. The algorithm starts by converting the pairwise
comparison data $\bm{y}$ into a transition matrix $\bm{P}=[P_{i,j}]_{1\leq i,j\leq n}$
in such a way that
\begin{equation} \label{eq:transition-matrix}
P_{i,j}=\begin{cases}
\frac{1}{d}y_{i,j}, & \text{if }(i,j)\in\cE,\\
1-\frac{1}{d}\sum_{k:(i,k)\in\cE}y_{i,k},\qquad & \text{if }i=j,\\
0, & \text{otherwise},
\end{cases}
\end{equation}
for some given normalization factor $d>0$, and then proceeds by computing the stationary distribution $\bm{\pi}\in\mathbb{R}^{n}$
of the Markov chain induced by $\bm{P}$. As we shall see later, the
parameter $d$ is taken to be on the same order of the maximum vertex degree of $\mathcal{G}$
while ensuring the non-negativity of $\bm{P}$. As asserted by \cite{negahban2016rank}, $\bm{\pi}$ is a faithful estimate of $\bm{w}^{*}$ up to some global scaling.
The algorithm is summarized in Algorithm \ref{alg:spectral}.

To develop some intuition regarding why this spectral algorithm gives
a reasonable estimate of $\bm{w}^{*}$, it is perhaps more convenient
to look at the population transition matrix $\ensuremath{\bm{P}^{*}=[P_{i,j}^{*}]}_{1\leq i,j\leq n}$:

\[
P_{i,j}^{*}=\begin{cases}
\frac{1}{d}\frac{w_{j}^{*}}{w_{i}^{*}+w_{j}^{*}}, & \quad\text{if }(i,j)\in\cE,\\
1-\frac{1}{d}\sum_{k:(i,k)\in\cE}\frac{w_{k}^{*}}{w_{i}^{*}+w_{k}^{*}}, & \quad\text{if }i=j,\\
0, & \quad\text{otherwise},
\end{cases}
\]
which coincides with $\bm{P}$ by taking $L\rightarrow\infty$. It
can be seen that the normalized score vector
\begin{equation}
\bm{\pi}^{*}:=\frac{1}{\sum_{i=1}^{n}w_{i}^{*}}\left[w_{1}^{*},w_{2}^{*},\ldots,w_{n}^{*}\right]^{\top}\label{eq:defn-pi-star}
\end{equation}
is the stationary distribution of the Markov chain induced by the transition matrix $\bm{P}^{*}$,
since $\bm{P}^{*}$ and $\bm{\pi}^{*}$ are in detailed balance, namely,
\begin{equation}
\pi_{i}^{*}P_{i,j}^{*}=\pi_{j}^{*}P_{j,i}^{*},\qquad\forall\left(i,j\right).\label{eq:detailed-balance}
\end{equation}
As a result, one expects the stationary distribution of the sample
version $\bm{P}$ to form a good estimate of $\bm{w}^{*}$,
provided the sample size is sufficiently large.
% The interested readers are referred to \cite{negahban2016rank} for more detailed
%interpretations.

\begin{algorithm}[t]
\caption{Spectral method (Rank Centrality).}
\label{alg:spectral} \begin{algorithmic}

\STATE \textbf{Input} the comparison graph $\cG$, sufficient
statistics $\by$, and the normalization factor $d$.

\STATE \textbf{Define} the probability transition matrix $\bm{P}=[P_{i,j}]_{1\leq i,j\leq n}$
as in \eqref{eq:transition-matrix}.
%\begin{equation}
%P_{i,j}=\begin{cases}
%\frac{1}{d}y_{i,j}, & \text{if }(i,j)\in\cE,\\
%1-\frac{1}{d}\sum_{k:(i,k)\in\cE}y_{i,k},\qquad & \text{if }i=j,\\
%0, & \text{otherwise}.
%\end{cases}
%\end{equation}
 \STATE \textbf{Compute} the leading left eigenvector $\bm{\pi}$
of $\bP$.

\STATE \textbf{Output} the $K$ items that correspond to the $K$
largest entries of $\bpi$. \end{algorithmic}
\end{algorithm}

\subsubsection{The regularized MLE}

Under the BTL model, the negative log-likelihood function conditioned on $\cG$ is given by (up to some global scaling)
\begin{align}
\mathcal{L}\left(\bm{\theta};\bm{y}\right) :&=-\sum_{(i,j)\in\mathcal{E},i>j}\left\{ y_{j,i}\log\frac{e^{\theta_{i}}}{e^{\theta_{i}}+e^{\theta_{j}}}+\left(1-y_{j,i}\right)\log\frac{e^{\theta_{j}}}{e^{\theta_{i}}+e^{\theta_{j}}}\right\} \nonumber \\
 &=\hphantom{-}\sum_{(i,j)\in\mathcal{E},i>j}\left\{ -y_{j,i}\left(\theta_{i}-\theta_{j}\right)+\log\big(1+e^{\theta_{i}-\theta_{j}}\big)\right\} .\label{eq:MLE}
\end{align}
The regularized MLE 
then amounts to solving the following convex program
\begin{align}
\text{minimize}_{\bm{\theta}\in\mathbb{R}^{n}}\quad & \mathcal{L}_{\lambda}\left(\bm{\theta};\bm{y}\right):=\mathcal{L}\left(\bm{\theta};\bm{y}\right)+\frac{1}{2}\lambda\|\bm{\theta}\|_{2}^{2},\label{eq:regularized-MLE}
\end{align}
for a  regularization parameter $\lambda>0$. As will be discussed
later, we shall adopt the choice $\lambda\asymp\sqrt{\frac{np\log n}{L}}$
throughout this paper. For the sake of brevity, we let $\bm{\theta}$
represent the resulting penalized maximum likelihood estimate whenever it is
clear from the context. Similar to the spectral method, one reports the $K$ items associated with the $K$ largest entries of $\bm{\theta}$.

\subsection{Main results\label{subsec:Main-contributions}}

The most challenging part of top-$K$ ranking is to distinguish
the $K$-th and the $(K+1)$-th items. In fact, the score difference
of these two items captures the distance between the item sets $\left\{ 1,\ldots,K\right\} $
and $\left\{ K+1,\ldots,n\right\} $. Unless their latent scores are
sufficiently separated, the finite-sample nature of the model would
make it infeasible to distinguish these two critical items. With this
consideration in mind, we define the following separation measure
\begin{equation}
\Delta_{K}:=\frac{w_{K}^{*}-w_{K+1}^{*}}{w_{\max}}.\label{eq:defn-separation-score}
\end{equation}
 This metric turns out to play a crucial role in determining the
minimal sample complexity for perfect top-$K$ identification.

The main finding of this paper concerns the optimality of both the
spectral method and the regularized MLE in the presence of a fixed dynamic range (i.e.~$\kappa = O(1)$). Recall that under the BTL
model, the total number $N$ of samples we collect
concentrates sharply around its mean, namely,
\begin{equation}
N=(1+o(1))\,\mathbb{E}\left[N\right]~=(1+o(1))\,n^{2}pL/2\label{eq:sample-size}
\end{equation}
occurs with high probability. Our main result is stated in terms of
the sample complexity required for exact top-$K$ identification.

\begin{theorem}\label{thm:main-samples}Consider the pairwise comparison
model specified in Section \ref{sec:Problem-formulation} with $\kappa=O(1)$. Suppose
that $p>\frac{c_{0}\log n}{n}$ and that
\begin{equation}
\frac{n^{2}pL}{2}\geq\frac{c_{1}n\log n}{\Delta_{K}^{2}}\label{eq:sample-complexity-spectral-MLE}
\end{equation}
for some sufficiently large positive constants $c_{0}$ and $c_{1}$. Further assume $L\leq c_{2}\cdot n^{c_{3}}$ for any absolute constants $c_{2}, c_{3}>0$. With probability
exceeding $1-O(n^{-5})$, the set of top-$K$ ranked items can be
recovered exactly by the spectral method given in Algorithm
\ref{alg:spectral}, and by the regularized MLE given in \eqref{eq:regularized-MLE}.
Here, we take $d=c_{d}np$ in the spectral method and $\lambda=c_{\lambda}\sqrt{\frac{np\log n}{L}}$
in the regularized MLE, where $c_{d}\geq2$ and $c_{\lambda}>0$ are
some absolute constants. \end{theorem}

\begin{remark}
We emphasize that $p \geq \frac{c_0 \log n} { n }$ for $c_0 \geq 1$ is a fundamental requirement for the ranking task. In fact, if $p < (1-\epsilon) \frac{\log n}{n}$ for any constant $\epsilon>0$, then the comparison graph $\mathcal{G}\sim \mathcal{G}_{n,p}$ is disconnected with high probability. This means that there exists at least one isolated item (which has not been compared with any other item) and cannot be ranked.
\end{remark}

\begin{remark}
In fact, the assumption that $L\leq c_{2}\cdot n^{c_{3}}$ for any absolute constants $c_{2}, c_{3}>0$ is not needed for the spectral method.
\end{remark}

\begin{remark}
Here, we assume the same number of comparisons $L$ to simplify the presentation as well as the proof. The result still holds true if we have distinct $L_{i,j}$'s for each $i\neq j$, as long as $n^2 p \min_{i\neq j} L_{i,j} \gtrsim \frac{n\log n}{\Delta_{K}^2 }$.
\end{remark}

%\begin{remark} Theorem \ref{thm:main-samples} deals with a fixed dynamic range (i.e.~$\kappa=O(1)$). Ranking in this case is considered to be a challenging regime \citep{negahban2016rank}. When $\kappa$ is larger, one usually first applies methods like Borda count \citep{ammar2012efficient} to filter out items with extremely large\,/\,small scores, followed by the spectral method or the regularized MLE. \end{remark}

Theorem \ref{thm:main-samples} asserts that both the spectral method
and the regularized MLE achieve a sample complexity on the order of
$\frac{n\log n}{\Delta_{{\it K}}^{2}}$. Encouragingly, this sample
complexity coincides with the minimax limit identified in \cite[Theorem 2]{chen2015spectral}  in the fixed dynamic range, i.e.~$\kappa =O(1)$.

\begin{theorem}[\cite{chen2015spectral}]\label{thm:lower-bound}Fix $\epsilon\in(0,\frac{1}{2})$, and suppose that
\begin{equation}
   n^{2}pL \leq 2 c_{2}\frac{(1-\epsilon)n\log n-2}{\Delta_{K}^{2}},
\end{equation}
where $c_{2}=w_{\min}^{4}/(4w_{\max}^{4})$.
Then for any ranking procedure $\psi$, one can find a score vector
$\bm{w}^*$ with separation $\Delta_{K}$ such that $\psi$ fails to
retrieve the top-$K$ items with probability at least $\epsilon$.
\end{theorem}

%Moving beyond the sample complexity issue, we remark that the spectral
%method can be carried out by means of the power method, while the
%regularized MLE can be computed via the standard gradient descent
%(as we will elaborate in Section \ref{subsec:coarse-analysis}). Both
%procedures are computationally feasible and can be accomplished within
%nearly linear time---in time proportional to reading all data. Consequently,
%these two algorithms are optimal in terms of both computational complexity
%and sample complexity. We emphasize that it is nontrivial to establish the nearly linear-time computational complexity of the regularized MLE, given that the log-likelihood function (or the logistic link function) is not strongly convex.  We will elaborate on this point in Section \ref{subsec:coarse-analysis}. {\color{red} {Check this section.}}

We are now positioned to compare our results with \cite{jang2016top},
which also investigates the accuracy of the spectral method for
top-$K$ ranking. Specifically, Theorem 3 in \cite{jang2016top} establishes
the optimality of the spectral method for the relatively dense regime
where
$$p\gtrsim\sqrt{\frac{\log n}{n}}.$$
In this regime, however, the total sample size necessarily exceeds
\begin{equation}
	{n^{2}pL}/2 ~\geq~ {n^{2}p}/2 ~\gtrsim~\sqrt{n^{3}\log n},
	\label{eq:dense-regime}
\end{equation}
which rules out the possibility of achieving minimal sample complexity
if $\Delta_{K}$ is sufficiently large. For instance, consider the
case where $\Delta_{K}\asymp1$, then the optimal sample size --- as
revealed by Theorem \ref{thm:main-samples} or \cite[Theorem 1]{chen2015spectral} --- is
on the order of
\[
	(n\log n) \,/\,{\Delta_{K}^{2}}\asymp n\log n,
\]
which is a factor of $\sqrt{\frac{ n}{\log n}}$ lower than the bound in (\ref{eq:dense-regime}).
By contrast, our results hold all the way down to the sparsest possible
regime where $p\asymp \frac{\log n}{n}$, confirming the optimality
of the spectral method even for the most challenging scenario. Furthermore,
we establish that the regularized MLE shares the same optimality guarantee
as the spectral method, which was previously out of reach.

\subsection{Optimal control of entrywise estimation errors}\label{sec:optimal-Linfty}

In order to establish the ranking accuracy as asserted by Theorem
\ref{thm:main-samples}, the key is to obtain precise control of the
$\ell_{\infty}$ loss of the score estimates. Our results are as follows.

\begin{theorem}[{\bf Entrywise error of the spectral method}]\label{thm:spectral-loss-infty}Consider the pairwise comparison
model in Section \ref{sec:Problem-formulation} with $\kappa=O(1)$. Suppose  $p>\frac{c_{0}\log n}{n}$
for some sufficiently large constant $c_{0}>0$. Choose $d=c_{d}np$
for some constant $c_{d}\geq2$ in Algorithm \ref{alg:spectral}.
Then the spectral estimate $\bm{\pi}$ satisfies
\begin{equation}
\frac{\|\bm{\pi}-\bm{\pi}^{*}\|_{\infty}}{\|\bm{\pi^{*}}\|_{\infty}}\lesssim\sqrt{\frac{\log n}{npL}}\label{eq:spectral-loss-inf}
\end{equation}
with probability $1-O(n^{-5})$, where $\bm{\pi}^{*}$ is
	the normalized score vector (cf.~(\ref{eq:defn-pi-star})).
\end{theorem}

%The regularized MLE achieves similar $\ell_{\infty}$ estimation loss as the spectral method, as stated below.

%\begin{theorem}[{\bf Entrywise error of the regularized MLE}]\label{thm:MLE-loss-inf}Consider the pairwise comparison
%model specified in Section \ref{sec:Problem-formulation} with $\kappa=O(1)$. Suppose that $p>\frac{c_{0}\log n}{n}$
%for some sufficiently large constant $c_{0}>0$, and that the regularization
%parameter is $\lambda=c_{\lambda}\sqrt{\frac{np\log n}{L}}$ for some
%absolute constant $c_{\lambda} > 0$. Then the regularized MLE $\bm{\theta}$
%satisfies
%\begin{equation}
%\frac{\|e^{\bm{\theta}}-e^{\bm{\theta}^{*}-\overline{\theta}^{*}\bm{1}}\|_{\infty}}{\Vert e^{\bm{\theta^*}-\overline{\theta}^{*}\bm{1}}\Vert_{\infty}}\lesssim\sqrt{\frac{\log n}{npL}}\label{eq:MLE-loss-inf}
%\end{equation}
%	with probability exceeding $1-O(n^{-5})$, where $\overline{\theta}^{*}:=\frac{1}{n}\bm{1}^{\top}\bm{\theta}^{*}$ and $e^{\bm{\theta}}:=[e^{\theta_1},\cdots,e^{\theta_n}]^{\top}$.
%\end{theorem}

\begin{theorem}[{\bf Entrywise error of the regularized MLE}]\label{thm:MLE-main}Consider the pairwise comparison
model specified in Section \ref{sec:Problem-formulation} with $\kappa=O\left(1\right)$.
Suppose that $p\geq\frac{c_{0}\log n}{n}$ for some sufficiently large
constant $c_{0}>0$ and that $L\leq c_{2}\cdot n^{c_{3}}$ for any absolute constants $c_{2}, c_{3}>0$. Set the regularization
parameter to be $\lambda=c_{\lambda}\sqrt{\frac{np\log n}{L}}$ for
some absolute constant $c_{\lambda}>0$. Then the regularized MLE
$\bm{\theta}$ satisfies
\[
\frac{\big\Vert e^{\bm{\theta}}-e^{\bm{\theta}^{*}-\overline{\theta}^{*}\bm{1}}\big\Vert _{\infty}}{\big\Vert e^{\bm{\theta}^{*}-\overline{\theta}^{*}\bm{1}}\big\Vert _{\infty}}\lesssim\sqrt{\frac{\log n}{npL}}
\]
with probability exceeding $1-O(n^{-5})$, where $\overline{\theta}^{*}:=\frac{1}{n}\bm{1}^{\top}\bm{\theta}^{*}$
and $e^{\bm{\theta}}:=[e^{\theta_{1}},\cdots,e^{\theta_{n}}]^{\top}$.

\end{theorem}

Theorems \ref{thm:spectral-loss-infty}--\ref{thm:MLE-main} indicate that if the number of comparisons associated with each item --- which concentrates around $npL$ --- exceeds the order of $\log n$, then both methods are able to achieve a small $\ell_\infty$ error when estimating the scores.

Recall that the $\ell_{2}$ estimation error of the spectral method
has been characterized by \cite{negahban2016rank} (or Theorem \ref{thm:L2-pi}
of this paper that improves it by removing the logarithmic factor), which obeys
\begin{equation}
\frac{\|\bm{\pi}-\bm{\pi}^{*}\|_{2}}{\|\bm{\pi^{*}}\|_{2}}\lesssim\sqrt{\frac{\log n}{npL}}\label{eq:spectral-loss-L2}
\end{equation}
with high probability. Similar theoretical guarantees have been derived
for another variant of the MLE (the constrained version) under a uniform sampling model as well \citep{negahban2016rank}.
In comparison, our results indicate that the estimation errors for
both algorithms are almost evenly spread out across all coordinates
rather than being localized or clustered. Notably, the pointwise errors
revealed by Theorems \ref{thm:spectral-loss-infty}-\ref{thm:MLE-main}
immediately lead to exact top-$K$ identification as claimed by Theorem
\ref{thm:main-samples}.

\begin{proof}[\textbf{Proof of Theorem \ref{thm:main-samples}}]In
what follows, we prove the theorem for the spectral method part. The
regularized MLE part follows from an almost identical argument and hence is omitted.

Since the spectral algorithm ranks the items in accordance with the
score estimate $\bm{\pi}$, it suffices to demonstrate that
\[
\pi_{i}-\pi_{j}>0,\qquad\forall1\leq i\leq K,\text{ }K+1\leq j\leq n.
\]
To this end, we first apply the triangle inequality to get
\begin{align}
\frac{\pi_{i}-\pi_{j}}{\|\bm{\pi}^{*}\|_{\infty}} & \geq\frac{\pi_{i}^{*}-\pi_{j}^{*}}{\|\bm{\pi}^{*}\|_{\infty}}-\frac{|\pi_{i}-\pi_{i}^{*}|}{\|\bm{\pi}^{*}\|_{\infty}}-\frac{|\pi_{j}-\pi_{j}^{*} |}{\|\bm{\pi}^{*}\|_{\infty}}\geq\Delta_{K}-\frac{2\|\bm{\pi}-\bm{\pi}^{*}\|_{\infty}}{\|\bm{\pi}^{*}\|_{\infty}}.\label{eq:pi-ij-sep}
\end{align}
In addition, it follows from Theorem \ref{thm:spectral-loss-infty}
as well as our sample complexity assumption that
\[
\frac{\|\bm{\pi}-\bm{\pi}^{*}\|_{\infty}}{\|\bm{\pi^{*}}\|_{\infty}}\lesssim\sqrt{\frac{\log n}{npL}}\qquad\text{and}\qquad n^{2}pL\gtrsim\frac{n\log n}{\Delta_{K}^{2}}.
\]
These conditions taken collectively imply that
$\frac{\|\bm{\pi}-\bm{\pi}^{*}\|_{\infty}}{\|\bm{\pi}^{*}\|_{\infty}}<\frac{1}{2}\Delta_{K}$ 
as long as $\frac{npL\Delta_{K}^{2}}{\log n}$ exceeds some sufficiently large constant.
Substitution into (\ref{eq:pi-ij-sep}) reveals that $\pi_{i}-\pi_{j}>0$,
as claimed. \end{proof}

\subsection{Heuristic arguments}

We pause to develop some heuristic explanation as to why the
 estimation errors are expected to be spread out across all entries.
For simplicity, we focus on the case where $p=1$ and $L$ is sufficiently
large, so that $\bm{y}$ and $\bm{P}$ sharply concentrate around
$\bm{y}^{*}$ and $\bm{P}^{*}$, respectively.

We begin with the spectral algorithm. Since $\bm{\pi}$ and $\bm{\pi}^{*}$
are respectively the invariant distributions of the Markov chains
induced by $\bm{P}$ and $\bm{P}^{*}$, we can decompose
\begin{align}
\left(\bm{\pi}-\bm{\pi}^{*}\right)^{\top} & =\bm{\pi}^{\top}\bm{P}-\bm{\pi}^{*\top}\bm{P}^{*}=\left(\bm{\pi}-\bm{\pi}^{*}\right)^{\top}\bm{P}+\underset{:=\bm{\xi}}{\underbrace{\bm{\pi}^{*\top}\left(\bm{P}-\bm{P}^{*}\right)}}.\label{eq:pi-1}
\end{align}
When $p=1$ and $\frac{w_{\max}}{w_{\min}}\asymp1$, the entries of
$\bm{\pi}^{*}$ (resp.~the off-diagonal entries of $\bm{P}^{*}$
and $\bm{P}-\bm{P}^{*}$) are all of the same order and, as a result,
the energy of the uncertainty term $\bm{\xi}$ is spread out (using standard concentration inequalities). In fact,
we will demonstrate in Section \ref{subsec:Proof-outline-of-Theorem-Spectral} that
\begin{align}
\frac{\left\Vert \bm{\xi}\right\Vert _{\infty}}{\|\bm{\pi}^{*}\|_{\infty}}\lesssim\sqrt{\frac{\log n}{npL}}\asymp
	\frac{\|\bm{\pi}-\bm{\pi}^{*}\|_{2}\sqrt{\log n}}{\|\bm{\pi}^{*}\|_{2}}\label{eq:spread_out},
\end{align}
which coincides with the optimal rate. Further, if we look at each
entry of (\ref{eq:pi-1}), then for all $1\leq m\leq n$,
\begin{align}
&\pi_{m}-\pi_{m}^{*}  =\big[\left(\bm{\pi}-\bm{\pi}^{*}\right)^{\top}\bm{P}\big]_{m}+\xi_{m},\nonumber \\
 & =\underset{\text{contraction}}{\underbrace{\vphantom{\left[\begin{array}{c}
\pi_{1}-\pi_{1}^{*}\\
\vdots\\
\pi_{n}-\pi_{n}^{*}
\end{array}\right]}\left(\pi_{m}-\pi_{m}^{*}\right)P_{m,m}}}+\underset{\text{error averaging}}{\underbrace{\left[P_{1,m},\cdots,P_{m-1,m},0,P_{m+1,m},\cdots,P_{n,m}\right]\left[\begin{array}{c}
\pi_{1}-\pi_{1}^{*}\\
\vdots\\
\pi_{n}-\pi_{n}^{*}
\end{array}\right]}}+\xi_{m}.\label{eq:pi-2}
\end{align}
By construction of the transition matrix, one can easily verify that
$P_{m,m}$ is bounded away from $1$ and $P_{j,m}\asymp \frac{1}{n}$ for all $j\neq m$. As a consequence, the
identity $\bm{\pi}^\top\bm{P}=\bm{\pi}^\top$ allows one to treat each $\pi_{m}-\pi_{m}^{*}$
as a mixture of three effects: (i) the first term of (\ref{eq:pi-2}) behaves as an entrywise
contraction of the error; (ii) the second term of (\ref{eq:pi-2})
is a (nearly uniformly weighted) average of the errors over all coordinates,
which can essentially be treated as a \emph{smoothing} operator applied
to the error components; and (iii) the uncertainty term
$\xi_{m}$. Rearranging terms in (\ref{eq:pi-2}), we are
left with
\begin{align}
\left(1-P_{m,m}\right)\left|\pi_{m}-\pi_{m}^{*}\right| & \lesssim\frac{1}{n}\sum\nolimits _{i=1}^{n}\left|\pi_{i}-\pi_{i}^{*}\right|+\xi_{m},\quad \forall m
	\label{eq:pi-2-1}
\end{align}
which further gives,
\begin{align}
\left\Vert \bm{\pi}-\bm{\pi}^{*}\right\Vert _{\infty} & \lesssim\frac{1}{n}\sum\nolimits _{i=1}^{n}\left|\pi_{i}-\pi_{i}^{*}\right|+\Vert\bm{\xi}\Vert_{\infty}.\label{eq:pi-2-1-1}
\end{align}
There are two possibilities compatible with this bound \eqref{eq:pi-2-1-1}: (1) $\left\Vert \bm{\pi}-\bm{\pi}^{*}\right\Vert _{\infty}\lesssim\frac{1}{n}\sum_{i=1}^{n}\left|\pi_{i}-\pi_{i}^{*}\right|$,
and (2) $\left\Vert \bm{\pi}-\bm{\pi}^{*}\right\Vert _{\infty}\lesssim\|\bm{\xi}\|_{\infty}\lesssim\frac{\|\bm{\pi}-\bm{\pi}^{*}\|_{2}}{\|\bm{\pi}^{*}\|_{2}}\left\Vert \bm{\pi}^{*}\right\Vert _{\infty}$ by \eqref{eq:spread_out}.
In either case, the errors are fairly delocalized, revealing that
\[
	\frac{\left\Vert \bm{\pi}-\bm{\pi}^{*}\right\Vert _{\infty}}{\left\Vert \bm{\pi}^{*}\right\Vert _{\infty}}\lesssim \left\{ \frac{1}{n} \frac{ \|\bm{\pi} - \bm{\pi}^* \|_1 }{ \left\Vert \bm{\pi}^{*}\right\Vert _{\infty} } ,~  \frac{\|\bm{\pi}-\bm{\pi}^{*}\|_{2}\sqrt{\log n}}{\|\bm{\pi}^{*}\|_{2}} \right\}.
\]

We now move on to the regularized MLE, following a very similar
argument. By the optimality condition that $\nabla\cL_{\lambda}\left(\bm{\theta}\right)=\bm{0}$,
one can derive (for some $\eta$ to be specified later)
\begin{align*}
\bm{\theta}-\bm{\theta}^{*} & =\bm{\theta}-\eta\nabla\cL_{\lambda}\left(\bm{\theta}\right)-\bm{\theta}^{*}\\
 & =\bm{\theta}-\eta\nabla\cL_{\lambda}\left(\bm{\theta}\right)-\left(\bm{\theta}^{*}-\eta\nabla\cL_{\lambda}\left(\bm{\theta}^{*}\right)\right)-\underset{:=\bm{\zeta}}{\underbrace{\eta\nabla\cL_{\lambda}\left(\bm{\theta}^{*}\right)}}\\
 &   \approx \left(\bm{I}-\eta\nabla^{2}\cL_{\lambda}\big(\bm{\theta}^*\big)\right)\left(\bm{\theta}-\bm{\theta}^{*}\right)-\bm{\zeta}.
\end{align*}
Write $\nabla^{2}\cL_{\lambda}\left(\bm{\theta}^* \right)=\bm{D}-\bm{A}$,
where $\bm{D}$ and $\bm{A}$ denote respectively the diagonal and
off-diagonal parts of $\nabla^{2}\cL_{\lambda}\left(\bm{\theta}^* \right)$. Under our assumptions, one can check that $D_{m,m}\asymp n$
for all $1\leq m\leq n$ and $A_{j,m}\asymp1$ for any $j\neq m$.
With these notations in place, one can write the entrywise error as
follows
\[
\theta_{m}-\theta_{m}^{*}=\left(1-\eta D_{m,m}\right)\left(\theta_{m}-\theta_{m}^{*}\right)+\sum_{j:j\neq m}\eta A_{j,m}\left(\theta_{j}-\theta_{j}^{*}\right)-\zeta_{m}.
\]
By choosing $\eta=c_{2}/n$ for some sufficiently small constant $c_{2}>0$
, we get $1-\eta D_{m,m}<1$ and $\eta A_{j,m}\asymp1/n$. Therefore,
the right-hand side of the above relation also comprises a contraction
term as well as an error smoothing term, similar to (\ref{eq:pi-2}).
Carrying out the same argument as for the spectral method, we see
that the estimation errors of the regularized MLE are expected to be spread out.

\subsection{Numerical experiments}

It is worth noting that extensive numerical experiments on both synthetic and real data have already
been conducted in \cite{negahban2016rank} to confirm the practicability
of both the spectral method and the regularized MLE.
See also \cite{chen2015spectral} for the experiments on the Spectral-MLE algorithm. This section provides some additional simulations to complement their experimental
results as well as our theory. Throughout the experiments, we set the number of items $n$
to be $200$, while the number of repeated comparisons $L$ and the
edge probability $p$ can vary with the experiments. Regarding the tuning
parameters, we choose $d=2d_{\max}$ in the spectral method where $d_{\max}$ is the maximum degree of the graph and $\lambda=2\sqrt{\frac{np\log n}{L}}$
in the regularized MLE, which are consistent with the configurations
considered in the main theorems. Additionally, we also display the experimental results for the unregularized MLE, i.e.~$\lambda = 0$. All of the results are averaged over
100 Monte Carlo simulations.

\begin{figure}
\centering
\begin{tabular}{ccc}
\includegraphics[width=0.3\textwidth]{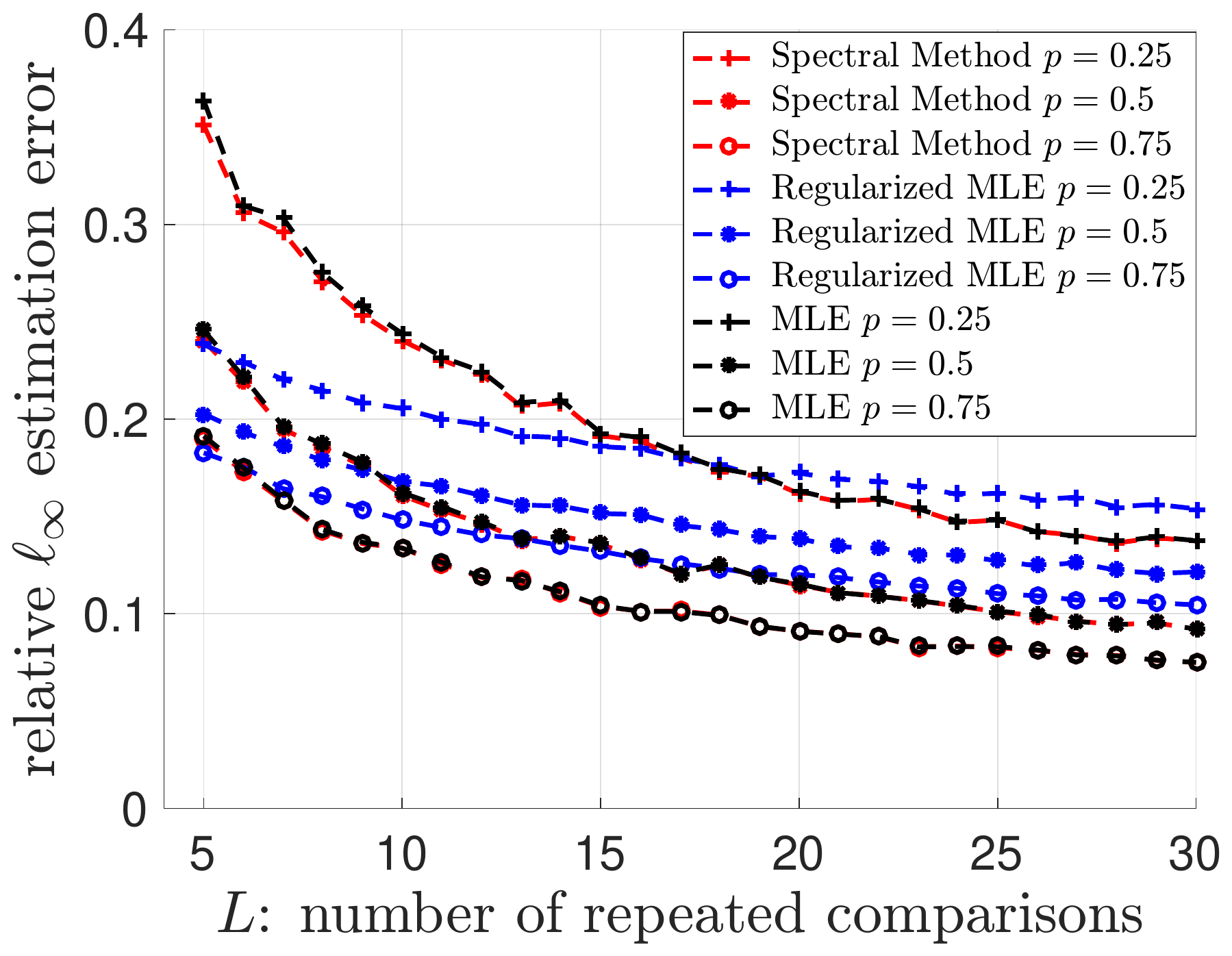} &
 \includegraphics[width=0.3\textwidth]{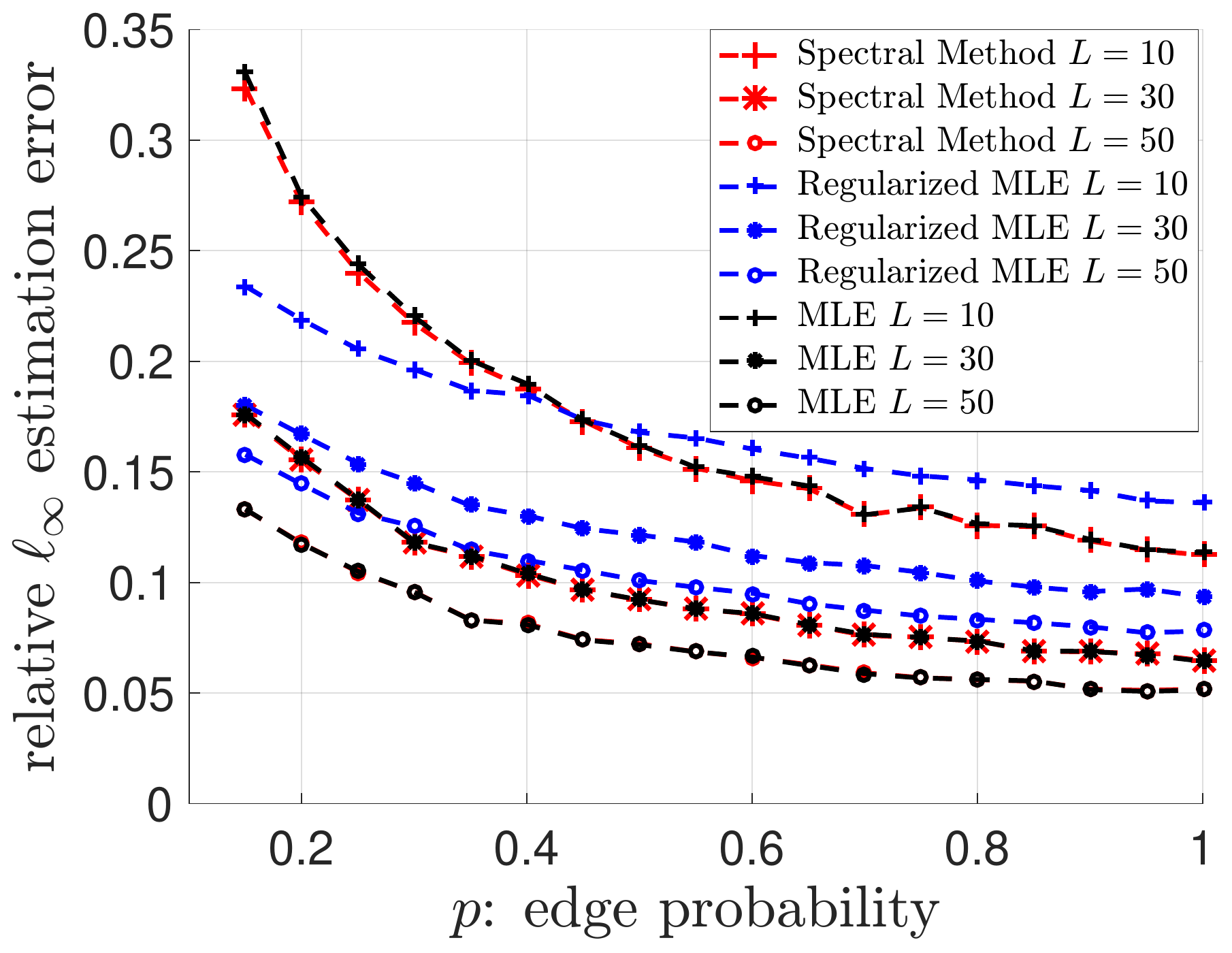} &
 \includegraphics[width=0.3\textwidth]{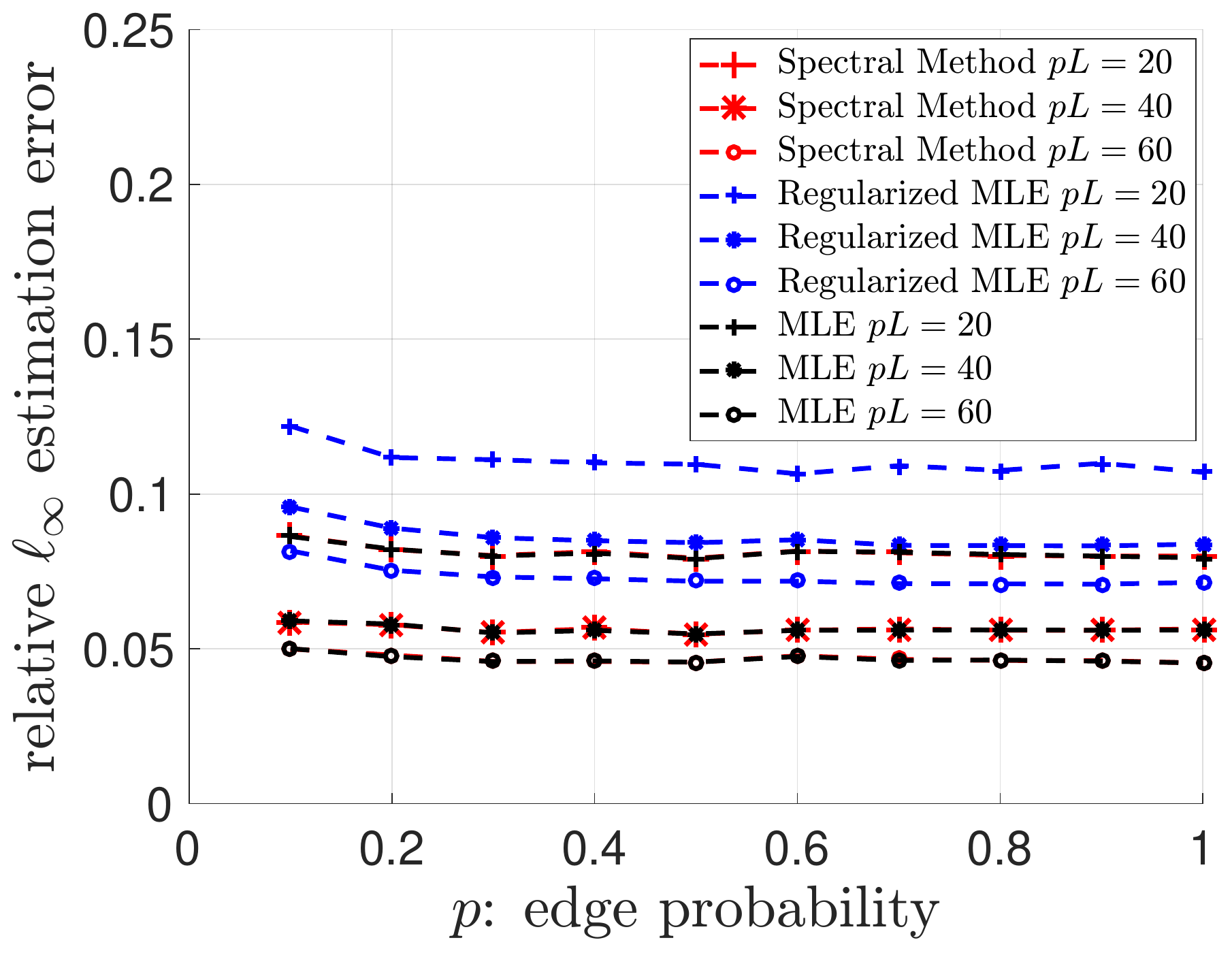}
 \tabularnewline
(a) & (b) & (c)\tabularnewline
\end{tabular}
\caption{Empirical performance of the spectral method and the (regularized)
MLE: (a) $\ell_{\infty}$ error vs.~$L$, (b) $\ell_{\infty}$ error
vs.~$p$ and (c) $\ell_{\infty}$ error vs.~$n^2 p L$ \label{fig:infty_vs_L-p} }
\end{figure}

\begin{figure}
\centering

\begin{tabular}{ccc}
\includegraphics[width=0.3\textwidth]{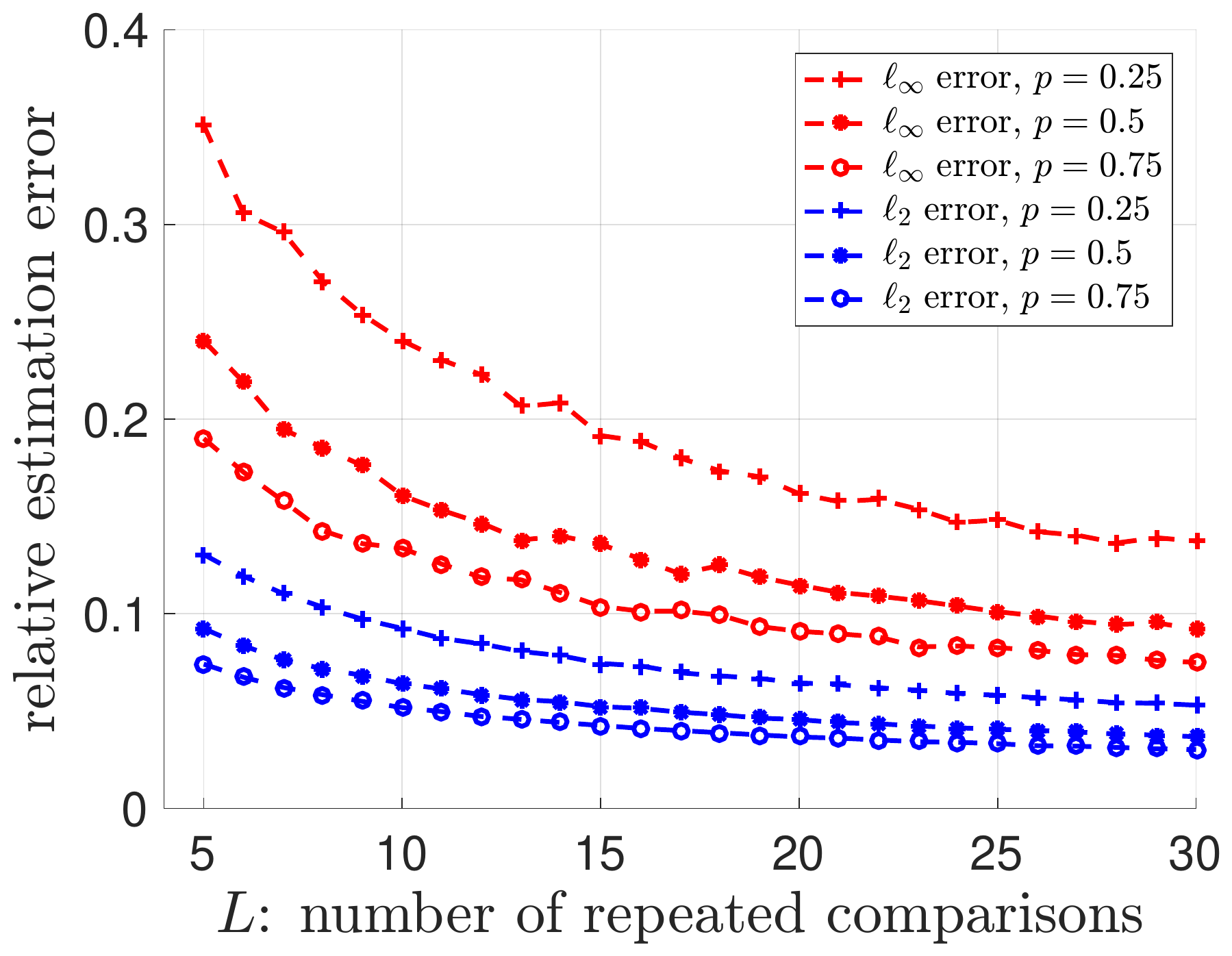}  & \includegraphics[width=0.3\textwidth]{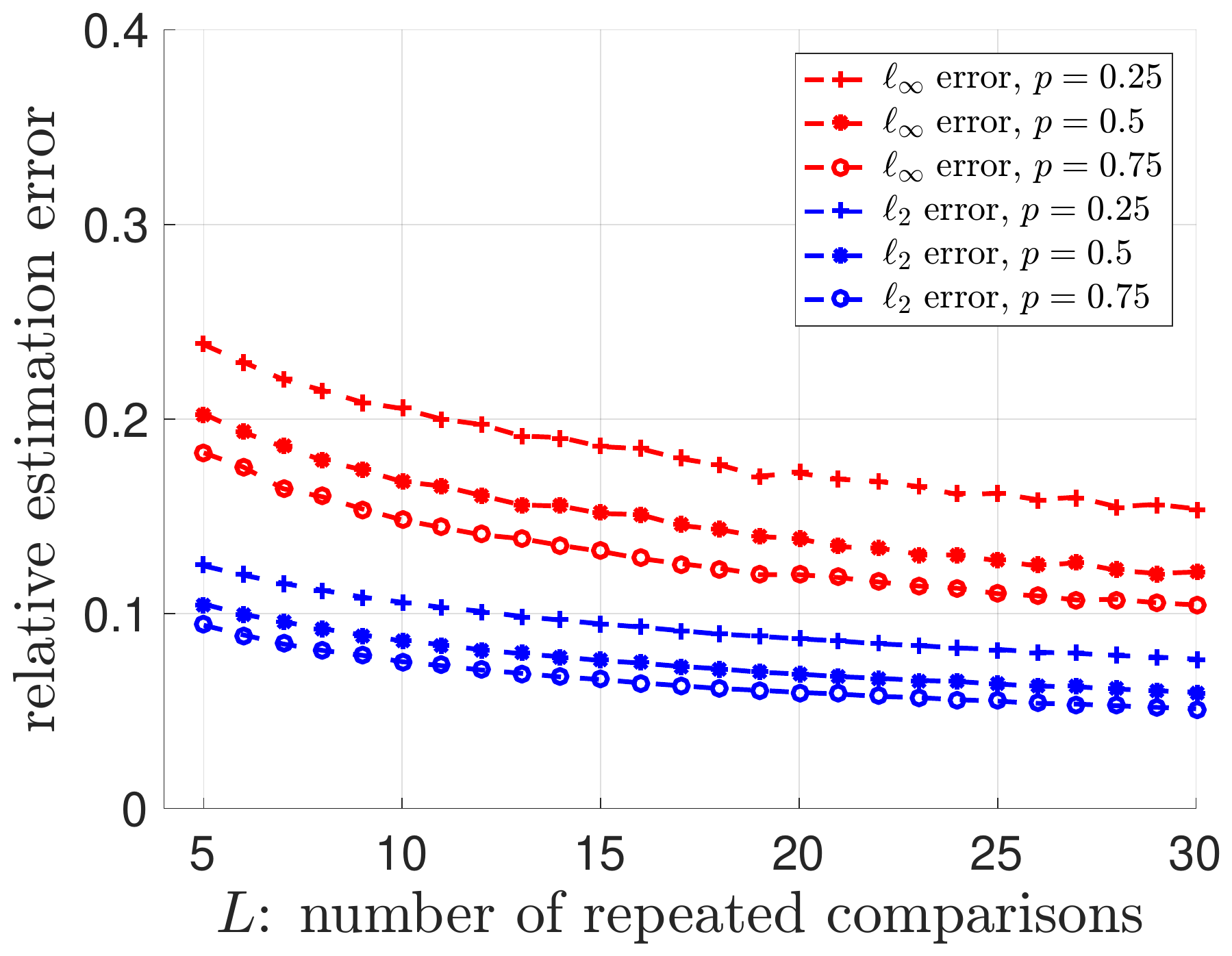} & \includegraphics[width=0.3\textwidth]{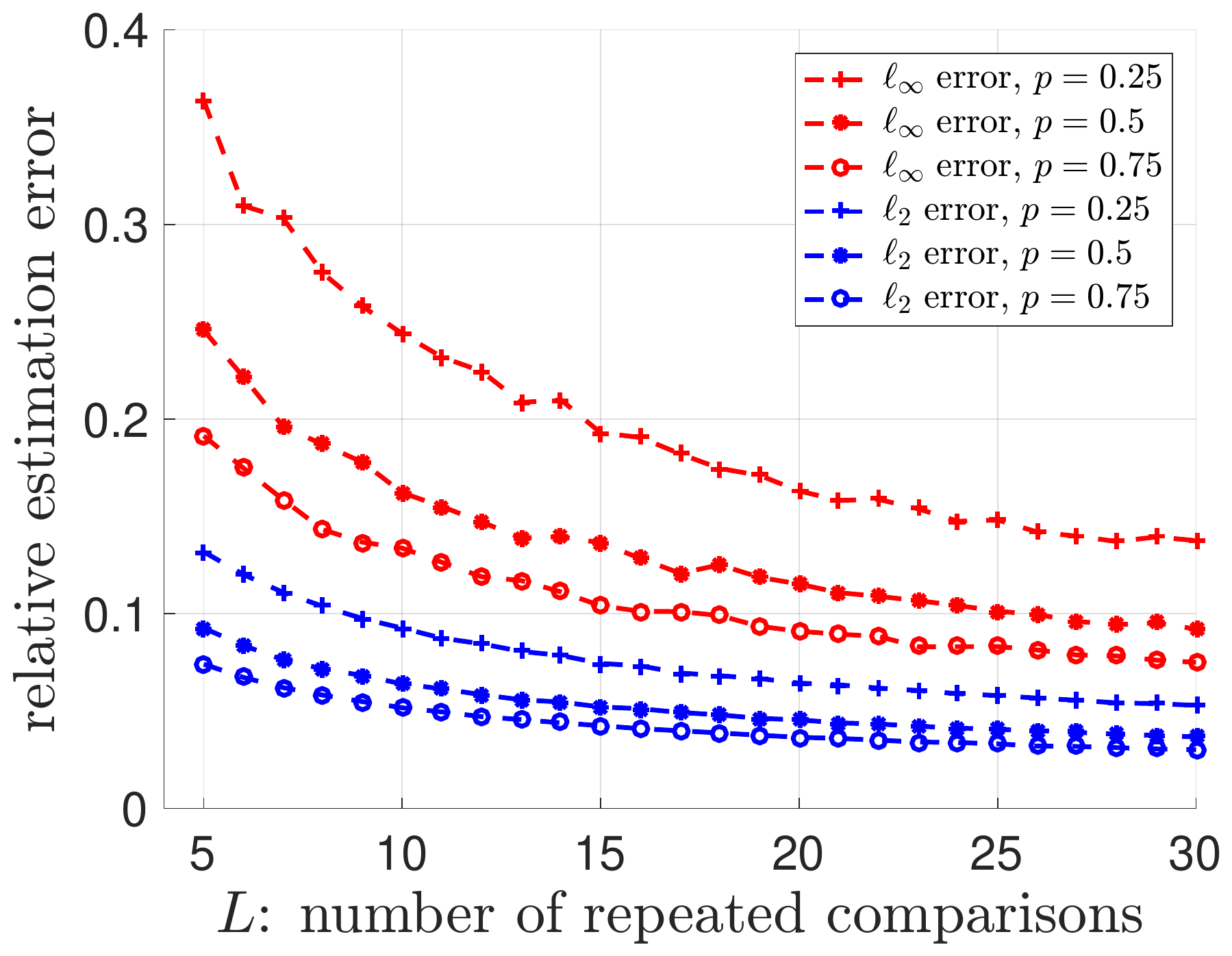}\tabularnewline
(a) spectral method & (b) regularized MLE & (c) MLE\tabularnewline
\end{tabular}

\caption{Comparisons between the relative $\ell_{\infty}$ error and the relative
$\ell_{2}$ error for (a) the spectral method, (b) the regularized
MLE and (c) the MLE.\label{fig:infty_vs_2}}
\end{figure}

\begin{figure}
\centering

\includegraphics[width=0.4\textwidth]{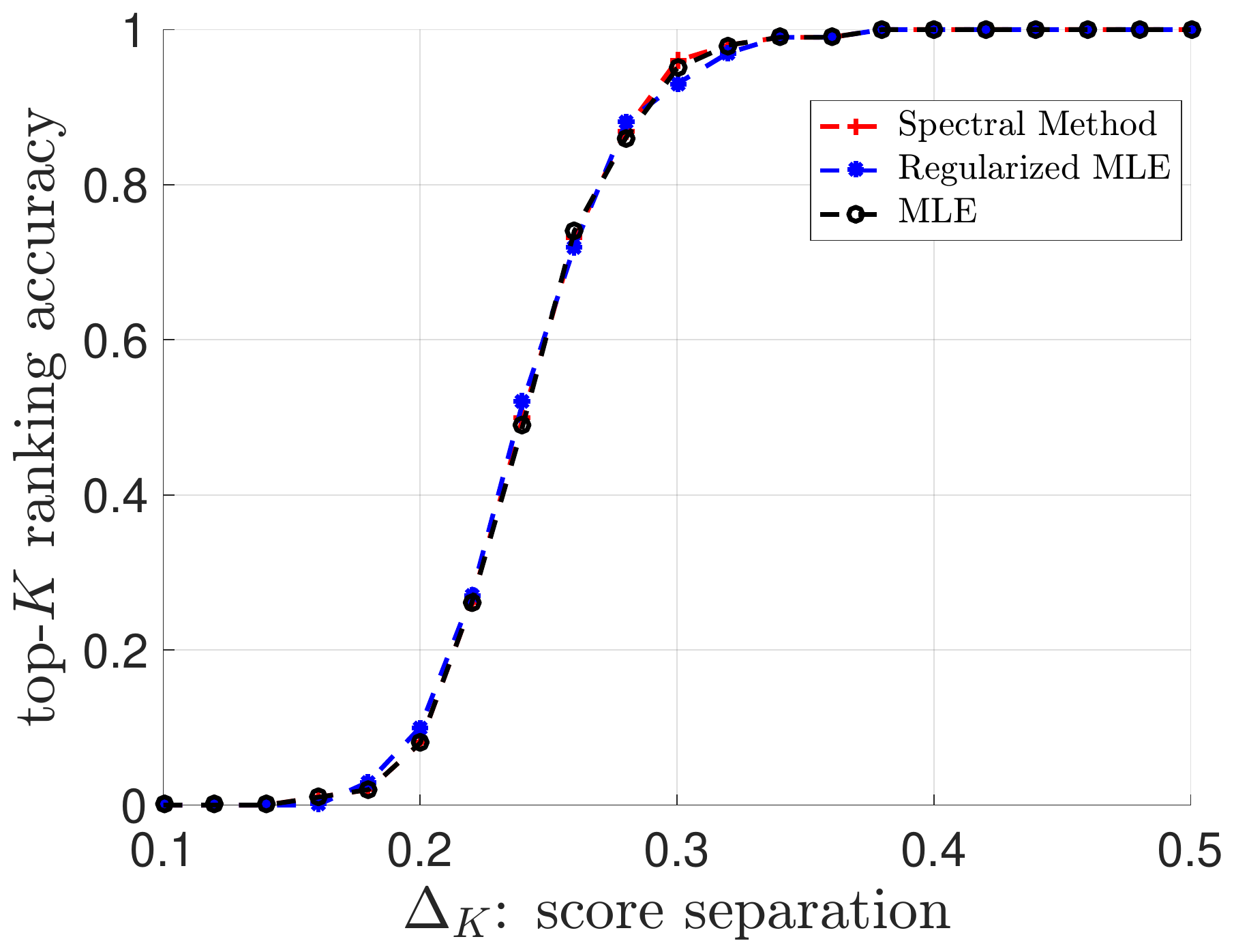}

\caption{The top-$K$ ranking accuracy of both the spectral method and the regularized
MLE. \label{fig:accuracy}}
\end{figure}

We first investigate the  $\ell_{\infty}$ error of the spectral method
and the (regularized) MLE when estimating the preference scores. To
this end, we generate the latent scores $w_{i}^{*}$ ($1\le i\leq n$)
independently and uniformly at random over the interval $\left[0.5,1\right]$.
Figure \ref{fig:infty_vs_L-p}(a) (resp.~Figure \ref{fig:infty_vs_L-p}(b))
displays the entrywise error in the spectral score estimation as the number
of repeated comparisons $L$ (resp.~the edge probability $p$) varies.
As is seen from the plots, the $\ell_{\infty}$ error of all
methods gets smaller as $p$ and $L$ increase, confirming our results
in Theorems \ref{thm:spectral-loss-infty}-\ref{thm:MLE-main}. Next, we show in Figure \ref{fig:infty_vs_L-p}(c) the relative $\ell_\infty$ error while fixing the total number of samples (i.e.~$n^2pL$).  It can be seen that the performance almost does not change if the sample complexity $n^2pL$ remains the same. It is also interesting to see that the $\ell_{\infty}$ error of the spectral method and the MLE are very similar.
In addition,  Figure \ref{fig:infty_vs_2} illustrates the relative
$\ell_{\infty}$ error and the relative $\ell_{2}$ error in score estimation for all three methods. As we can see, the relative
$\ell_{\infty}$ errors are not much larger than the relative $\ell_{2}$
errors (recall that $n=200$), thus offering empirical evidence that
the errors in the score estimates are spread out across all entries.

Further, we examine the top-$K$ ranking accuracy of all three methods.
Here, we fix $p=0.25$ and $L=20$, set ${\it K}=10$, and let $w^*_{i}=1$
for all $1\leq i\leq K$ and $w^*_{j}=1-\Delta$ for all $K+1\leq j\leq n$.
By construction, the score separation satisfies $\Delta_{K}=\Delta$.
Figure \ref{fig:accuracy} illustrates the accuracy
in identifying the top-$K$ ranked items. The performance of them improves when the score separation becomes larger, which matches our theory in Theorem \ref{thm:main-samples}.

\subsection{Other related works}

The problem of ranking based on partial preferences has received much
attention during the past decade. Two types of observation models
have been considered: $(1)$ \emph{the cardinal-based model}, where
users provide explicit numerical ratings of the items,
$(2)$ \emph{the ordinal-based model}, where users are asked to make
comparative measurements. See \cite{ammar2011ranking}
for detailed  comparisons between them.

In terms of the ordinal-based model --- and in particular, ranking
from pairwise comparisons --- both parametric and nonparametric models
have been extensively studied. For example, \cite{hunter2004mm} examined
variants of the parametric BTL model, and established
the convergence properties of the minorization-maximization
algorithm for computing the MLE. 
%The BTL model was also investigated in \cite{borkar2016randomized}, where the authors
%proposed a randomized Kaczmarz algorithm to infer the latent scores
%with provable $\ell_{2}$ error guarantees.
Moreover, the BTL model falls under the category of {\em low-rank parametric models}, since the preference matrix is generated by passing a rank-2 matrix through the logistic link function
\citep{rajkumar2016can}. Additionally, the work  \cite{jiang2011statistical} proposed a least-squares type method to estimate the full ranking, which generalizes the simple Borda count algorithm \citep{ammar2011ranking}.
For many of these algorithms, the sample complexities needed for perfect total ranking were determined by
\cite{rajkumar2014statistical}, although the top-$K$ ranking accuracy was not considered there.

Going beyond the parametric
models, a recent line of works \cite{shah2016stochastically,shah2015simple,chen2017competitive, martin2017worstranking}
considered the nonparametric \emph{stochastically transitive model},
where the only model assumption is that the comparison probability
matrix follows certain transitivity rules. This type of models subsumes
the BTL model as a special case. For instance, \cite{shah2015simple}
suggested a simple counting-based algorithm which can
reliably recover the top-$K$ ranked items for various models.
However, the sampling paradigm considered therein is quite different
from ours in the sparse regime; for instance, their model does not
come close to the setting where $p$ is small but $L$ is large, which
is the most challenging regime of the model adopted in our paper and
\cite{negahban2016rank,chen2015spectral}.

All of the aforementioned papers concentrate on the case where there
is a single ground-truth ordering. It would also be interesting to
investigate the scenarios where different users might have different
preference scores.
To this end, \cite{negahban2017learning,lu2014individualized} imposed
the low-rank structure on the underlying preference matrix and adopted
the nuclear-norm relaxation approach to recover the users' preferences.
%The paper \cite{wu2015clustering} proposed a two-step procedure,where a clustering step was first applied to label different users, followed by an estimation step built upon the previous clustering outcome. 
Additionally, several papers explored the ranking problem for
the more general Plackett-Luce model \citep{hajek2014minimax,SoufianiChenParkesXioa2013},
in the presence of adaptive sampling \citep{jamieson2011active,Hullermeier2013topKranking,heckel2016active,agarwal2017learning},
for the crowdsourcing scenario \citep{chen2013pairwise},
and in the adversarial setting \citep{suh2017adversarial}. These are
beyond the scope of the present paper.

Speaking of the error metric, the $\ell_{\infty}$ norm is appropriate for top-$K$ ranking problem and other learning problems as well. In particular, $\ell_{\infty}$   perturbation bounds for eigenvectors of symmetric matrices \citep{KLo16,fan2016ell_,EBW17,WZ2017unfinished} and singular vectors of general matrices \citep{KXi16} have been studied.
In stark contrast, we study the $\ell_{\infty}$ norm errors of the leading eigenvector of a class of asymmetric matrices (probability transition matrix) and the regularized MLE. Furthermore, most existing results require the expectations of data matrices to have low rank, at least approximately. We do not impose such assumptions.

When it comes to the technical tools, it is worth noting that the
leave-one-out idea has been invoked to analyze random designs for
other high-dimensional problems, e.g.~robust M-estimators
\citep{el2015impact}, confidence intervals for Lasso
\citep{javanmard2015biasing}, likelihood ratio test \citep{sur2017likelihood}, and nonconvex statistical learning \citep{ma2017implicit,chen2018gradient}.
In particular, \cite{zhong2017near} and \cite{WZ2017unfinished} use it to precisely characterize entrywise behavior of eigenvectors of a large class of symmetric random matrices, which improves upon  prior  $\ell_{\infty}$ eigenvector analysis. Consequently, they are able to show the sharpness of spectral methods in many popular models. Our introduction of leave-one-out auxiliary quantities is similar in spirit to these papers.

%\yxc{can you discuss your paper with Yiqiao here as well as   }

Finally,  the family of spectral methods has
been successfully applied in numerous applications, e.g.~matrix completion
\citep{keshavan2009matrix}, phase retrieval \citep{chen2015TWF},
graph clustering \citep{rohe2011spectral,WZ2017unfinished},  joint alignment \citep{chen2016projected}. All
of them are designed based on the eigenvectors of some symmetric matrix,
or the singular vectors if the matrix of interest is asymmetric. Our
paper contributes to this growing literature by establishing a sharp {\it{eigenvector}}
perturbation analysis framework for an important class of \emph{asymmetric}
matrices --- the probability transition matrices.

%% file: extension.tex
\section{Extension: general dynamic range\label{sec:Extension}}

All of the preceding results concern the regime with a fixed dynamic
range (i.e.~$\kappa=O(1)$). This section moves on to discussing the case with large $\kappa$.

To start with, by going through the same proof technique, we can readily
obtain \textemdash{} in the general $\kappa$ setting \textemdash{}
the following performance guarantees for both the spectral estimate
$\bm{\pi}$ and the regularized MLE $\bm{\theta}$.

\begin{theorem}\label{thm:spectral-main-general-kappa}Consider the
pairwise comparison model in Section \ref{sec:Problem-formulation}.
Suppose that $p>\frac{c_{0}\kappa^{5}\log n}{n}$ for some sufficiently
large constant $c_{0}>0$, and choose $d=c_{d}np$ for some constant
$c_{d}\geq2$ in Algorithm \ref{alg:spectral}. Then with probability exceeding $1-O\big(n^{-5}\big)$,
\begin{enumerate}
\item the spectral estimate $\bm{\pi}$ satisfies
\[
\frac{\left\Vert \bm{\pi}-\bm{\pi}^{*}\right\Vert _{\infty}}{\left\Vert \bm{\pi}^{*}\right\Vert _{\infty}}\lesssim\kappa\sqrt{\frac{\log n}{npL}},
\]
 where $\bm{\pi}^{*}$
is the normalized score vector as defined in \eqref{eq:defn-pi-star}.
\item the set of top-$K$ ranked items can be
recovered exactly by the spectral method given in Algorithm \ref{alg:spectral}, as long as
\[
	\frac{n^{2}pL}{2}\geq
	%c_{1}\frac{\kappa^{5}n\log n}{\Delta_{K}^{2}}  
		 c_{1} \frac{\kappa^{2}n\log n}{\Delta_{K}^{2}} %\max\left\{ {\kappa^{5}n\log n},~ \right\} 
\]
for some sufficiently large constant $c_{1}>0$.
\end{enumerate}
\end{theorem}

\begin{theorem}\label{thm:MLE-main-general-kappa}Consider the pairwise
comparison model in Section \ref{sec:Problem-formulation}. Suppose
that $p\geq\frac{c_{0}\kappa^{4}\log n}{n}$ for some sufficiently
large constant $c_{0}>0$ and that $L\leq c_{2}\cdot n^{c_{3}}$ for any absolute constants $c_{2}, c_{3}>0$. Set the regularization
parameter to be $\lambda=c_{\lambda}\frac{1}{\log\kappa}\sqrt{\frac{np\log n}{L}}$
for some absolute constant $c_{\lambda}>0$. Then with probability exceeding $1-O\big(n^{-5}\big)$,
\begin{enumerate}
\item the regularized MLE $\bm{\theta}$ satisfies
\[
\frac{\big\| e^{\bm{\theta}}-e^{\bm{\theta}^{*}-\overline{\theta}^{*}\bm{1}}\big\|_{\infty}}{\big\| e^{\bm{\theta}^{*}-\overline{\theta}^{*}\bm{1}}\big\|_{\infty}}\lesssim\kappa^{2}\sqrt{\frac{\log n}{npL}}
\]
 where $\overline{\theta}^{*}:=\frac{1}{n}\bm{1}^{\top}\bm{\theta}^{*}$
and $e^{\bm{\theta}}:=[e^{\theta_{1}},\cdots,e^{\theta_{n}}]^{\top}$.
\item the set of top-$K$ ranked items can be
recovered exactly by the regularized MLE given in (\ref{eq:regularized-MLE}), as long as
\[
\frac{n^{2}pL}{2}\geq c_{1}\frac{\kappa^{4}n\log n}{\Delta_{K}^{2}}
\]
for some sufficiently large constant $c_{1}>0$.
	%then with probability exceeding $1-O\left(n^{-5}\right)$, .
\end{enumerate}
\end{theorem}
\begin{remark}The guarantees on exact top-$K$ recovery
for both the spectral method and the regularized MLE are immediate
consequences of their $\ell_{\infty}$ error bound, as we have argued
in Section \ref{sec:optimal-Linfty}. Hence we will focus on proving the $\ell_{\infty}$
error bound in Sections \ref{sec:Analysis-for-spectral}--\ref{sec:Analysis-for-MLE}. \end{remark}

Notably, the achievability bounds for top-$K$ ranking in Theorems
\ref{thm:spectral-main-general-kappa}--\ref{thm:MLE-main-general-kappa}
do not match the lower bound asserted in Theorem \ref{thm:lower-bound}
in terms of $\kappa$. This is partly because the separation measure
$\Delta_{K}$ fails to capture the information bottleneck for the
general $\kappa$ setting. In light of this, we introduce the following
new measure that seems to be a more suitable metric to reflect the
hardness of the top-$K$ ranking problem:
\begin{equation}
\Delta_{K}^{*}:=\frac{w_{K}^{*}-w_{K+1}^{*}}{w_{K+1}^{*}}\cdot\sqrt{\frac{1}{n}\sum\nolimits_{i=1}^{n}\frac{w_{K+1}^{*}w_{i}^{*}}{\left(w_{K}^{*}+w_{i}^{*}\right)^{2}}},\label{eq:new-hardness-measure}
\end{equation}
which will be termed the \emph{generalized separation measure}. Informally,
$(\Delta_{K}^{*})^{2}$ is a reasonably tight upper bound on certain
normalized KL divergence metric. With this metric in place, we derive another
lower bound as follows.

\begin{theorem}\label{thm:new-lower-bound}Fix $\epsilon\in(0,\frac{1}{2})$,
and let $\mathcal{G}\sim\mathcal{G}_{n,p}$. Consider any preference score
vector $\bm{w}^*$, and let $\Delta_{K}^{*}$ denote its generalized
separation. If
\[
	n^{2}pL\leq\frac{\epsilon^{2}}{2}\frac{n}{(\Delta_{K}^{*})^{2}},
\]
then there exists another preference score vector $\tilde{\bm{w}}$ with
the same generalized separation $\Delta_{K}^{*}$ and different top-$K$
items such that $P_{\mathrm{e}}(\psi)\geq\frac{1-\epsilon}{2}$ for
any ranking scheme $\psi$. Here, $P_{\mathrm{e}}(\psi)$ represents
the probability of error in distinguishing these two vectors given $\bm{y}$.\end{theorem}
\begin{proof}See Appendix \ref{sec:Proof-of-Theorem-new-lower-bound}.\end{proof}

The preceding sample complexity lower bound scales inversely proportionally to $\left(\Delta_K^{*}\right)^{2}$.
To see why this generalized measure may be more suitable compared to the original separation
metric, we single out three examples in Appendix \ref{sec:appendix_example}.
Unfortunately, 
our current analyses do not yield a matching upper bound
with respect to $\Delta_{K}^{*}$ unless $\kappa$ is a constant.
For instance, the analysis of the spectral method relies on the eigenvector perturbation bound (Theorem \ref{lemma:mc-perturbation}), where the spectral gap and matrix perturbation play a crucial rule. However, the current results for controlling these quantities have explicit dependency on $\kappa$ \cite{negahban2016rank}. 
%Similarly, in the analysis of regularized MLE, the smallest positive eigenvalue of the Hessian is critical in the current proof, which also depends on $\kappa$. 
It is not clear whether we could incorporate the new measure to eliminate such dependency on $\kappa$. 
This calls for more refined analysis techniques, which
we leave for future investigation.

Moreover, it is not obvious whether the spectral method alone or the regularized MLE alone can achieve the minimal sample complexity in the general $\kappa$ regime. It is possible that one needs to first screen out those items with extremely high or low scores using methods like Borda count \citep{ammar2012efficient}, as advocated by \citep{negahban2016rank, chen2015spectral, jang2016top}. All in all, finding tight upper bounds for general $\kappa$ remains an open question.

%% file: discussion.tex
\section{Discussion\label{sec:Discussion}}

This paper justifies the optimality of both the spectral method
and the regularized MLE for top-$K$ rank aggregation for the fixed dynamic range case. Our theoretical
studies are by no means exhaustive, and there are numerous directions
that would be of interest for future investigations. We point out
a few possibilities as follows. 

%\begin{itemize}
%\item 
\smallskip
\noindent\emph{General condition number $\kappa$. } As mentioned before, our current theory is 
optimal in the presence of a fixed dynamic range with $\kappa=O\left(1\right)$.
We have also made a first attempt in considering the large $\kappa$
regime. It is desirable to characterize the statistical and computational
limits for more general $\kappa$. 

%\item 
\smallskip
\noindent\emph{Goodness-of-fit. }Throughout this paper, we have assumed the
BTL model captures the randomness underlying the data we collect. A practical question is whether the real data actually follows the
BTL model. It would be interesting to investigate how to test the goodness-of-fit of this model.
		%which is essentially a goodness-of-fit problem. 
%\item

\smallskip
\noindent\emph{Unregularized MLE.} We have studied the optimality of the regularized
MLE with the regularization parameter $\lambda\asymp\sqrt{\frac{np\log n}{L}}$.
Our analysis relies on the regularization term to obtain convergence of the gradient descent algorithm (see Lemma \ref{lemma:convergence}). It is natural to ask whether
such a regularization term is necessary or not. This question remains
open. %Actually our regularizer $\lambda \asymp \sqrt{\frac{np \log n}{L}}$ is already quite small compared with the typical size of $\| \nabla^2 \mathcal{L}(\bm  \btheta^*;\bm y ) \|$ which is of order $np$. To see this, consider the minimal sample complexity required for top-$K$ identification outlined in (\ref{eq:sample-complexity-spectral-MLE}).
%\item 

\smallskip
\noindent
\emph{More general comparison graphs.} So far we have focused on a
tractable but somewhat restrictive comparison graph, namely, the Erd\H{o}s\textendash R\'enyi
random graph. It would certainly be important to understand
the performance of both methods under a broader family of comparison
graphs, and to see which algorithms would enable optimal sample complexities
under general sampling patterns. 
%\item \emph{Beyond pairwise comparisons.} In real world applications, we often encounter the case when we are given more than two items to compare. For instance, in online advertisement scenarios, the users might be asked to compare three to five different advertisements. How to extend the methods and analyses to handle more general \emph{listwise} comparisons is of significant practical interest. 

%\item 
\smallskip
\noindent
\emph{Entrywise perturbation analysis for convex optimization.}
This paper provides the $\ell_{\infty}$ perturbation analysis
for the regularized MLE using the leave-one-out trick as well as an
inductive argument along the algorithmic updates. We expect this analysis
framework to carry over to a much broader family of convex optimization
problems, which may in turn offer a powerful tool for showing the
stability of optimization procedures in an entrywise fashion. 
%\end{itemize}

%% file: analysis_spectral.tex
\section{Analysis for the spectral method\label{sec:Analysis-for-spectral}}

This section is devoted to proving Theorem \ref{thm:spectral-main-general-kappa} and hence Theorem \ref{thm:spectral-loss-infty},
which characterizes the pointwise error of the spectral estimate.

\subsection{Preliminaries}

Here, we gather some preliminary facts about reversible Markov chains
as well as the Erd\H{o}s\textendash R\'enyi random graph.

The first important result concerns the eigenvector perturbation for probability
transition matrices, which can be treated as the analogue of the
celebrated Davis-Kahan $\sin\Theta$ theorem \citep{davis1970rotation}. Due to its potential importance for other
problems, we promote it to a theorem as follows.

\begin{theorem}[{\bf Eigenvector perturbation}]\label{lemma:mc-perturbation}Suppose
that $\bm{P}$, $\hat{\bm{P}}$, and $\bm{P}^{*}$ are probability
transition matrices with stationary distributions $\bm{\pi}$, $\hat{\bm{\pi}}$,
$\bm{\pi}^{*}$, respectively. Also, assume that $\bm{P}^{*}$ represents
a reversible Markov chain. When $\big\|\bm{P}-\hat{\bm{P}}\big\|_{\bm{\pi}^{*}} < 1- \max\left\{ \lambda_{2}(\bm{P}^{*}),-\lambda_{n}\left(\bm{P}^{*}\right)\right\} $, it holds that
\[
\|\bm{\pi}-\hat{\bm{\pi}}\|_{\bm{\pi}^{*}}\leq\frac{\big\|\bm{\pi}^{\top}(\bm{P}-\hat{\bm{P}})\big\|_{\bm{\pi}^{*}}}{1-\max\left\{ \lambda_{2}(\bm{P}^{*}),-\lambda_{n}\left(\bm{P}^{*}\right)\right\} -\big\|\bm{P}-\hat{\bm{P}}\big\|_{\bm{\pi}^{*}}}.
\]
 \end{theorem}
\begin{proof}See Appendix \ref{sec:Proof-of-Lemma-mc-perturbation}. \end{proof}

%\yxc{Kaizheng: can you add any other things that you think the readers should know?}

Several remarks regarding Theorem \ref{lemma:mc-perturbation} are in order. First, in contrast to standard perturbation results like Davis-Kahan's $\sin\Theta$ theorem, our theorem involves three matrices in total, where $\bm{P}$, $\hat{\bm{P}}$, and $\bm{P}^{*}$ can all be arbitrary. For example, one may choose  $\bm{P}^*$ to be the population transition matrix, and $\bm{P}$ and $\hat{\bm{P}}$ as two finite-sample versions associated with $\bm{P}^*$. Second, we only impose reversibility on $\bm{P}^*$, whereas $\bm{P}$ and $\hat{\bm{P}}$ need not induce reversible Markov Chains. Third, Theorem \ref{lemma:mc-perturbation} allows one to derive
the $\ell_{2}$ estimation error in \cite{negahban2016rank} directly
without resorting to the power method;
%, which we defer to Section \ref{sec:Proof-of-ell_2} in the Supplementary Material.
 in fact, our $\ell_2$ estimation error bound improves upon \cite{negahban2016rank} by some logarithmic factor. 
\begin{theorem}\label{thm:L2-pi}
Consider the pairwise comparison
model specified in Section \ref{sec:Problem-formulation} with $\kappa=O(1)$.
Suppose $p\geq c_{0}\frac{\log n}{n}$
	for some sufficiently large constant $c_{0}>0$ and $d\geq c_{d}np$
	for $c_{d}\geq2$ in Algorithm \ref{alg:spectral}. With probability
	exceeding $1-O(n^{-5})$, one has
	\[
	\frac{\|\bm{\pi}-\bm{\pi}^{*}\|_{2}}{\|\bm{\pi^{*}}\|_{2}}\lesssim  \frac{1}{\sqrt{npL}}.
	\]
\end{theorem}
\begin{proof} See Appendix \ref{sec:Proof-of-Theorem-L2-pi}. \end{proof}

Notably, Theorem \ref{thm:L2-pi} matches the minimax lower bound derived in 
\cite[Theorem 3]{negahban2016rank}. As far as we know, this is the first result that demonstrates the orderwise optimality of the spectral method when measured by the $\ell_2$ loss.

%In regard to the concentration of the vertex
%degrees in an Erd\H{o}s\textendash R\'enyi random graph, since $d$ is chosen to be $c_{d}np$ for some constant $c_{d}\geq2$,
%we have, by Lemma \ref{lemma:degree} in the Supplementary Material, that the maximum vertex degree obeys $d_{\max} < d$ with
%high probability.

The next result is concerned with the concentration of the vertex
degrees in an Erd\H{o}s\textendash R\'enyi random graph.

\begin{lemma}[Degree concentration]\label{lemma:degree}Suppose
that $\cG\sim\cG_{n,p}$. Let $d_{i}$ be the degree of node $i$,
$d_{\min}=\min_{1\leq i\leq n}d_{i}$ and $d_{\max}=\max_{1\leq i\leq n}d_{i}$.
If $p\geq\frac{c_{0}\log n}{n}$ for some sufficiently large constant
$c_{0}>0$, then the following event
\begin{align}
\cA_{0}=\left\{ \frac{np}{2}\leq d_{\min}\leq d_{\max}\leq\frac{3np}{2}\right\}\label{eq:event_0}
\end{align}
obeys
\[
\PP\left(\cA_{0}\right)\geq1-O(n^{-10}).
\]
\end{lemma}\begin{proof}The proof follows from the standard Chernoff
bound and is hence omitted. \end{proof}

Since $d$ is chosen to be $c_{d}np$ for some constant $c_{d}\geq2$,
we have, by Lemma \ref{lemma:degree}, that the maximum vertex degree obeys $d_{\max} < d$ with
high probability.

\subsection{Proof outline of Theorem \ref{thm:spectral-main-general-kappa}\label{subsec:Proof-outline-of-Theorem-Spectral}}
%\kzw{Shall we introduce the leave-one-out stuff at the beginning, or later when it comes up in $\left(\bm{\pi}-\bm{\pi}^{*}\right)^{\top}\bm{P}_{\cdot m}$ (after $I_1^m$ and $I_2^m$)?}

In this subsection, we outline the proof of Theorem \ref{thm:spectral-main-general-kappa}.

Recall that $\bm{\pi}=\left[\pi_{1},\cdots,\pi_{n}\right]^{\top}$
and $\bm{\pi}^{*}=\left[\pi_{1}^{*},\cdots,\pi_{n}^{*}\right]^{\top}$
are the stationary distributions associated with $\bm{P}$ and $\bm{P}^{*}$,
respectively. This gives
\[
\bm{\pi}^{\top}\bm{P}=\bm{\pi}^{\top}\qquad\text{and}\qquad\bm{\pi}^{*\top}\bm{P}^{*}=\bm{\pi}^{*\top}.
\]
For each $1\leq m\leq n$, one can decompose
%\kzw{I made some changes below to emphasize the necessity of leave-one-out method. How about your opinions?}
\begin{align*}
\pi_{m}-\pi_{m}^{*} & =\bm{\pi}^{\top}\bm{P}_{\cdot m}-\bm{\pi}^{*\top}\bm{P}_{\cdot m}^{*}=\bm{\pi}^{*\top}\left(\bm{P}_{\cdot m}-\bm{P}_{\cdot m}^{*}\right)+\left(\bm{\pi}-\bm{\pi}^{*}\right)^{\top}\bm{P}_{\cdot m}\\
 & =\underset{:=I_{1}^{m}}{\underbrace{\sum_{j}\pi_{j}^{*}\left(P_{j,m}-P_{j,m}^{*}\right)}}+\underset{:=I_{2}^{m}}{\underbrace{\left(\pi_{m}-\pi_{m}^{*}\right)P_{m,m}\vphantom{\sum_{j}\pi_{j}^{*}}}}+\sum_{j:j\neq m}\left(\pi_{j}-\pi_{j}^{*}\right)P_{j,m},
\end{align*}
where $\bm{P}_{\cdot m}$ (resp. $\bm{P}^*_{\cdot m}$) denotes the $m$-th column of $\bm{P}$ (resp. $\bm{P}^*$).
Then it boils down to controlling $I_{1}^{m}$, $I_{2}^{m}$ and $\sum_{j:j\neq m}(\pi_{j}-\pi_{j}^{*})P_{j,m}$.
\begin{enumerate}
\item Since $\bm{\pi}^{*}$ is deterministic while $\bm{P}$ is random, we
can easily control $I_{1}^{m}$ using Hoeffding's inequality. The
bound is the following. \begin{lemma}\label{lemma:spectral_I1m}
With probability exceeding $1-O(n^{-5})$, one has
\[
\max_{m}|I_{1}^{m}|\lesssim\sqrt{\frac{\log n}{Ld}}\|\bm{\pi}^{*}\|_{\infty}.
\]
\end{lemma}
\begin{proof}See Appendix \ref{sec:Proof-of-Lemma-spectral_I1m}. \end{proof}
\item Next, we show the term $I_{2}^{m}$ behaves as a contraction of $|\pi_{m}-\pi_{m}^{*}|$.
%as demonstrated below. 
\begin{lemma}\label{lemma:spectral_I2m}With
probability exceeding $1-O(n^{-5})$, there exists some constant $c>0$
such that for all $1\leq m \leq n$,
\[
\ensuremath{|I_{2}^{m}|\leq\left(1-\frac{np}{2(1+\kappa)d}+c\sqrt{\frac{\log n}{Ld}}\right)|\pi_{m}-\pi_{m}^{*}|}.
\]
\end{lemma} 
\begin{proof}See Appendix \ref{sec:Proof-of-Lemma-spectral_I2m}. \end{proof}

\item The statistical dependency between $\bm{\pi}$ and $\bm{P}$ introduces difficulty in obtaining a sharp estimate of the third term $\sum_{j:j\neq m}(\pi_{j}-\pi_{j}^{*})P_{j,m}$. Nevertheless, the leave-one-out technique helps us decouple the dependency and obtain effective control of this term.
The key component of the analysis is the introduction of a new probability
transition matrix $\bm{P}^{(m)}$, which is a leave-one-out version
of the original matrix $\bm{P}$. More precisely, $\bm{P}^{(m)}$
replaces all of the transition probabilities involving the $m$-th
item with their expected values (unconditional on $\mathcal{G}$);
that is, for any $i\neq j$,
\[
P_{i,j}^{(m)}:=\begin{cases}
P_{i,j}, & \quad i\neq m,\text{ }j\neq m \\
\frac{p}{d}y_{i,j}^{*}, & \quad i=m\text{ or }j=m
\end{cases}
\]
with $y_{i,j}^{*}:=\frac{w_{j}^{*}}{w_{i}^{*}+w_{j}^{*}}$. 
For any $1\leq i\leq n$, set
\[
P_{i,i}^{(m)}:=1-\sum\nolimits_{j:j\neq i}P_{i,j}^{(m)}
\]
in order to ensure that $\bm{P}^{(m)}$ is a probability transition
matrix. In addition, we let $\bm{\pi}^{\left(m\right)}$ be the stationary
distribution of the Markov chain induced by $\bm{P}^{(m)}$. As will be demonstrated later, the main
advantages of introducing $\bm{\pi}^{(m)}$ are two-fold: (1) the original spectral estimate $\bm{\pi}$
is very well approximated by $\bm{\pi}^{(m)}$, and (2) $\bm{\pi}^{(m)}$
is statistically independent of the connectivity of the $m$-th node
and the comparisons with regards to the $m$-th item. Now we further decompose $\sum_{j:j\neq m}(\pi_{j}-\pi_{j}^{*})P_{j,m}$:
\begin{align*}
\sum\nolimits_{j:j\neq m}\left(\pi_{j}-\pi_{j}^{*}\right)P_{j,m}
=\underbrace{\sum\nolimits_{j:j\neq m}(\pi_{j}-\pi_{j}^{(m)})P_{j,m}}_{:=I_{3}^{m}}+\underbrace{\sum\nolimits_{j:j\neq m}(\pi_{j}^{(m)}-\pi_{j}^{*})P_{j,m}}_{:=I_{4}^{m}}.
\end{align*}

\item For $I_{3}^{m}$, we apply the Cauchy-Schwarz inequality to obtain
that with probability at least $1-O(n^{-10}),$
\[
|I_{3}^{m}|\leq\|\bm{\pi}^{(m)}-\bm{\pi}\|_{2}\left(\sum\nolimits _{j:j\neq m}P_{j,m}^{2}\right)^{\frac{1}{2}}\overset{\left(\text{i}\right)}{\leq}\frac{1}{\sqrt{d}}\|\bm{\pi}^{(m)}-\bm{\pi}\|_{2},
\]
where $\left(\text{i}\right)$ follows from the fact that $P_{j,m}\leq \frac{1}{d}$ for all $j\neq m$ and $d_{\max}\leq d$ on the event $\cA_{0}$ (defined in Lemma \ref{lemma:degree}). Consequently,
it suffices to control the $\ell_{2}$ difference between the original
spectral estimate $\bm{\pi}$ and its leave-one-out version $\bm{\pi}^{\left(m\right)}$.
 This is accomplished in the following lemma.

\begin{lemma}\label{lemma:spectral-leave-ell-2}
Suppose that $np\gamma^2 > c \kappa \log n$ for some sufficiently large constant $c>0$. With probability
at least $1-O\left(n^{-5}\right)$,
\begin{align}
\|\bm{\pi}^{(m)}-\bm{\pi}\|_{2} & \leq\frac{16\sqrt{\kappa}}{\gamma}\sqrt{\frac{\log n}{Ld}}\|\bm{\pi}^{*}\|_{\infty}+\|\bm{\pi}-\bm{\pi}^{*}\|_{\infty},\label{eq:spectral-ell-2-infty}
\end{align}
where $\kappa=w_{\max}/w_{\min}$ and $\gamma=1-\max\left\{ \lambda_{2}(\bm{P}^{*}),-\lambda_{n}\left(\bm{P}^{*}\right)\right\} -\|\bm{P}-\bm{P}^{*}\|_{\bm{\pi}^{*}}$.
\end{lemma}
\begin{proof}See Appendix \ref{sec:Proof-of-Lemma-spectral_leave-ell-2}. \end{proof}
Using \cite[Lemmas 3 and 4]{negahban2016rank} and our Lemma \ref{lemma:degree}, we can bound $\gamma$ from below:
\begin{lemma}[{\bf Spectral gap}, \cite{negahban2016rank}]\label{lemma:spectral-gap} Under the
model specified in Section \ref{sec:Problem-formulation}, if $p\geq c_{0} \frac{\log n}{n} \max \{ 1 , \frac{\kappa^5}{L}
\}$ for some sufficiently large constant $c_{0}>0$, then with probability
at least $1-O(n^{-5})$, 
\[
\gamma:=1-\max\left\{ \lambda_{2}(\bm{P}^{*}),-\lambda_{n}\left(\bm{P}^{*}\right)\right\} -\|\bm{P}-\bm{P}^{*}\|_{\bm{\pi}^{*}}\geq\frac{1}{2\kappa^{2}}\frac{d_{\min}}{d_{\max}}\geq\frac{1}{6\kappa^{2}}.
\]
 \end{lemma}
%With this spectral gap in mind, Lemma \ref{lemma:spectral-leave-ell-2}
%reads
%\begin{align}
%\|\bm{\pi}^{(m)}-\bm{\pi}\|_{2} & \lesssim\sqrt{\frac{ \kappa^5 \log n}{Ld}}\|\bm{\pi}^{*}\|_{\infty}+\|\bm{\pi}-\bm{\pi}^{*}\|_{\infty}.
%\end{align}

\item In order to control $I_{4}^{m}$, we exploit the statistical independence
between $\bm{\pi}^{(m)}$ and $\bm{P}_{\cdot m}$. Specifically, we
demonstrate that:

\begin{lemma}\label{lemma:spectral-I4m}Suppose that $p>\frac{c_{0}\log n}{n}$
for some sufficiently large constant $c_{0}>0$. With probability
at least $1-O\left(n^{-10}\right)$,
\[
|I_{4}^{m}| \lesssim\frac{1}{\sqrt{n}}\|\bm{\pi}^{(m)}-\bm{\pi}\|_{2}+\sqrt{\frac{\log n}{Ld}}\|\bm{\pi}^{*}\|_{\infty}+\frac{\sqrt{np\log n}+\log n}{d}\|\bm{\pi}^{(m)}-\bm{\pi}^{*}\|_{\infty}.
\]
 \end{lemma} 
 \begin{proof}See Appendix \ref{sec:Proof-of-Lemma-spectral-I4m}. \end{proof} 
 The above bound depends on both $\|\bm{\pi}^{(m)}-\bm{\pi}\|_{2}$
and $\|\bm{\pi}^{(m)}-\bm{\pi}^{*}\|_{\infty}$. We can invoke Lemma
\ref{lemma:spectral-leave-ell-2} and the inequality $\|\bm{\pi}^{(m)}-\bm{\pi}^{*}\|_{\infty}\leq\|\bm{\pi}^{(m)}-\bm{\pi}\|_{2}+\|\bm{\pi}-\bm{\pi}^{*}\|_{\infty}$
to reach
\begin{align*}
|I_{4}^{m}| & \lesssim\left(\frac{1}{\sqrt{n}}+\frac{\sqrt{np\log n}+\log n}{d}\right)\|\bm{\pi}^{(m)}-\bm{\pi}\|_{2}+\sqrt{\frac{\log n}{Ld}}\|\bm{\pi}^{*}\|_{\infty}\\
&\quad+\frac{\sqrt{np\log n}+\log n}{d}\|\bm{\pi}-\bm{\pi}^{*}\|_{\infty}\\
 & \lesssim\left\{ \left(\frac{1}{\sqrt{n}}+\frac{\sqrt{np\log n}+\log n}{d}\right)\frac{\sqrt{\kappa}}{\gamma}+1\right\} \sqrt{\frac{\log n}{Ld}}\|\bm{\pi}^{*}\|_{\infty}\\
 &\quad+\left(\frac{1}{\sqrt{n}}+\frac{\sqrt{np\log n}+\log n}{d}\right)\|\bm{\pi}-\bm{\pi}^{*}\|_{\infty}.
\end{align*}

\item Finally we put the preceding bounds together. 
When $\frac{np}{\kappa^5 \log n}$ is large enough, with high probability, for some absolute constants $c_{1},c_{2},c_{3}>0$
one has {
\begin{align*}
&\left(\frac{np}{2(1+\kappa)d}-c_{1}\sqrt{\frac{\log n}{Ld}}\right)|\pi_{m}-\pi_{m}^{*}| \\
&\quad \leq\left\{ c_{2}+\left(\frac{1}{\sqrt{d}}+\frac{c_{3}}{\sqrt{n}}+c_{3}\frac{\sqrt{np\log n}+\log n}{d}\right)\frac{16\sqrt{\kappa}}{\gamma}\right\} \sqrt{\frac{\log n}{Ld}}\|\bm{\pi}^{*}\|_{\infty}\\
 &\quad +\left(\frac{1}{\sqrt{d}}+\frac{c_{3}}{\sqrt{n}}+2c_{3}\frac{\sqrt{np\log n}+\log n}{d}\right)\|\bm{\pi}-\bm{\pi}^{*}\|_{\infty}.
\end{align*}}
simultaneously for all $1\leq m\leq n$. By taking the maximum over
$m$ on the left-hand side and combining terms, we get
\begin{align*}
 & \underset{:=\alpha_{1}}{\underbrace{\left(
 		\frac{np}{2(1+\kappa)d}-c_{1}\sqrt{\frac{\log n}{Ld}}-\frac{1}{\sqrt{d}}-\frac{c_{3}}{\sqrt{n}}-2c_{3}\frac{\sqrt{np\log n}+\log n}{d}
 		\right)}}\|\bm{\pi}-\bm{\pi}^{*}\|_{\infty}\\
 & \qquad\leq\underset{:=\alpha_{2}}{\underbrace{\left\{ c_{2}+\left(\frac{1}{\sqrt{d}}+\frac{c_{3}}{\sqrt{n}}+c_{3}\frac{\sqrt{np\log n}+\log n}{d}\right)\frac{16\sqrt{\kappa}}{\gamma}\right\} }}\sqrt{\frac{\log n}{Ld}}\|\bm{\pi}^{*}\|_{\infty}.
\end{align*}
Hence, as long as $\frac{np}{\kappa^5 \log n}$ is sufficiently large, one has
\[
c_{1}\sqrt{\frac{\log n}{Ld}} + \frac{1}{\sqrt{d}} + \frac{c_{3}}{\sqrt{n}} + 2c_{3}\frac{\sqrt{np\log n}+\log n}{d}
\lesssim
\sqrt{\frac{\log n}{np}}
\lesssim  \frac{1}{\sqrt{\kappa^{5}}},
\]
which further leads to $\alpha_1 \gtrsim 1/\kappa$, $\alpha_{2}\lesssim 1$, and
\[
\|\bm{\pi}-\bm{\pi}^{*}\|_{\infty}\lesssim
\kappa \sqrt{\frac{\log n}{Ld}}\|\bm{\pi}^{*}\|_{\infty}\asymp
\kappa \sqrt{\frac{\log n}{npL}}\|\bm{\pi}^{*}\|_{\infty}.
\]
This finishes the proof of Theorem \ref{thm:spectral-main-general-kappa} and Theorem \ref{thm:spectral-loss-infty}.
\end{enumerate}

%% file: analysis_MLE.tex
\section{Analysis for the regularized MLE\label{sec:Analysis-for-MLE}}

This section  establishes the $\ell_{\infty}$ error
of the regularized MLE as claimed in Theorem \ref{thm:MLE-main-general-kappa} (and also Theorem \ref{thm:MLE-main}). Recall that in Theorem \ref{thm:MLE-main-general-kappa},
we compare the regularized MLE $\bm{\theta}$ with $\bm{\theta}^{*}-\overline{\theta}^{*}\bm{1}$.
Therefore, without loss of generality we can assume that 
\begin{equation}
\bm{1}^{\top}\bm{\theta}^{*}=0.\label{eq:assumption-mean}
\end{equation}
This combined with the fact that $\theta_{\max} - \theta_{\min} = \log \kappa$ reveals that
\begin{equation*}
\left\|\bm{\theta}^*\right\|_{\infty} \leq \log \kappa \qquad\text{and}\qquad \left\|\bm{\theta}^*\right\|_{2} \leq \sqrt{n}\log \kappa.
\end{equation*}
In addition, we assume that $L=O\left(n^5\right)$ in this section. It is straightforward to extend the proof to cover $L\leq c_{2} \cdot n^{c_{3}}$ for any constants $c_{2},c_{3}>0$.

\subsection{Preliminaries and notation}

Before proceeding to the proof, we gather some basic facts. To begin
with, the gradient and the Hessian of $\mathcal{L}\left(\cdot;\bm{y}\right)$
in (\ref{eq:MLE}) can be computed as 
\begin{align}
\nabla\mathcal{L}\left(\bm{\theta};\bm{y}\right) & =\sum_{(i,j)\in\mathcal{E},i>j}\left\{ -y_{j,i}+\frac{e^{\theta_{i}}}{e^{\theta_{i}}+e^{\theta_{j}}}\right\} \left(\bm{e}_{i}-\bm{e}_{j}\right);\label{eq:gradient}
\end{align}
\begin{align}
\nabla^{2}\mathcal{L}\left(\bm{\theta};\bm{y}\right) & =\sum_{(i,j)\in\mathcal{E},i>j}\frac{e^{\theta_{i}}e^{\theta_{j}}}{\left(e^{\theta_{i}}+e^{\theta_{j}}\right)^{2}}\left(\bm{e}_{i}-\bm{e}_{j}\right)\left(\bm{e}_{i}-\bm{e}_{j}\right)^{\top}.\label{eq:defn-Hessian}
\end{align}
Here $\bm{e}_{1},\cdots,\bm{e}_{n}$ stand for the canonical basis vectors
in $\RR^{n}$. When evaluated at the truth $\bm{\theta}^{*}$, the
size of the gradient can be controlled as follows.

\begin{lemma}\label{lemma:grad-L}
Let $\lambda$ be as specified in Theorem \ref{thm:MLE-main-general-kappa}. The following event 
\begin{equation}
\mathcal{A}_{2}:=\left\{ \left\Vert \nabla\mathcal{L}_{\lambda}\left(\bm{\theta}^{*};\bm{y}\right)\right\Vert _{2}\lesssim\sqrt{\frac{n^{2}p\log n}{L}}\right\} .\label{eq:event_2}
\end{equation}
occurs with probability exceeding $1-O(n^{-10})$. \end{lemma}
\begin{proof}See Appendix \ref{sec:Proof-of-Lemma-grad-L}.\end{proof}

The following lemmas characterize the smoothness and the strong convexity
of the function $\mathcal{L}_{\lambda}\left(\cdot;\bm{y}\right)$.
In the sequel, we denote by $\bm{L}_{\mathcal{G}}=\sum_{(i,j)\in\mathcal{E},i>j}\left(\bm{e}_{i}-\bm{e}_{j}\right)\left(\bm{e}_{i}-\bm{e}_{j}\right)^{\top}$
the (unnormalized) Laplacian matrix \citep{chung1997spectral} associated
with $\mathcal{G}$. For any matrix $\bm{A}$ we let 
\begin{equation}
\lambda_{\min,\perp}(\bm{A}):= \min\left\{ \mu\mid\bm{z}^{\top}\bm{A}\bm{z}\geq\mu\|\bm{z}\|_2^{2}\text{ for all }\bm{z}\text{ with }\bm{1}^{\top}\bm{z}=0\right\} ,\label{eq:defn-lambda-min}
\end{equation}
namely, the smallest eigenvalue when restricted to vectors orthogonal to $\bm{1}$.

\begin{lemma}\label{lem:smoothness-L}Suppose that $p>\frac{c_{0}\log n}{n}$
for some sufficiently large constant $c_{0}>0$. Then on the event
$\mathcal{A}_{0}$ as defined in (\ref{eq:event_0}), one has 
\[
\lambda_{\max}\left(\nabla^{2}\mathcal{L}_{\lambda}\left(\bm{\theta};\bm{y}\right)\right)\leq\lambda+np,\qquad\forall\bm{\theta}\in\mathbb{R}^{n}.
\]

\end{lemma}\begin{proof}Note that $\frac{e^{\theta_{i}}e^{\theta_{j}}}{(e^{\theta_{i}}+e^{\theta_{j}})^{2}}\leq\frac{1}{4}$.
It follows immediately from the Hessian in (\ref{eq:defn-Hessian})
that 
\[
\lambda_{\max}\left(\nabla^{2}\mathcal{L}_{\lambda}\left(\bm{\theta};\bm{y}\right)\right)\leq\lambda+\frac{1}{4}\left\Vert \bm{L}_{\mathcal{G}}\right\Vert \leq\lambda+\frac{1}{2}d_{\max},
\]
where $d_{\max}$ is the maximum vertex degree in the graph $\cG$.
In addition, on the event $\cA_{0}$ we have $d_{\max}\leq2np$, which
completes the proof. \end{proof}

\begin{lemma}\label{lemma:strong-convexity-L}For all $\bm{\theta}\in\RR^{n}$
such that $\left\Vert \bm{\theta}-\bm{\theta}^{*}\right\Vert _{\infty}\leq C$
for some $C\geq0$, we have 
\[
\lambda_{\min,\perp}\left(\nabla^{2}\mathcal{L}_{\lambda}\left(\bm{\theta};\bm{y}\right)\right)\geq\lambda+\frac{1}{4\kappa e^{2C}}\lambda_{\min,\perp}\left(\bm{L}_{\mathcal{G}}\right).
\]

\end{lemma}
\begin{proof}See Appendix \ref{subsec:Proof-of-Lemma-strong-convexity}.\end{proof}

\begin{lemma}\label{lemma:laplacian}Let $\cG\sim\cG_{n,p}$, and
suppose that $p>\frac{c_{0}\log n}{n}$ for some sufficiently large
constant $c_{0}>0$. Then one has 
\[
\PP\left(\lambda_{\min,\perp}\left(\bm{L}_{\mathcal{G}}\right)\geq np/2 \right)\geq1-O\left(n^{-10}\right).
\]
\end{lemma}\begin{proof}Note that $\lambda_{\min,\perp}\left(\bm{L}_{\mathcal{G}}\right)$
is exactly the spectral gap of the Laplacian matrix. See \cite[Sec 5.3.3]{tropp2015introduction}
for the derivation of this lemma. \end{proof}

By combining Lemma \ref{lemma:strong-convexity-L} with Lemma \ref{lemma:laplacian},
we reach the following result.

\begin{corollary}\label{coro:strong-convexity}Under the assumptions
of Lemma \ref{lemma:laplacian}, with probability exceeding $1-O\left(n^{-10}\right)$
one has 
\[
\lambda_{\min,\perp}\left(\nabla^{2}\mathcal{L}_{\lambda}\left(\bm{\theta};\bm{y}\right)\right)\geq\lambda+\frac{1}{8\kappa e^{2C}}np
\]
simultaneously for all $\bm{\theta}$ obeying $\left\Vert \bm{\theta}-\bm{\theta}^{*}\right\Vert _{\infty}\leq C$
for some $C\geq0$. \end{corollary}

\subsection{Proof outline of Theorem \ref{thm:MLE-main-general-kappa}}

This subsection outlines the main steps for establishing Theorem \ref{thm:MLE-main-general-kappa}. 

Rather than directly resorting to the optimality condition, we adopt
an algorithmic perspective to analyze the regularized MLE $\bm{\theta}$.
Specifically, we consider the standard gradient descent algorithm
that is expected to converge to the minimizer $\bm{\theta}$, and
analyze the trajectory of this iterative algorithm instead. The algorithm
is stated in Algorithm \ref{alg:Gradient-Descent}. 
\begin{algorithm}
\begin{algorithmic} \STATE \textbf{Initialize} {$\bm{\theta}^{0}=\bm{\theta}^{*}$}.
\FOR{$t=0,1,2,\ldots,T-1$} \STATE{ 
	\vspace{-0.5em}
\begin{equation}
\bm{\theta}^{t+1}=\bm{\theta}^{t}-\eta_{t}\nabla\mathcal{L}_{\lambda}\left(\bm{\theta}^{t};\bm{y}\right);\label{eq:gradient-update-original}
\end{equation}
}\ENDFOR 
%\STATE \textbf{Output} {$\bm{\theta}^{T}$}. 
\end{algorithmic}

\caption{Gradient descent for computing the regularized MLE.\label{alg:Gradient-Descent}}
\end{algorithm}

Notably, this gradient descent algorithm is not practical since the
initial point is set to be $\bm{\theta}^{*}$. Nevertheless, it is
helpful for analyzing the statistical accuracy of the regularized
MLE $\bm{\theta}$. In what follows, we shall adopt a time-invariant
step size rule: 
\begin{equation}
\eta_{t}\equiv\eta=%\frac{1}{20\kappa}
	\frac{1}{\lambda+np},\qquad t=0,1,2,\cdots \label{eq:eta-choice}
\end{equation}
Our proof can be divided into three steps: 
\begin{enumerate}
\item[I.] establish \textemdash{} via standard optimization theory \textemdash{}
that the output $\bm{\theta}^{T}$ of Algorithm \ref{alg:Gradient-Descent}
is sufficiently close to the regularized MLE $\bm{\theta}$, namely,
\begin{equation}
\left\Vert \bm{\theta}^{T}-\bm{\theta}\right\Vert _{\infty}\leq\left\Vert \bm{\theta}^{T}-\bm{\theta}\right\Vert _{2}\leq C_{0}\kappa^{2}\sqrt{\frac{\log n}{npL}}\label{eq:thetaT-theta-distance}
\end{equation}
for $T=n^{5}$, where $C_{0}>0$ is some absolute constant; 
\item[II.] use the leave-one-out argument to demonstrate that: the output $\bm{\theta}^{T}$
is close to the truth $\bm{\theta}^{*}$ in an entrywise fashion,
i.e. 
\[
\left\Vert \bm{\theta}^{T}-\bm{\theta}^{*}\right\Vert _{\infty}\leq C_{4}\kappa^{2}\sqrt{\frac{\log n}{npL}}
\]
for some universal constant $C_{4}>0$. Combining this with
(\ref{eq:thetaT-theta-distance}) yields 
\[
\left\Vert \bm{\theta}-\bm{\theta}^{*}\right\Vert _{\infty}\lesssim\kappa^{2}\sqrt{\frac{\log n}{npL}};
\]
\item[III.] the final step is to translate the perturbation bound on $\|\bm{\theta}-\bm{\theta}^{*}\|_{\infty}$
to $\|e^{\bm{\theta}}-e^{\bm{\theta}^{*}}\|_{\infty}$ as claimed
in the theorem. 
\end{enumerate}
Before continuing, we single out an important fact that will be used
throughout the proof. \begin{fact}\label{fact:mean} Suppose $\bm{1}^{\top}\bm{\theta}^{*}=0$. Then we have $\bm{1}^{\top}\bm{\theta}^{t}=0$
for all $t\geq0$. \end{fact} 
\begin{proof}See Appendix \ref{subsec:Proof-of-fact:mean}. \end{proof}
\subsection{Step I}

The first step relies heavily on optimization theory, namely the
theory of gradient descent on strongly convex and smooth functions. 
\begin{enumerate}
\item It is seen that the sequence $\left\{ \bm{\theta}^{t}\right\} _{t=1}^{\infty}$
converges geometrically fast to the regularized MLE $\bm{\theta}$,
a property that is standard in convex optimization literature. This
claim is summarized in the following lemma.\begin{lemma}\label{lemma:convergence}On
the event $\cA_{0}$ as defined in (\ref{eq:event_0}), one has 
\[
\left\Vert \bm{\theta}^{t}-\bm{\theta}\right\Vert _{2}\leq\rho^{t}\left\Vert \bm{\theta}^{0}-\bm{\theta}\right\Vert _{2},
\]
where $\rho=1-%\frac{1}{20\kappa}
\frac{\lambda}{\lambda+np}$.\end{lemma}\begin{proof}This result directly follows from the smoothness
property (see Lemma \ref{lem:smoothness-L}), the trivial strong convexity
of $\cL_{\lambda}\left(\bm{\theta};\bm{y}\right)$ ($\nabla^{2}\cL_{\lambda}\left(\bm{\theta};\bm{y}\right)\succeq\lambda\bm{I}_{n},\;\forall\bm{\theta}$),
as well as the convergence property of the gradient descent algorithm
(e.g.~\cite[Theorem 3.10]{bubeck2015convex}). \end{proof}

A direct consequence of this convergence result and Fact \ref{fact:mean}
is that $\bm{1}^{\top}\bm{\theta}=0$ for the regularized MLE $\bm{\btheta}$.
\item We then control $\|\bm{\theta}^{0}-\bm{\theta}\|_{2}$. Recall that
$\bm{\theta}^{0}=\bm{\theta}^{*}$, and we have: 
\begin{lemma}\label{lemma:ell_2_rate}On the event $\mathcal{A}_{2}$
as defined in (\ref{eq:event_2}), there exists some constant $c_{2}>0$
such that 
\[
\|\bm{\theta}^{0}-\bm{\theta}\|_{2}=\left\Vert \bm{\theta}-\bm{\theta}^{*}\right\Vert _{2}\leq c_{2}\sqrt{n}\log\kappa.
\]
\end{lemma}
\begin{proof}See Appendix \ref{subsec:Proof-of-Lemma-Coarse-ell-2}. \end{proof} 
\item The previous two claims taken together lead us to conclude that 
\begin{align*}
\left\Vert \bm{\theta}^{T}-\bm{\theta}\right\Vert _{2} & \leq\rho^{T}\|\bm{\theta}^{0}-\bm{\theta}\|_{2}\leq\rho^{T}c_{2}\sqrt{n}\log\kappa=c_{2}\left(1-\frac{\lambda}{\lambda+np}\right)^{T}\sqrt{n}\log\kappa \\
	&\leq c_{2}\exp\left(-\frac{T\lambda}{\lambda+np}\right)\sqrt{n}\kappa^{2}  \\
 & \leq c_{2}\exp\left(-\frac{T}{c_{3}\log\kappa}\sqrt{\frac{\log n}{npL}}\right)\sqrt{n}\kappa^{2}\qquad(\text{by }\lambda\asymp\frac{1}{\log\kappa}\sqrt{\frac{np\log n}{L}}<np)\\
 & \leq c_{2}\exp\left(-\frac{T}{c_{4}\log n}\sqrt{\frac{\log n}{npL}}\right)\sqrt{n}\kappa^{2}\qquad(\text{by }\kappa^{4}\lesssim\frac{np}{\log n}\leq\frac{n}{\log n})\\
	& \leq C_{0}\kappa^{2}\sqrt{\frac{\log n}{npL}}, 
\end{align*}
for some constants $c_{3},c_{4},C_{0}>0$, $L\lesssim n^{5}$, and
sufficiently large $T$ (recall that $T=n^{5}$). The above bounds are somewhat loose, but they suffice for our purpose. We then naturally obtain
\[
\left\Vert \bm{\theta}^{T}-\bm{\theta}\right\Vert _{\infty}\leq\left\Vert \bm{\theta}^{T}-\bm{\theta}\right\Vert _{2}\leq C_{0}\kappa^{2}\sqrt{\frac{\log n}{npL}}
\]
as claimed. This finishes the first step of the proof.
\end{enumerate}

\subsection{Step II}

The purpose of this step is to show that all iterates $\left\{ \bm{\theta}^{t}\right\} _{0\leq t\leq T}$
are sufficiently close to $\bm{\theta}^{*}$ in terms of the $\ell_{\infty}$-norm
distance. To facilitate analysis, for each $1\leq m\leq n$, we
introduce a leave-one-out sequence $\left\{ \bm{\theta}^{t,\left(m\right)}\right\} $
constructed via the following update rule 
\begin{equation}
\bm{\theta}^{t+1,\left(m\right)}=\bm{\theta}^{t,\left(m\right)}-\eta\nabla\mathcal{L}_{\lambda}^{\left(m\right)}\big(\bm{\theta}^{t,\left(m\right)}\big),\label{eq:gradient-update-loo}
\end{equation}
where $\bm{\theta}^{0,\left(m\right)}=\bm{\theta}^{0}=\bm{\theta}^{*}$ and 
\begin{align}
\mathcal{L}_{\lambda}^{(m)}\left(\bm{\theta};\bm{y}\right) & :=\sum_{(i,j)\in\mathcal{E},i>j,i\neq m,j\neq m}\left\{ -y_{j,i}\left(\theta_{i}-\theta_{j}\right)+\log\big(1+e^{\theta_{i}-\theta_{j}}\big)\right\} \nonumber \\
 & \quad+\sum_{i:i\neq m}p\left\{ -\frac{e^{\theta_{i}^{*}}}{e^{\theta_{i}^{*}}+e^{\theta_{m}^{*}}}\left(\theta_{i}-\theta_{m}\right)+\log\big(1+e^{\theta_{i}-\theta_{m}}\big)\right\} +\frac{1}{2}\lambda\|\bm{\theta}\|_{2}^{2}.\label{eq:defn-L-lambda-m}
\end{align}
Here, the leave-one-out loss function $\mathcal{L}_{\lambda}^{(m)}\left(\bm{\theta};\bm{y}\right)$
replaces all log-likelihood components involving the $m$-th item
with their expected values (unconditional on $\mathcal{G}$). For any $1\leq m \leq n$, the
auxiliary sequence $\left\{ \bm{\theta}^{t,\left(m\right)}\right\} $
serves as a reasonably good proxy for $\left\{ \bm{\theta}^{t}\right\} $,
while remaining statistically independent of $\left\{ y_{i,m}\mid(i,m)\in\mathcal{E}\right\} $.

Our proof in this step is inductive in nature. For the sake of clarity,
we first list all induction hypotheses needed in our analysis:
\begin{subequations}\label{subeq:induction} 
\begin{align}
\left\Vert \bm{\theta}^{t}-\bm{\theta}^{*}\right\Vert _{2} & \leq C_{1}\kappa\sqrt{\frac{\log n}{pL}},\label{eq:induction-ell-2}\\
\max_{1\leq m\leq n}\left|\theta_{m}^{t,\left(m\right)}-\theta_{m}^{*}\right| & \leq C_{2}\kappa^{2}\sqrt{\frac{\log n}{npL}},\label{eq:induction-loo-error}\\
\max_{1\leq m\leq n}\big\Vert \bm{\theta}^{t}-\bm{\theta}^{t,\left(m\right)}\big\Vert _{2} & \leq C_{3}\kappa\sqrt{\frac{\log n}{npL}},\label{eq:induction-loo-perturbation}\\
\left\Vert \bm{\theta}^{t}-\bm{\theta}^{*}\right\Vert _{\infty} & \leq C_{4}\kappa^{2}\sqrt{\frac{\log n}{npL}},\label{eq:induction-infty}
\end{align}
\end{subequations}where $C_{1},\cdots, C_{4}>0$ are some
absolute  constants. We aim to show that if the iterates at
the $t$-th iteration \textemdash{} i.e.~$\bm{\theta}^{t}$ and $\left\{ \bm{\theta}^{t,\left(m\right)}\right\} _{1\leq m\leq n}$
\textemdash{} satisfy the induction hypotheses (\ref{subeq:induction}),
then the $\left(t+1\right)$-th iterates continue to satisfy these
hypotheses. Clearly, it suffices to justify (\ref{subeq:induction})
for all $0\leq t\leq T=n^5$.

Before we dive into the inductive arguments, there are a few direct
consequences of (\ref{subeq:induction}) that are worth listing. We
gather them in the next lemma. \begin{lemma}\label{lemma:consequence}Suppose
the induction hypotheses (\ref{subeq:induction}) hold true for the
$t$-th iteration, then there exist some universal constants $C_{5},C_{6}>0$
such that the following two bounds hold: \begin{subequations} 
\begin{align}
\max_{1\leq m\leq n}\big\Vert \bm{\theta}^{t,\left(m\right)}-\bm{\theta}^{*}\big\Vert _{\infty} & \leq C_{5}\kappa^{2}\sqrt{\frac{\log n}{npL}},\label{eq:consequence-theta-t-m-infty}\\
\max_{1\leq m\leq n}\big\Vert \bm{\theta}^{t,\left(m\right)}-\bm{\theta}^{*}\big\Vert _{2} & \leq C_{6}\kappa\sqrt{\frac{\log n}{pL}}. \label{eq:consequence-theta-t-m-ell-2}
\end{align}
\end{subequations} \end{lemma}
\begin{proof}See Appendix \ref{subsec:Proof-of-Lemma-consequence}. \end{proof}

Note that the base case (i.e.~the case for $t=0$) is trivially true
due to the same initial points, namely, $\bm{\theta}^{0,\left(m\right)}=\bm{\theta}^{0}=\bm{\theta}^{*}$
for all $1\leq m\leq n$. We start with the first induction hypothesis
(\ref{eq:induction-ell-2}), which is supplied below.

\begin{lemma}\label{lemma:ell_2_contraction}Suppose the induction
hypotheses (\ref{subeq:induction}) hold true for the $t$-th iteration,
then with probability at least $1-O\left(n^{-10}\right)$, one has
\[
\left\Vert \bm{\theta}^{t+1}-\bm{\theta}^{*}\right\Vert _{2}\leq C_{1}\kappa\sqrt{\frac{\log n}{pL}},
\]
as long as the step size obeys $0<\eta\leq%\frac{1}{10\kappa}
\frac{1}{\lambda+np}$ and $C_{1}>0$ is sufficiently large. \end{lemma}
\begin{proof}See Appendix \ref{subsec:Proof-of-Lemma-ell-2-contraction}. \end{proof}
The remaining induction steps are provided in the following lemmas.
\begin{lemma}\label{lemma:loo-m-entry-contraction}Suppose the induction
hypotheses (\ref{subeq:induction}) hold true for the $t$th iteration,
then with probability at least $1-O\left(n^{-10}\right)$, one has
\[
\max_{1\leq m\leq n}\big|\theta_{m}^{t+1,\left(m\right)}-\theta_{m}^{*}\big|\leq C_{2}\kappa^{2}\sqrt{\frac{\log n}{npL}},
\]
with the proviso that $0<\eta\leq\frac{1}{\lambda+np}$ and $C_{2}\gtrsim C_{6}+c_{\lambda}$.
\end{lemma}
\begin{proof}See Appendix \ref{subsec:Proof-of-Lemma-loo-m-entry-contraction}.\end{proof}

\begin{lemma}\label{lemma:loo-perturbation-contraction}Suppose the
induction hypotheses (\ref{subeq:induction}) hold true for the $t$-th
iteration, then with probability at least $1-O\left(n^{-10}\right)$,
one has 
\[
\max_{1\leq m\leq n}\big\Vert \bm{\theta}^{t+1}-\bm{\theta}^{t+1,\left(m\right)}\big\Vert _{2}\leq C_{3}\kappa\sqrt{\frac{\log n}{npL}},
\]
as long as the step size obeys $0<\eta\leq%\frac{1}{20\kappa}
\frac{1}{\lambda+np}$ and $C_{3}>0$ is sufficiently large. \end{lemma}
\begin{proof}See Appendix \ref{subsec:Proof-of-Lemma-loo-perturbation-contraction}. \end{proof}

\begin{lemma}\label{lemma:error-t}
Suppose the induction hypotheses (\ref{subeq:induction})
hold true for the $t$-th iteration, then with probability at least
$1-O\left(n^{-10}\right)$, one has 
\[
\left\Vert \bm{\theta}^{t+1}-\bm{\theta}^{*}\right\Vert _{\infty}\leq C_{4}\kappa^{2}\sqrt{\frac{\log n}{npL}}
\]
for any $C_{4}\geq C_{3}+C_{2}$. \end{lemma}
\begin{proof}See Appendix \ref{subsec:Proof-of-lemma:error-t}\end{proof}

Taking
the union bound over $T=n^{5}$ iterations yields that with probability
at least $1-O\left(n^{-5}\right)$, 
\[
\left\Vert \bm{\theta}^{T}-\bm{\theta}^{*}\right\Vert _{\infty}\leq C_{4}\kappa^{2}\sqrt{\frac{\log n}{npL}},
\]
which together with the conclusion in Step I results in 
\begin{equation}
\left\Vert \bm{\theta}-\bm{\theta}^{*}\right\Vert _{\infty}\leq\left\Vert \bm{\theta}^{T}-\bm{\theta}^{*}\right\Vert _{\infty}+\left\Vert \bm{\theta}^{T}-\bm{\theta}\right\Vert _{\infty}\leq\left(C_{0}+C_{4}\right)\kappa^{2}\sqrt{\frac{\log n}{npL}}.\label{eq:theta-infty-bound}
\end{equation}

\subsection{Step III}

It remains to show that 
\[
\frac{\left\Vert e^{\bm{\theta}}-e^{\bm{\theta}^{*}}\right\Vert _{\infty}}{\left\Vert e^{\bm{\theta^{*}}}\right\Vert _{\infty}}\lesssim\kappa^{2}\sqrt{\frac{\log n}{npL}}.
\]
Toward this end, we observe that for each $1\leq m\leq n$, 
\[
\frac{\big|e^{\theta_{m}}-e^{\theta_{m}^{*}}\big|}{e^{\theta_{\max}}}=\frac{\big|e^{\tilde{\theta}_{m}}\left(\theta_{m}-\theta_{m}^{*}\right)\big|}{e^{\theta_{\max}}}\leq\frac{e^{\theta_{\max}+\left\Vert \bm{\theta}-\bm{\theta}^{*}\right\Vert _{\infty}}\cdot\big|\theta_{m}-\theta_{m}^{*}\big|}{e^{\theta_{\max}}},
\]
where $\tilde{\theta}_{m}$ is between $\theta_{m}$ and $\theta_{m}^{*}$,
and $\theta_{\max}$ is the largest entry of $\theta^{*}$. Continuing
the derivation and using (\ref{eq:theta-infty-bound}), we arrive
at 
\[
\max_{1\leq m\leq n}\frac{\left|e^{\theta_{m}}-e^{\theta_{m}^{*}}\right|}{e^{\theta_{\max}}}\leq\frac{e^{\theta_{\max}+\left\Vert \bm{\theta}-\bm{\theta}^{*}\right\Vert _{\infty}}}{e^{\theta_{\max}}}\left\Vert \bm{\theta}-\bm{\theta}^{*}\right\Vert _{\infty}\lesssim\kappa^{2}\sqrt{\frac{\log n}{npL}}
\]
as long as $\kappa^{2}\sqrt{\frac{\log n}{npL}}$ is small enough. This completes the proof of Theorem~\ref{thm:MLE-main-general-kappa}.

%% file: proof-extension.tex
\section{Proof of Theorem \ref{thm:new-lower-bound}\label{sec:Proof-of-Theorem-new-lower-bound}}

As usual, suppose the truth has preference scores $w_{1}^*,\cdots,w_{n}^*$.
To establish the lower bound, we construct another slightly perturbed
scenario where the score of the $i$th ranked item is $\tilde{w}_{i}$
as defined by 
\begin{equation}
\tilde{w}_{i}=\begin{cases}
w_{i}^*,\qquad & \text{if }i\neq K\text{ and }i\neq K+1,\\
w_{K+1}^*, & \text{if }i=K,\\
w_{K}^*, & \text{if }i=K+1.
\end{cases}\label{eq:def-wtilde}
\end{equation}
In words, $\tilde{\bm{w}}$ is obtained by swapping the scores of
the $K$th and the $(K+1)$th items in $\bm{w}^*$.  Clearly, these two score vectors
share the same generalized separation measure $\Delta_{K}^{*}$, although
the top-$K$ items in these two scenarios are not identical. It thus
suffices to bound the probability of error in distinguishing these
two score vectors given the data.

In the sequel, we denote by $\mathbb{P}_{0}$
and $\mathbb{P}_{1}$ the probability measures under the scores $\{w_{i}^*\}$
and $\{\tilde{w}_{i}\}$, respectively, and let $P_{\mathrm{e}}(\psi)$ represent the probability of error in distinguishing  $\mathbb{P}_{0}$ and $\mathbb{P}_{1}$ using a procedure $\psi$. 
In view of \cite[Theorem 2.2]{tsybakov2009introduction}, if 
\begin{equation}
\mathsf{TV}(\mathbb{P}_{0},\mathbb{P}_{1})\leq\epsilon\label{eq:TV-ub}
\end{equation}
for some fixed constant $0<\epsilon<1$, then  
\begin{equation}
\inf_{\psi}P_{\mathrm{e}}(\psi)\geq\frac{1-\epsilon}{2} . \label{eq:error-LB-1}
\end{equation}
Here,
$\mathsf{TV}(\mathbb{P}_{0},\mathbb{P}_{1})$ represents the total
variation distance between $\mathbb{P}_{0}$ and $\mathbb{P}_{1}$.

The next step then boils down to characterizing $\mathsf{TV}(\mathbb{P}_{0},\mathbb{P}_{1})$.
To this end, denoting by $\mathbb{P}_{0}^{i,j}$ (resp.~$\mathbb{P}_{1}^{i,j}$)
the distribution of the samples comparing items $i$ and $j$  under $\mathbb{P}_{0}$
(resp.~$\mathbb{P}_{1}$), we obtain that 
\begin{align}
\mathsf{TV}(\mathbb{P}_{0},\mathbb{P}_{1}) & =\mathsf{TV}\Big(\otimes_{i>j}\mathbb{P}_{0}^{i,j},\otimes_{i>j}\mathbb{P}_{1}^{i,j}\Big)\nonumber \\
 & \leq\mathsf{TV}\left(\otimes_{i:i\neq K}\mathbb{P}_{0}^{i,K},\otimes_{i:i\neq K}\mathbb{P}_{1}^{i,K}\right)+\mathsf{TV}\left(\otimes_{i:i\neq K+1}\mathbb{P}_{0}^{i,K+1},\otimes_{i:i\neq K+1}\mathbb{P}_{1}^{i,K+1}\right)\label{eq:TV-UB-2terms}\\
 & \leq\sqrt{\mathsf{KL}\left(\otimes_{i:i\neq K}\mathbb{P}_{0}^{i,K}\hspace{0.2em}\|\hspace{0.2em}\otimes_{i:i\neq K}\mathbb{P}_{1}^{i,K}\right)/2} \\
 & \qquad +\sqrt{\mathsf{KL}\left(\otimes_{i:i\neq K+1}\mathbb{P}_{1}^{i,K+1}\hspace{0.2em}\|\hspace{0.2em}\otimes_{i:i\neq K+1}\mathbb{P}_{0}^{i,K+1}\right)/2},\label{eq:KL-UB-2terms}
\end{align}
where $\mathsf{KL}(P \hspace{0.2em}\|\hspace{0.2em} Q)$ is the KL divergence from $Q$ to $P$. 
Here, (\ref{eq:TV-UB-2terms}) arises from two facts: (i) $\mathbb{P}_{0}$
and $\mathbb{P}_{1}$ differ only over locations within $\{(i,K)\mid i\neq K\}$
and $\{(i,K+1)\mid i\neq K+1\}$, and (ii) for any product measure
one has 
\[
\mathsf{TV}(P_{1}\otimes Q_{1},P_{2}\otimes Q_{2})\leq\mathsf{TV}(P_{1},P_{2})+\mathsf{TV}(Q_{1},Q_{2}).
\]
Additionally, the inequality (\ref{eq:KL-UB-2terms}) comes from Pinsker's
inequality \cite[Lemma 2.5]{tsybakov2009introduction}.

We then look at each term of (\ref{eq:KL-UB-2terms}) separately.
To begin with, repeating the analysis in \cite[Appendix B]{chen2015spectral}, we
can demonstrate (using the independence assumption) that 
\begin{align*}
	& \mathsf{KL}\left(\otimes_{i:i\neq K}\mathbb{P}_{0}^{i,K}\hspace{0.2em}\|\hspace{0.2em}\otimes_{i:i\neq K}\mathbb{P}_{1}^{i,K}\right)  
	=pL\sum_{i:i\neq K}\mathsf{KL}\left(\mathbb{P}_{0}(y_{i,K}^{(1)})\hspace{0.2em}\|\hspace{0.2em}\mathbb{P}_{1}(y_{i,K}^{(1)})\right)\\
 & \quad =pL\sum_{i:i\neq K}\mathsf{KL}\left(\mathsf{Bern}\left(\frac{w_{i}^*}{w_{i}^*+w_{K}^*}\right)\hspace{0.2em}\|\hspace{0.2em}\mathsf{Bern}\left(\frac{w_{i}^*}{w_{i}^*+w_{K+1}^*}\right)\right),
\end{align*}
where $\mathsf{Bern}(p)$ denotes the Bernoulli distribution with mean
$p$. Upper bounding the KL divergence via $\chi^{2}$ divergence
(see \cite[Lemma 2.7]{tsybakov2009introduction}), namely, 
\begin{align*}
	\mathsf{KL}\left(\mathsf{Bern}\left(p\right)\hspace{0.2em}\|\hspace{0.2em}\mathsf{Bern}\left(q\right)\right) &\leq\chi^{2}\left(\mathsf{Bern}\left(p\right)\hspace{0.2em}\|\hspace{0.2em}\mathsf{Bern}\left(q\right)\right) \\
&=\frac{(p-q)^{2}}{q}+\frac{(p-q)^{2}}{1-q}=\frac{(p-q)^{2}}{q(1-q)},
\end{align*}
we arrive at 
\begin{align*}
\mathsf{KL}\left(\otimes_{i:i\neq K}\mathbb{P}_{0}^{i,K}\hspace{0.2em}\|\hspace{0.2em}\otimes_{i:i\neq K}\mathbb{P}_{1}^{i,K}\right) & \leq pL\sum_{i=1}^{n}\frac{\left(\frac{w_{i}^*}{w_{i}^*+w_{K+1}^*}-\frac{w_{i}^*}{w_{i}^*+w_{K}^*}\right)^{2}}{\frac{w_{i}^*}{w_{i}^*+w_{K+1}^*}\frac{w_{K+1}^*}{w_{i}^*+w_{K+1}^*}}\\
 & =pL\frac{\left(w_{K}^*-w_{K+1}^*\right)^{2}}{w_{K+1}^*}\sum_{i=1}^{n}\frac{w_{i}^*}{\left(w_{i}^*+w_{K}^*\right)^{2}}.
\end{align*}
Similarly, one can derive 
\begin{align*}
\mathsf{KL}\left(\otimes_{i:i\neq K+1}\mathbb{P}_{1}^{i,K+1}\hspace{0.2em}\|\hspace{0.2em}\otimes_{i:i\neq K+1}\mathbb{P}_{0}^{i,K+1}\right) & \leq pL\frac{(w_{K}^*-w_{K+1}^*)^{2}}{w_{K+1}}\sum_{i=1}^{n}\frac{w_{i}^*}{\left(w_{i}^*+w_{K}^*\right)^{2}}.
\end{align*}
Put together the preceding bounds to reach 
\[
\mathsf{TV}(\mathbb{P}_{0},\mathbb{P}_{1})\leq\sqrt{2pL\frac{(w_{K}^*-w_{K+1}^*)^{2}}{w_{K+1}^*}\sum_{i=1}^{n}\frac{w_{i}^*}{\left(w_{i}^*+w_{K}^*\right)^{2}}}=\sqrt{2pLn}\Delta_{K}^{*}.
\]
As a consequence, if 
\[
2npL\Delta_{K}^{*2}\leq\epsilon^{2},
\]
one necessarily has $\mathsf{TV}(\mathbb{P}_{0},\mathbb{P}_{1})\leq\epsilon$,
which combined with (\ref{eq:error-LB-1}) yields $\inf_{\psi}P_{\mathrm{e}}(\psi)\geq\frac{1-\epsilon}{2}$
as claimed.

\section{Examples for the general $\kappa$ setting \label{sec:appendix_example}}

Recall that in Section \ref{sec:Extension}, we introduce a new metric $\Delta_{K}^{*}$. 
The following three examples shed some light on the potential effectiveness
of $\Delta_{K}^{*}$ as a fundamental information measure. 
\begin{itemize}
\itemsep0.5em
\item \emph{Case 1: }$\kappa=O\left(1\right)$. Under this circumstance,
it is easy to verify that
\[
\Delta_{K}^{*}\asymp\Delta_{K},
\]
and hence all our preceding results for spectral method and regularized MLE for $\kappa=O(1)$ continue to
hold with $\Delta_{K}$ replaced by $\Delta_{K}^{*}$. In this case, the new lower bound for sample complexity is slightly worse than the previous one (Theorem \ref{thm:lower-bound} in the main text) by a factor of $\log n$.
\item \emph{Case 2:} Suppose there are 100 items with $w_{1}^{*}=\cdots=w_{5}^{*}=10$,
$w_{6}^{*}=\cdots=w_{99}=5$ and $w_{100}^{*}=10^{-6}$. Our goal is to
find the top-5 ranked items. Intuitively, the presence of the 100th item should
not affect the hardness of top-5 ranking by much. This intuition
is well captured by our new metric in (\ref{eq:new-hardness-measure}) in the main text.
Observe that
\begin{align*}
%\frac{1}{100}\sum_{i=1}^{100}\frac{w_{6}^{*}w_{i}^{*}}{\left(w_{5}^{*}+w_{i}^{*}\right)^{2}}
	%&=\frac{w_{6}^{*}}{w_{5}^{*}}\cdot\frac{1}{100}\sum_{i=1}^{100}\frac{w_{i}^{*}/w_{5}^{*}}{\left(1+w_{i}^{*}/w_{5}^{*}\right)^{2}}\\
\left(\Delta_K^{*}\right)^{2}
=\frac{(w_5^{*}-w_6^*)^2}{w_{6}^{*2}} \frac{w_{6}^{*}}{w_{5}^{*}}\left(\frac{1}{100}\sum_{i=1}^{99}\frac{w_{i}^{*}/w_{5}^{*}}{\left(1+w_{i}^{*}/w_{5}^{*}\right)^{2}}+\frac{1}{100}\frac{w_{100}^{*}/w_{5}^{*}}{\left(1+w_{100}^{*}/w_{5}^{*}\right)^{2}}\right).
\end{align*}
Since $w_{100}^{*}/w_{5}^{*}$ is exceedingly small ($10^{-7}$ in this example),
it is easily seen that $\frac{w_{100}^{*}/w_{5}^{*}}{(1+w_{100}^{*}/w_{5}^{*})^{2}}\lesssim w_{100}^{*}/w_{5}^{*}$
is also extremely small, and hence\emph{ $\left(\Delta_{K}^{*}\right)^{2}$ }is not
changed by much compared with the case when the 100th item is absent ($\left(\Delta_K^{*}\right)^{2}\approx 0.1107$ (resp. 0.1118) 
for the case when the $100$th item is present (resp. absent)).
Similarly, consider the case when $w_{1}^{*}=10^{6}$, $w_{2}^{*}=\cdots =w_{5}^{*}=1$
and $w_{6}^{*}=\cdots =w_{100}^{*}=0.5$. As one can see,  adding the first item
will have little influence upon $\Delta_{K}^{*}$.
\item \emph{Case 3:} Consider finding the top-5 items out of 100 items
with $w_{1}^{*}=\cdots=w_{5}^{*}=10$, $w_{6}^{*}=\cdots=w_{10}^{*}=5$ and $w_{11}^{*}=\cdots w_{100}^{*}=10^{-6}$.
The sample complexity needed for exact top-5 recovery will surely
increase since the comparisons between $\left\{ 1,\cdots,10\right\} $
and $\left\{ 11,\cdots100\right\} $ are, with high probability, not useful  in determining
the relative strength within the group $\left\{ 1,\cdots,10\right\} $. This
is also reflected in the  generalized separation measure
$\Delta_{K}^{*}$. Recall that we have
\[
%\frac{1}{100}\sum_{i=1}^{100}\frac{w_{6}^{*}w_{i}^{*}}{\left(w_{5}^{*}+w_{i}^{*}\right)^{2}}=
	\left(\Delta_K^{*}\right)^{2} =\frac{(w_5^{*}-w_6^*)^2}{w_{6}^{*2}} \frac{w_{6}^{*}}{w_{5}^{*}}\left(\frac{1}{100}\sum_{i=1}^{10}\frac{w_{i}^{*}/w_{5}^{*}}{\left(1+w_{i}^{*}/w_{5}^{*}\right)^{2}}+\frac{1}{100}\sum_{i=11}^{100}\frac{w_{i}^{*}/w_{5}^{*}}{\left(1+w_{i}^{*}/w_{5}^{*}\right)^{2}}\right).
\]
For each $11\leq i\leq100$, $\frac{w_{i}^{*}/w_{5}^{*}}{\left(1+w_{i}^{*}/w_{5}^{*}\right)^{2}}$ is exceedingly small.
This makes $\left(\Delta_{K}^{*}\right)^{2}$ much smaller compared with the case
when only $\left\{ 1,\cdots10\right\} $ are present ($\left(\Delta_K^{*}\right)^{2}\approx 0.0118$ (resp. 0.1181) 
for the case when items $\{11,\cdots,100\}$ are present (resp. absent)). As a result,
the required sample size increases accordingly.
\end{itemize}

%% file: proof-spectral.tex
\section{Proofs in Section \ref{sec:Analysis-for-spectral}}

This section collects proofs of the theorems and lemmas that appear
in Section \ref{sec:Analysis-for-spectral}. 

Before moving on, we note that by Lemma \ref{lemma:degree} in the main text,
the event
\[
\cA_{0}=\left\{ \frac{1}{2}np\leq d_{\min}\leq d_{\max}\leq\frac{3}{2}np\right\}
\]
happens with probability at least $1-O(n^{-10})$. Throughout this
section, we shall assume that we are on this event without explicitly
referring to it each time. An immediate consequence is that $d_{\max}\leq d$
on this event.

\subsection{Proof of Theorem \ref{lemma:mc-perturbation}\label{sec:Proof-of-Lemma-mc-perturbation}}

To begin with, we write
\begin{align}
\bm{\pi}^{\top}-\hat{\bm{\pi}}^{\top} & =\bm{\pi}^{\top}\bm{P}-\hat{\bm{\pi}}^{\top}\hat{\bm{P}}=\bm{\pi}^{\top}(\bm{P}-\hat{\bm{P}})+\left(\bm{\pi}-\hat{\bm{\pi}}\right)^{\top}\hat{\bm{P}}.\label{eq:pi-pih-1}
\end{align}
The last term of the above identity can be further decomposed as
\begin{align}
\left(\bm{\pi}-\hat{\bm{\pi}}\right)^{\top}\hat{\bm{P}} & =\left(\bm{\pi}-\hat{\bm{\pi}}\right)^{\top}\bm{P}^{*}+\left(\bm{\pi}-\hat{\bm{\pi}}\right)^{\top}(\hat{\bm{P}}-\bm{P}^{*})\nonumber \\
 & =\left(\bm{\pi}-\hat{\bm{\pi}}\right)^{\top}\left(\bm{P}^{*}-\bm{1}\bm{\pi}^{*\top}\right)+\left(\bm{\pi}-\hat{\bm{\pi}}\right)^{\top}(\hat{\bm{P}}-\bm{P}^{*}),\label{eq:pi-pih-2}
\end{align}
where we have used the fact that $\left(\bm{\pi}-\hat{\bm{\pi}}\right)^{\top}\bm{1}\bm{\pi}^{*\top}=\bm{0}$.
Combining (\ref{eq:pi-pih-1}) and (\ref{eq:pi-pih-2}) we get
\begin{align*}
\bm{\pi}^{\top}-\hat{\bm{\pi}}^{\top} & =\bm{\pi}^{\top}(\bm{P}-\hat{\bm{P}})+\left(\bm{\pi}-\hat{\bm{\pi}}\right)^{\top}\left(\bm{P}^{*}-\bm{1}\bm{\pi}^{*\top}\right)+\left(\bm{\pi}-\hat{\bm{\pi}}\right)^{\top}(\hat{\bm{P}}-\bm{P}^{*}),
\end{align*}
which together with a little algebra gives
\begin{align*}
\left\Vert \bm{\pi}-\hat{\bm{\pi}}\right\Vert _{\bm{\pi}^{*}} & \leq\left\Vert \bm{\pi}^{\top}(\bm{P}-\hat{\bm{P}})\right\Vert _{\bm{\pi}^{*}}+\left\Vert \bm{\pi}-\hat{\bm{\pi}}\right\Vert _{\bm{\pi}^{*}}\left\Vert \bm{P}^{*}-\bm{1}\bm{\pi}^{*\top}\right\Vert _{\bm{\pi}^{*}}+\|\bm{\pi}-\hat{\bm{\pi}}\|_{\bm{\pi}^{*}}\left\Vert \hat{\bm{P}}-\bm{P}^{*}\right\Vert _{\bm{\pi}^{*}}
\end{align*}
\begin{align*}
\Longrightarrow\qquad\left\Vert \bm{\pi}-\hat{\bm{\pi}}\right\Vert _{\bm{\pi}^{*}} & \leq\frac{\big\Vert \bm{\pi}^{\top}(\bm{P}-\hat{\bm{P}})\big\Vert _{\bm{\pi}^{*}}}{1-\left\Vert \bm{P}^{*}-\bm{1}\bm{\pi}^{*\top}\right\Vert _{\bm{\pi}^{*}}-\big\|\hat{\bm{P}}-\bm{P}^{*}\big\|_{\bm{\pi}^{*}}}.
\end{align*}
The theorem follows by recognizing that
\[
\left\Vert \bm{P}^{*}-\bm{1}\bm{\pi}^{*\top}\right\Vert _{\bm{\pi}^{*}}=\max\left\{ \lambda_{2}(\bm{P}^{*}),-\lambda_{n}\left(\bm{P}^{*}\right)\right\} .
\]

\subsection{Proof of Theorem \ref{thm:L2-pi} \label{sec:Proof-of-Theorem-L2-pi}}

%\begin{proof}
By Theorem \ref{lemma:mc-perturbation},
we obtain
\begin{align*}
\|\bm{\pi^{*}}-\bm{\pi}\|_{\bm{\pi}^{*}} & \leq\frac{\big\|\bm{\pi}^{*\top}(\bm{P}^{*}-\bm{P})\big\|_{\bm{\pi}^{*}}}{1-\max\left\{ \lambda_{2}(\bm{P}^{*}),-\lambda_{n}\left(\bm{P}^{*}\right)\right\} -\big\|\bm{P}^{*}-\bm{P}\big\|_{\bm{\pi}^{*}}}\overset{(\text{i})}{\lesssim}\big\|\bm{\pi}^{*\top}(\bm{P}^{*}-\bm{P})\big\|_{\bm{\pi}^{*}}\\
& \overset{\left(\text{ii}\right)}{\lesssim}\|\bm{\pi}^{*\top}(\bm{P}^{*}-\bm{P})\|_{2}
%\lesssim\|\bm{P}-\bm{P}^{*}\|_{2}\|\bm{\pi}^{*}\|_{2}
\overset{(\text{iii})}{\lesssim}  \frac{1}{\sqrt{npL}} \|\bm{\pi}^{*}\|_{2},
\end{align*}
where (i) is a consequence of Lemma \ref{lemma:spectral-gap}, (ii)
follows from the relationship between $\|\cdot\|_{2}$ and $\|\cdot\|_{\bm{\pi}^{*}}$,
and (iii) follows as long as one can justify that
\begin{align}
\label{eq:UB1}
\left\Vert \bm{\pi}^{*\top}\left(\bm{P}-\bm{P}^{*}\right)\right\Vert _{2} \lesssim \frac{1}{\sqrt{npL}} \|{\bm{\pi}^{*}}\|_2.
\end{align}
Therefore, the rest of the proof is devoted to establishing (\ref{eq:UB1}). 
To simplify the notations hereafter, we denote $\bm{\Delta}:=\bm{P}-\bm{P}^{*}$. In fact, it is easy to check that for any $i\ne j$, 
\begin{equation}
\Delta_{i,j}=P_{i,j}-P_{i,j}^{*}=\ind_{\left\{ \left(i,j\right)\in\cE\right\} }\frac{1}{d}\left(\frac{1}{L}\sum_{l=1}^{L}y_{i,j}^{\left(l\right)}-y_{i,j}^{*}\right)\label{eq:ell-2-off-diag}
\end{equation}
and for $1\leq i\leq n$, one has 
\begin{align}
\Delta_{i,i} & =P_{i,i}-P_{i,i}^{*}=\left(1-\sum_{j:j\neq i}P_{i,j}\right)-\left(1-\sum_{j:j\neq i}P_{i,j}^{*}\right)\nonumber \\
& =-\sum_{j:j\neq i}\left(P_{i,j}-P_{i,j}^{*}\right) = -\sum_{j:j\neq i}\Delta_{i,j} \nonumber \\
& = - \Delta_{i,i}^{\mathrm{lower}} - \Delta_{i,i}^{\mathrm{upper}}.\label{eq:ell-2-diag}
\end{align}
where
\[
\Delta_{i,i}^{\mathrm{lower}}:= \sum_{j:j< i}\Delta_{i,j} \qquad \text{and} \qquad
\Delta_{i,i}^{\mathrm{upper}} := \sum_{j:j> i}\Delta_{i,j}. 
\]

Towards proving (\ref{eq:UB1}), we decompose $\bm{\Delta}$ into four parts $\bm{\Delta}=\bm{\Delta}_{\mathrm{lower}} + \bm{\Delta}_{\mathrm{upper}} + \bm{\Delta}_{\mathrm{diag,l}} + \bm{\Delta}_{\mathrm{diag,u}}$, where $\bm{\Delta}_{\mathrm{lower}}$ is the lower triangular part (excluding
the diagonal) of $\bm{\Delta}$, $\bm{\Delta}_{\mathrm{upper}}$
is the upper triangular part, and 
\[
\bm{\Delta}_{\mathrm{diag,l}} = -\mathrm{diag}\Big([\Delta_{i,i}^{\mathrm{lower}}]_{1\leq i\leq n}\Big)
\qquad \text{and} \qquad
\bm{\Delta}_{\mathrm{diag,u}} = -\mathrm{diag}\Big([\Delta_{i,i}^{\mathrm{upper}}]_{1\leq i\leq n}\Big). 
\]
The triangle inequality then gives
\[
\left\Vert \bm{\pi}^{*\top}\bm{\Delta}\right\Vert _{2}\leq\underbrace{\left\Vert \bm{\pi}^{*\top}\bm{\Delta}_{\mathrm{lower}}\right\Vert _{2}}_{:=I_{\mathrm{lower}}}
+ \underbrace{\left\Vert \bm{\pi}^{*\top}\bm{\Delta}_{\mathrm{upper}}\right\Vert _{2}}_{:=I_{\mathrm{upper}}}
+ \underbrace{\left\Vert \bm{\pi}^{*\top}\bm{\Delta}_{\mathrm{diag,l}}\right\Vert _{2}}_{:=I_{\mathrm{diag,l}}}
+ \underbrace{\left\Vert \bm{\pi}^{*\top}\bm{\Delta}_{\mathrm{diag,u}}\right\Vert _{2}}_{:=I_{\mathrm{diag,u}}}. 
\]
%The main point of this decomposition is that all entries of $\bm{\Delta}_{\mathrm{lower}}$ (resp.~$\bm{\Delta}_{\mathrm{upper}}$) are statistically 
In what follows, we will focus on controlling the first term $I_{\mathrm{lower}}$. The other three terms can be bounded using nearly identical arguments. 

Note that the $j$  component of $\bm{\pi}^{*\top}\bm{\Delta}_{\mathrm{lower}}$  can be expressed as 
\begin{align*}
\left[\bm{\pi}^{*\top}\bm{\Delta}_{\mathrm{lower}}\right]_{j} 
& =  \sum_{i:i>j}\pi_{i}^{*}\Delta_{i,j}. 
%	\pi_{j}^{*}\Delta_{j,j} + \\
%& \overset{\left(\text{i}\right)}{=} -\sum_{i:i\neq j}\pi_{j}^{*}\Delta_{j,i} + \sum_{i:i>j}\pi_{i}^{*}\Delta_{i,j}\\
%& \overset{\left(\text{ii}\right)}{=} \sum_{i:i\neq j}\pi_{j}^{*}\Delta_{i,j} + \sum_{i:i>j}\pi_{i}^{*}\Delta_{i,j},
\end{align*}
%where (i) follows from (\ref{eq:ell-2-diag}) and (ii) utilizes the
%identity $\Delta_{j,i}=-\Delta_{i,j}$ for $i\neq j$. Rearranging terms, we are left with 
%\[
%\left[\bm{\pi}^{*\top}\bm{\Delta}_{\mathrm{lower}}\right]_{j}=\sum_{i:i<j}\pi_{j}^{*}\Delta_{i,j}+\sum_{i:i>j}\left(\pi_{i}^{*}+\pi_{j}^{*}\right)\Delta_{i,j}.
%\]
Recall that for any pair $\left(i,j\right)\in\cE$, $\Delta_{i,j}$  is a sum of $L$  independent zero-mean random variables, and hence
$\left[\bm{\pi}^{*\top}\bm{\Delta}_{\mathrm{lower}}\right]_{j}$ is a sum of $Ld_j^{\mathrm{lower}}$ independent zero-mean random variables, where
\[
d_j^{\mathrm{lower}} := \big| \{(i,j)\mid (i,j)\in \mathcal{E} ~\text{and}~ i>j\} \big|. 
\]
In view of Hoeffding's inequality (Lemma \ref{lemma:hoeffding}), one has,  when conditional on $\mathcal{G}$, that
\[
\PP\left(\left|\left[\bm{\pi}^{*\top}\bm{\Delta}_{\mathrm{lower}}\right]_{j}\right|\geq t\right)
\leq
2\exp\left(-\frac{2t^2}{ \frac{1}{(Ld)^2} d_j^{\mathrm{lower}} L\left(2\left\Vert \bm{\pi}^{*}\right\Vert _{\infty}\right)^{2}}\right) .
% = 2\exp\left(-\frac{L d^2 t^2}{ 2 d_j^{\mathrm{lower}} \left\Vert \bm{\pi}^{*}\right\Vert _{\infty} ^{2}}\right).
\]
Hence $\left[\bm{\pi}^{*\top}\bm{\Delta}_{\mathrm{lower}}\right]_{j}$
can be treated as a sub-Gaussian random variable with variance proxy
\[
\sigma^{2}  \asymp\frac{d_j^{\mathrm{lower}}}{d^2 L}\left\Vert \bm{\pi}^{*}\right\Vert _{\infty}^{2}  \lesssim \frac{1}{dL}\left\Vert \bm{\pi}^{*}\right\Vert _{\infty}^{2}.
\]
Given that the entries of $\bm{\pi}^{*\top}\bm{\Delta}_{\mathrm{lower}}$
are independent, we see that
\[
\left(I_{\text{lower}}\right)^{2}=\sum_{j=1}^{n}\left[\bm{\pi}^{*\top}\bm{\Delta}_{\mathrm{lower}}\right]_{j}^{2},
\]
is a quadratic form of a sub-Gaussian vector. On the one hand, $\mathbb{E} [I_{\text{lower}}]^2 \lesssim n \sigma^2$. On the other hand, we invoke \cite[Theorem 1.1]{rudelson2013hanson} to reach
\[
\mathbb{P}\{  I_{\text{lower}}^{2} - \mathbb{E}[I_{\text{lower}}^{2}] \geq t \} \leq 2\exp \left\{ -c \min\left\{ \frac{t^2}{n\sigma^4}, \frac{t}{\sigma^2} \right\} \right\}
\]
for some constant $c>0$. By choosing $t \asymp \sigma^2 \sqrt{ n \log n}$, we see that with probability at least $1-O(n^{-10})$,
\[
I_{\text{lower}}^{2} \lesssim \mathbb{E}[I_{\text{lower}}^{2}] + \sigma^2 \sqrt{n\log n} \lesssim n\sigma^2 
\lesssim \frac{n}{dL}\left\Vert \bm{\pi}^{*}\right\Vert _{\infty}^{2}
\lesssim \frac{1}{dL}\left\Vert \bm{\pi}^{*}\right\Vert _{2}^{2}. 
\]
The same upper bounds can be derived for other terms using the same arguments. We
have thus established (\ref{eq:UB1}) by recognizing that $d\gtrsim np$.

\subsection{Proof of Lemma \ref{lemma:spectral_I1m}\label{sec:Proof-of-Lemma-spectral_I1m}}

Observe that
\begin{align}
I_{1}^{m} & =\sum_{j:j\neq m}\pi_{j}^{*}\left(P_{j,m}-P_{j,m}^{*}\right)+\pi_{m}^{*}\left(P_{m,m}-P_{m,m}^{*}\right)\nonumber \\
 & \overset{\left(\text{i}\right)}{=}\sum_{j:j\neq m}\pi_{j}^{*}\left(P_{j,m}-P_{j,m}^{*}\right)+\pi_{m}^{*}\Bigg\{ \Bigg(1-\sum_{j:j\neq m}P_{m,j}\Bigg)-\Bigg(1-\sum_{j:j\neq m}P_{m,j}^{*}\Bigg)\Bigg\} \nonumber \\
 & =\sum_{j:j\neq m}\left(\pi_{j}^{*}+\pi_{m}^{*}\right)\left(P_{j,m}-P_{j,m}^{*}\right)\nonumber \\
 & =\frac{1}{Ld}\sum_{j:j\neq m}\sum_{l=1}^{L}\left(\pi_{j}^{*}+\pi_{m}^{*}\right)\ind_{(j,m)\in\cE}\left(y_{j,m}^{(l)}-y_{j,m}^{*}\right),\label{eq:I1m}
\end{align}
 where $\left(\text{i}\right)$ follows from the fact that $\bm{P}$
and $\bm{P}^{*}$ are both probability transition matrices. By Lemma
\ref{lemma:hoeffding}, one can derive
\begin{align*}
\PP\Big\{|I_{1}^{m}|\geq t \Big| \cG \Big\} & =\PP\left\{ \left|\sum_{j:j\neq m}\sum_{l=1}^{L}\left(\pi_{j}^{*}+\pi_{m}^{*}\right)\ind_{(j,m)\in\cE}\left(y_{j,m}^{(l)}-\mathbb{}y_{j,m}^{*}\right)\right|\geq Ldt\text{}\Bigg|\text{ }\cG\right\} \\
 & \leq\text{ }2\exp\left(-\frac{2(Ldt)^{2}}{Ld_{\max}(2\|\bm{\pi}^{*}\|_{\infty})^{2}}\right).
%\\& \leq\text{ }2\exp\left(\frac{-Ldt^{2}}{2\|\bm{\pi}^{*}\|_{\infty}^{2}}\right).
\end{align*}
When $d_{\max}\leq d$, the right hand side is bounded by $2\exp\left(\frac{-Ldt^{2}}{2\|\bm{\pi}^{*}\|_{\infty}^{2}}\right)$. Hence
\[
\PP\left\{ |I_{1}^{m}|\geq4\sqrt{\frac{\log n}{Ld}}\|\bm{\pi}^{*}\|_{\infty}
\bigg| d_{\max} \leq d
\right\} \leq2n^{-8}.
\]
The lemma is established by taking the union bounds and using the fact that $\PP ( d_{\max}\leq d )\geq 1-O(n^{-10})$ (by Lemma \ref{lemma:degree} in the main text and the remarks after that).

\subsection{Proof of Lemma \ref{lemma:spectral_I2m}\label{sec:Proof-of-Lemma-spectral_I2m}}
Similar to the proof of Lemma \ref{lemma:spectral_I1m} above, by applying Lemma \ref{lemma:hoeffding} to the quantity
\[
P_{m,m}-P_{m,m}^{*}=-\sum_{j:j\neq m}(P_{m,j}-P_{m,j}^{*})=-\frac{1}{Ld}\sum_{j:j\neq m}\sum_{l=1}^{L}\ind_{(j,m)\in\cE}\left(y_{j,m}^{(l)}-\mathbb{}y_{j,m}^{*}\right),
\]
we get
\[
\PP\left\{ \max_{m}|P_{m,m}-P_{m,m}^{*}|\geq2\sqrt{\frac{\log n}{Ld}}\text{ }
\bigg| d_{\max} \leq d
\right\} \leq2n^{-7}.
\]
On the other hand, when the event $\cA_0$ (defined in Lemma \ref{lemma:degree} in the main text) happens, we have $d_{\min}\geq np/2$ and for all $1\leq m\leq n$,
\[
P_{m,m}^{*}=1-\sum_{j:j\neq m}P_{m,j}^{*}\leq1-\frac{d_{\min}}{d}\cdot\frac{1}{1+\kappa}\leq1-\frac{np}{2d}\cdot\frac{1}{1+\kappa}.
\]
Combining these two pieces completes the proof.

\subsection{Proof of Lemma \ref{lemma:spectral-leave-ell-2}\label{sec:Proof-of-Lemma-spectral_leave-ell-2}}

First, by the relationship between $\|\cdot\|_{2}$ and $\|\cdot\|_{\bm{\pi}^{*}}$,
we have
\[
\|\bm{\pi}^{\left(m\right)}-\bm{\pi}\|_{2}\leq\frac{1}{\sqrt{\pi_{\min}^{*}}}\|\bm{\pi}^{\left(m\right)}-\bm{\pi}\|_{\pi^{*}},
\]
where $\pi^*_{\min} := \min_{i}\pi_i^*$. 
Invoking Theorem \ref{lemma:mc-perturbation}, we obtain
\begin{align*}
\|\bm{\pi}^{\left(m\right)}-\bm{\pi}\|_{\bm{\pi}^{*}} & \leq\frac{\|\bm{\pi}^{\left(m\right)\top}(\bm{P}^{(m)}-\bm{P})\|_{\bm{\pi}^{*}}}{1-\max\left\{ \lambda_{2}(\bm{P}^{*}),-\lambda_{n}\left(\bm{P}^{*}\right)\right\} -\|\bm{P}-\bm{P}^{*}\|_{\bm{\pi}^{*}}}\\
 & \leq\frac{1}{\gamma}\sqrt{\pi_{\max}^{*}}\|\bm{\pi}^{\left(m\right)\top}(\bm{P}^{(m)}-\bm{P})\|_{2},
\end{align*}
 where we define $\gamma:=1-\max\left\{ \lambda_{2}(\bm{P}^{*}),-\lambda_{n}\left(\bm{P}^{*}\right)\right\} -\|\bm{P}-\bm{P}^{*}\|_{\bm{\pi}^{*}}$ and $\pi^*_{\max} := \max_{i}\pi_i^*$.

To facilitate the analysis of $\|\bm{\pi}^{\left(m\right)\top}(\bm{P}^{(m)}-\bm{P})\|_{2}$,
we introduce another Markov chain with transition probability matrices
$\bm{P}^{(m),\cG}$, which is also a leave-one-out version of the
transition matrix $\bm{P}$. Similar to $\bm{P}^{(m)}$, $\bm{P}^{(m),\cG}$
replaces all the transition probabilities involving the $m$-th item
with their expected values (conditional on $\mathcal{G}$). Concretely,
for $i\neq j$,
\[
P_{i,j}^{(m),\cG}=\begin{cases}
P_{i,j}, & \quad i\neq m,j\neq m,\\
\frac{1}{d}y_{i,j}^{*}\ind_{(i,j)\in\cE}, & \quad i=m\text{ or }j=m.
\end{cases}
\]
And for each $1\leq i\leq n$, we define
\[
P_{i,i}^{(m),\cG}=1-\sum_{j:j\neq i}P_{i,j}^{(m),\cG}
\]
to make $\bm{P}^{(m),\cG}$ a valid probability transition matrix.
Hence by the triangle inequality, we see that
\[
\|\bm{\pi}^{\left(m\right)\top}(\bm{P}^{(m)}-\bm{P})\|_{2}\leq\underbrace{\|\bm{\pi}^{\left(m\right)\top}(\bm{P}-\bm{P}^{(m),\cG})\|_{2}}_{:=J_{1}^{m}}+\underbrace{\|\bm{\pi}^{\left(m\right)\top}(\bm{P}^{(m)}-\bm{P}^{(m),\cG})\|_{2}}_{:=J_{2}^{m}}.
\]
The next step is then to bound $J_{1}^{m}$ and $J_{2}^{m}$ separately.

For $J_{1}^{m}$, similar to \eqref{eq:I1m}, one has
\begin{align*}
\left[\bm{\pi}^{\left(m\right)\top}(\bm{P}-\bm{P}^{(m),\cG})\right]_{m} & \overset{\left(\text{i}\right)}{=}\left[\bm{\pi}^{\left(m\right)\top}(\bm{P}-\bm{P}^{*})\right]_{m}\\
 & =\frac{1}{Ld}\sum_{j:j\neq m}\sum_{l=1}^{L}\left(\pi_{j}^{\left(m\right)}+\pi_{m}^{\left(m\right)}\right)\ind_{(j,m)\in\cE}\left(y_{j,m}^{(l)}-y_{j,m}^{*}\right),
\end{align*}
 where $\left(\text{i}\right)$ comes from the fact that $\bm{P}_{\cdot m}^{(m),\cG}=\bm{P}_{\cdot m}^{*}$.
Recognizing that $\bm{\pi}^{\left(m\right)}$ is statistically independent
of $\left\{ y_{j,m}\right\} _{j\neq m}$, by Hoeffding's inequality
in Lemma \ref{lemma:hoeffding}, we get
\begin{equation}
\PP\left\{ \left|[\bm{\pi}^{\left(m\right)\top}(\bm{P}-\bm{P}^{(m),\cG})]_{m}\right|\geq4\sqrt{\frac{\log n}{Ld}}\|\bm{\pi}^{\left(m\right)}\|_{\infty}\right\} \leq2n^{-8}.\label{eq:pim-m-1}
\end{equation}
 And for $j\neq m$, we have
\begin{align*}
\left[\bm{\pi}^{\left(m\right)\top}\big(\bm{P}-\bm{P}^{(m),\cG}\big)\right]_{j} & =\sum_{i}\pi_{i}^{(m)}\big(P_{i,j}-P_{i,j}^{(m),\cG}\big)\\
 & =\pi_{j}^{(m)}\big(P_{j,j}-P_{j,j}^{(m),\cG}\big)+\pi_{m}^{(m)}\big(P_{m,j}-P_{m,j}^{(m),\cG}\big)\\
 & =\pi_{j}^{(m)}\frac{1}{d}(y_{j,m}^{*}-y_{j,m})\ind_{(j,m)\in\cE}\text{ }+\text{ }\pi_{m}^{(m)}\frac{1}{d}(y_{m,j}-y_{m,j}^{*})\ind_{(j,m)\in\cE}\\
 & =\pi_{j}^{(m)}\left(P_{j,m}^{*}-P_{j,m}\right)+\pi_{m}^{(m)}\left(P_{m,j}-P_{m,j}^{*}\right).
\end{align*}
In addition by Hoeffding's inequality in Lemma \ref{lemma:hoeffding},
we have
\[
\max_{j\neq m}\left|P_{j,m}-P_{j,m}^{*}\right|\leq\frac{2}{d}\sqrt{\frac{\log n}{L}}
\]
with probability at least $1-O(n^{-5})$. As a consequence,
\begin{equation}
\left|\left[\bm{\pi}^{\left(m\right)\top}\big(\bm{P}-\bm{P}^{(m),\cG}\big)\right]_{j}\right|\leq\begin{cases}
\frac{4}{d}\sqrt{\frac{\log n}{L}}\|\bm{\pi}^{\left(m\right)}\|_{\infty}, & \quad\text{ if }\left(j,m\right)\in\cE,\\
0, & \quad\text{ else}.
\end{cases}\label{eq:pim-j}
\end{equation}
Combining \eqref{eq:pim-m-1} and \eqref{eq:pim-j} yields
\[
J_{1}^{m}\leq4\sqrt{\frac{\log n}{Ld}}\|\bm{\pi}^{\left(m\right)}\|_{\infty}+\frac{4\sqrt{d_{\max}-1}}{d}\sqrt{\frac{\log n}{L}}\|\bm{\pi}^{\left(m\right)}\|_{\infty}\overset{(\text{i})}{\leq}8\sqrt{\frac{\log n}{Ld}}\|\bm{\pi}^{\left(m\right)}\|_{\infty},
\]
 where $(\text{i})$ comes from the fact that $d_{\max}\leq d$.

Regarding $J_{2}^{m}$, we invoke the identity $\bm{\pi}^{*\top}(\bm{P}^{(m)}-\bm{P}^{(m),\cG})=\bm{0}$
to get
\[
\bm{\pi}^{\left(m\right)\top}(\bm{P}^{(m)}-\bm{P}^{(m),\cG})=\left(\bm{\pi}^{\left(m\right)}-\bm{\pi}^{*}\right)^{\top}(\bm{P}^{(m)}-\bm{P}^{(m),\cG}).
\]
Therefore, for $j\neq m$ we have
\begin{align*}
 & \left[\left(\bm{\pi}^{\left(m\right)}-\bm{\pi}^{*}\right)^{\top}\big(\bm{P}^{(m)}-\bm{P}^{(m),\cG}\big)\right]_{j}\\
 & \qquad=\sum_{i}(\pi_{i}^{(m)}-\pi_{i}^{*})\big(P_{i,j}^{(m)}-P_{i,j}^{(m),\cG}\big)\\
 & \qquad=(\pi_{j}^{(m)}-\pi_{j}^{*})\big(P_{j,j}^{(m)}-P_{j,j}^{(m),\cG}\big)+(\pi_{m}^{(m)}-\pi_{m}^{*})\big(P_{m,j}^{(m)}-P_{m,j}^{(m),\cG}\big)\\
 & \qquad=-(\pi_{j}^{(m)}-\pi_{j}^{*})\big(P_{j,m}^{(m)}-P_{j,m}^{(m),\cG}\big)+(\pi_{m}^{(m)}-\pi_{m}^{*})\big(P_{m,j}^{(m)}-P_{m,j}^{(m),\cG}\big).
\end{align*}
 Recognizing that $\big|P_{j,m}^{(m)}-P_{j,m}^{(m),\cG}\big|\leq\frac{2}{d}$
for $\left(j,m\right)\in\cE$ and $\big|P_{j,m}^{(m)}-P_{j,m}^{(m),\cG}\big|\leq\frac{p}{d}$
for $\left(j,m\right)\notin\cE$, we have
\begin{equation}
\left|\left[\left(\bm{\pi}^{\left(m\right)}-\bm{\pi}^{*}\right)(\bm{P}^{(m)}-\bm{P}^{(m),\cG})\right]_{j}\right|\leq\begin{cases}
\frac{4}{d}\|\bm{\pi}^{\left(m\right)}-\bm{\pi}^{*}\|_{\infty}, & \quad\text{ if }\left(j,m\right)\in\cE,\\
\frac{2p}{d}\|\bm{\pi}^{\left(m\right)}-\bm{\pi}^{*}\|_{\infty}, & \quad\text{ else}.
\end{cases}\label{eq:J2m-j}
\end{equation}
 And for $j=m$, it holds that
\begin{align}
 & \left|\left[\left(\bm{\pi}^{\left(m\right)}-\bm{\pi}^{*}\right)^{\top}\big(\bm{P}^{(m)}-\bm{P}^{(m),\cG}\big)\right]_{m}\right|\nonumber \\
 & \qquad=\left|(\pi_{m}^{(m)}-\pi_{m}^{*})\big(P_{m,m}^{(m)}-P_{m,m}^{(m),\cG}\big)+\sum_{j:j\neq m}(\pi_{j}^{(m)}-\pi_{j}^{*})\big(P_{j,m}^{(m)}-P_{j,m}^{(m),\cG}\big)\right|\\
 & \qquad\leq\left|(\pi_{m}^{(m)}-\pi_{m}^{*})\big(P_{m,m}^{(m)}-P_{m,m}^{(m),\cG}\big)\right|+\Bigg|\sum_{j:j\neq m}(\pi_{j}^{(m)}-\pi_{j}^{*})\big(P_{j,m}^{(m)}-P_{j,m}^{(m),\cG}\big)\Bigg|\nonumber \\
 & \qquad=\Bigg|\underbrace{\sum_{j:j\neq m}(\pi_{m}^{(m)}-\pi_{m}^{*})\big(P_{m,j}^{(m)}-P_{m,j}^{(m),\cG}\big)}_{:=J_{3}^{m}}\Bigg|+\Bigg|\underbrace{\sum_{j:j\neq m}(\pi_{j}^{(m)}-\pi_{j}^{*})\big(P_{j,m}^{(m)}-P_{j,m}^{(m),\cG}\big)}_{:=J_{4}^{m}}\Bigg|.\label{eq:J2m-m}
\end{align}
Given that $P_{m,j}^{(m)}-P_{m,j}^{(m),\cG}=\frac{y_{m,j}^{*}}{d}(p-\ind_{(m,j)\in\cE})$,
we have
\begin{equation}
J_{3}^{m}=\sum_{j:j\neq m}\underbrace{(\pi_{m}^{(m)}-\pi_{m}^{*})\frac{y_{m,j}^{*}}{d}}_{:=\xi_{j}^{(m)}}(p-\ind_{(m,j)\in\cE}).\label{eq:J3m}
\end{equation}
 Since $\|\bm{\xi}^{(m)}\|_{\infty}\leq\frac{1}{d}\|\bm{\pi}^{(m)}-\bm{\pi}^{*}\|_{\infty}$
and $\|\bm{\xi}^{(m)}\|_{2}\leq\frac{1}{d}\|\bm{\pi}^{(m)}-\bm{\pi}^{*}\|_{2}$,
Lemma \ref{lemma:bernstein} implies that
\[
|J_{3}^{m}|\lesssim\frac{\sqrt{np\log n}+\log n}{d}\|\bm{\pi}^{(m)}-\bm{\pi}^{*}\|_{\infty}
\]
with high probability. The same bound holds for $J_{4}^{m}$. Combine
\eqref{eq:J2m-j}, \eqref{eq:J2m-m} and \eqref{eq:J3m} to arrive
at
\[
J_{2}^{m}\lesssim\left(\frac{\sqrt{np\log n}+\log n}{d}+\frac{p\sqrt{n}}{d}+\frac{\sqrt{d}}{d}\right)\|\bm{\pi}^{(m)}-\bm{\pi}^{*}\|_{\infty}.
\]

Combining all, we deduce that{
\medmuskip=0mu
\thinmuskip=0mu
\thickmuskip=0mu
\begin{align*}
\|\bm{\pi}^{(m)}-\bm{\pi}\|_{2} & \leq\frac{\sqrt{\kappa}}{\gamma}(J_{1}^{m}+J_{2}^{m})\\
 & \leq\frac{\sqrt{\kappa}}{\gamma}\left(8\sqrt{\frac{\log n}{Ld}}\|\bm{\pi}^{(m)}\|_{\infty}+C\left(\frac{\sqrt{np\log n}+\log n}{d}+\frac{p\sqrt{n}}{d}+\frac{\sqrt{d}}{d}\right)\|\bm{\pi}^{(m)}-\bm{\pi}^{*}\|_{\infty}\right)\\
 & \leq\frac{\sqrt{\kappa}}{\gamma}\left(8\sqrt{\frac{\log n}{Ld}}\|\bm{\pi}^{*}\|_{\infty}+C\left(8\sqrt{\frac{\log n}{Ld}}+\frac{\sqrt{np\log n}+\log n}{d}+\frac{p\sqrt{n}}{d}+\frac{\sqrt{d}}{d}\right)\|\bm{\pi}^{(m)}-\bm{\pi}^{*}\|_{\infty}\right)\\
 & \overset{(\text{i})}{\leq}\frac{8\sqrt{\kappa}}{\gamma}\sqrt{\frac{\log n}{Ld}}\|\bm{\pi}^{*}\|_{\infty}+\frac{1}{2}\|\bm{\pi}^{(m)}-\bm{\pi}^{*}\|_{\infty},
\end{align*}}
where $(\text{i})$ holds as long as $np \gamma^2 \geq c \kappa \log n$ for $c$
sufficiently large. The triangle inequality
\[
\|\bm{\pi}^{(m)}-\bm{\pi}^{*}\|_{\infty}\leq\|\bm{\pi}^{(m)}-\bm{\pi}\|_{2}+\|\bm{\pi}-\bm{\pi}^{*}\|_{\infty}
\]
yields
\begin{equation}
\|\bm{\pi}^{(m)}-\bm{\pi}\|_{2}\leq\frac{16\sqrt{\kappa}}{\gamma}\sqrt{\frac{\log n}{Ld}}\|\bm{\pi}^{*}\|_{\infty}+\|\bm{\pi}-\bm{\pi}^{*}\|_{\infty},\label{eq:ell2_2}
\end{equation}
which concludes the proof.

\subsection{Proof of Lemma \ref{lemma:spectral-I4m}\label{sec:Proof-of-Lemma-spectral-I4m}}

For ease of presentation, we define
\[
\tilde{y}_{i,j}:=\frac{1}{L}\sum_{l=1}^{L}\tilde{y}_{i,j}^{\left(l\right)}
\]
for all $i\neq j$, where
\[
\tilde{y}_{i,j}^{(l)}\text{ }\overset{\text{ind.}}{=}\text{ }\begin{cases}
1,\quad & \text{with probability }\frac{w_{j}^{*}}{w_{i}^{*}+w_{j}^{*}},\\
0, & \text{else}.
\end{cases}
\]
This allows us to write $\bm{y}$ as $y_{i,j}=\tilde{y}_{i,j}\ind_{(i,j)\in\cE}$
for all $i\neq j$. With this notation in place, we can obtain

\[
I_{4}^{m}=\sum_{j:j\neq m}(\pi_{j}^{\left(m\right)}-\pi_{j}^{*})\left(\frac{1}{Ld}\sum_{l=1}^{L}\tilde{y}_{j,m}^{(l)}\right)\ind_{\left(j,m\right)\in\cE}.
\]
We can further decompose $I_{4}^{m}$ into
\[
I_{4}^{m}=\EE\left[I_{4}^{m}\mid\cG_{-m},\tilde{\bm{y}}\right]+I_{4}^{m}-\EE\left[I_{4}^{m}\mid\cG_{-m},\bm{\tilde{y}}\right],
\]
where $\cG_{-m}$ represent the graph without the $m$-th node, and
$\tilde{\bm{y}}=\{\tilde{y}_{i,j}|i\neq j\}$ represents all the binary
outcomes.

We start with the expectation term
\begin{align*}
\EE\left[I_{4}^{m}\mid\cG_{-m},\tilde{\bm{y}}\right] & =\sum_{j:j\neq m}(\pi_{j}^{\left(m\right)}-\pi_{j}^{*})\left(\frac{1}{Ld}\sum_{l=1}^{L}\tilde{y}_{j,m}^{(l)}\right)\PP\left\{ {\left(j,m\right)\in\cE}\right\} \\
 & \overset{\left(\text{i}\right)}{\leq}\text{ }p\|\bm{\pi}^{\left(m\right)}-\bm{\pi}^{*}\|_{2}\frac{\sqrt{n}}{d}\text{ }\overset{\left(\text{ii}\right)}{\leq}\text{ }\frac{1}{2\sqrt{n}}\|\bm{\pi}^{\left(m\right)}-\bm{\pi}^{*}\|_{2}\\
 & \overset{\left(\text{iii}\right)}{\leq}\frac{1}{2\sqrt{n}}(\|\bm{\pi}^{\left(m\right)}-\bm{\pi}\|_{2}+\|\bm{\pi}-\bm{\pi}^{*}\|_{2}),
\end{align*}
 where $\left(\text{i}\right)$ comes from the Cauchy-Schwarz inequality,
$\left(\text{ii}\right)$ follows from the choice $d=c_{d}np\geq2np$
and $\left(\text{iii}\right)$ results from the triangle inequality.
By Theorem \ref{thm:L2-pi}, with high probability we have
\[
\|\bm{\pi}-\bm{\pi}^{*}\|_{2}\leq C_{N}\sqrt{\frac{\log n}{Ld}}\|\bm{\pi}^{*}\|_{2}\leq C_{N}\sqrt{\frac{\log n}{Ld}}\sqrt{n}\|\bm{\pi}^{*}\|_{\infty},
\]
thus indicating that
\begin{align}
\EE\left[I_{4}^{m}\mid\cG_{-m},\bm{\tilde{y}}\right] & \leq\frac{1}{2\sqrt{n}}\|\bm{\pi}^{\left(m\right)}-\bm{\pi}\|_{2}+\frac{C_{N}}{2}\sqrt{\frac{\log n}{Ld}}\|\bm{\pi}^{*}\|_{\infty}.\label{eq:I4-expectation}
\end{align}

When it comes to the fluctuation term, one can write
\[
I_{4}^{m}-\EE[I_{4}^{m}\mid\cG_{-m},\tilde{\bm{y}}]=\sum_{j:j\neq m}\underbrace{(\pi_{j}^{\left(m\right)}-\pi_{j}^{*})\left(\frac{1}{Ld}\sum_{l=1}^{L}\tilde{y}_{j,m}^{(l)}\right)}_{:=\beta_{j}^{(m)}}\left(\ind_{\left(j,m\right)\in\cE}-\PP\left\{ {\left(j,m\right)\in\cE}\right\} \right).
\]
 Since $\|\bm{\beta}^{(m)}\|_{2}\leq\frac{1}{d}\|\bm{\pi}^{(m)}-\bm{\pi}^{*}\|_{2}$
and $\|\bm{\beta}^{(m)}\|_{\infty}\leq\frac{1}{d}\|\bm{\pi}^{(m)}-\bm{\pi}^{*}\|_{\infty}$,
one can apply Lemma \ref{lemma:bernstein} to derive
\begin{equation}
\left|I_{4}^{m}-\EE\big[I_{4}^{m}\mid\cG_{-m},\tilde{\bm{y}}\big]\right|\lesssim\frac{\sqrt{np\log n}+\log n}{d}\|\bm{\pi}^{\left(m\right)}-\bm{\pi}^{*}\|_{\infty}\label{eq:I4-fluctuation}
\end{equation}
with high enough probability. The bounds \eqref{eq:I4-expectation} and \eqref{eq:I4-fluctuation}
taken together complete the proof.

%% file: proof-MLE.tex
\section{Proofs in Section \ref{sec:Analysis-for-MLE}}\label{sec:Proofs-in-Section-MLE}

This section gathers the proofs of the lemmas in Section \ref{sec:Analysis-for-MLE}.

\subsection{Proof of Lemma \ref{lemma:grad-L}\label{sec:Proof-of-Lemma-grad-L}}

Observe that 
\begin{align*}
\nabla\mathcal{L}_{\lambda}\left(\bm{\theta}^{*};\bm{y}\right) & =\lambda\bm{\theta}^{*}+\sum_{(i,j)\in\mathcal{E},i>j}\left(-y_{j,i}+\frac{e^{\theta_{i}^{*}}}{e^{\theta_{i}^{*}}+e^{\theta_{j}^{*}}}\right)\left(\bm{e}_{i}-\bm{e}_{j}\right)\\
 & =\lambda\bm{\theta}^{*}+\frac{1}{L}\sum_{(i,j)\in\mathcal{E},i>j}\sum_{l=1}^{L}\underset{:=\bm{z}_{i,j}^{(l)}}{\underbrace{\left(-y_{j,i}^{(l)}+\frac{e^{\theta_{i}^{*}}}{e^{\theta_{i}^{*}}+e^{\theta_{j}^{*}}}\right)\left(\bm{e}_{i}-\bm{e}_{j}\right)}}.
\end{align*}
It is seen that $\mathbb{E}[\bm{z}_{i,j}^{(l)}]=\bm{0}$, $\|\bm{z}_{i,j}^{(l)}\|_{2}\leq\sqrt{2}$,
\begin{align*}
\mathbb{E}\left[\bm{z}_{i,j}^{(l)}\bm{z}_{i,j}^{(l)\top}\right] & =\mathsf{Var}\left[y_{i,j}^{(l)}\right]\left(\bm{e}_{i}-\bm{e}_{j}\right)\left(\bm{e}_{i}-\bm{e}_{j}\right)^{\top}\preceq\left(\bm{e}_{i}-\bm{e}_{j}\right)\left(\bm{e}_{i}-\bm{e}_{j}\right)^{\top}
\end{align*}
\begin{align*}
\text{and}\qquad\mathbb{E}\left[\bm{z}_{i,j}^{(l)\top}\bm{z}_{i,j}^{(l)}\right] & =\mathrm{Tr}\left(\mathbb{E}\left[\bm{z}_{i,j}^{(l)}\bm{z}_{i,j}^{(l)\top}\right]\right)\leq2.
\end{align*}
This implies that with high probability (note that the randomness
comes from $\mathcal{G}$), 
\begin{align*}
\left\Vert \sum_{(i,j)\in\mathcal{E},i>j}\sum_{l=1}^{L}\mathbb{E}\left[\bm{z}_{i,j}^{(l)}\bm{z}_{i,j}^{(l)\top}\right]\right\Vert  & \leq L\left\Vert \sum_{(i,j)\in\mathcal{E},i>j}\left(\bm{e}_{i}-\bm{e}_{j}\right)\left(\bm{e}_{i}-\bm{e}_{j}\right)^{\top}\right\Vert =L\|\bm{L}_{\mathcal{G}}\|\lesssim Lnp
\end{align*}
and 
\[
\left|\sum_{(i,j)\in\mathcal{E},i>j}\sum_{l=1}^{L}\mathbb{E}\left[\bm{z}_{i,j}^{(l)\top}\bm{z}_{i,j}^{(l)}\right]\right|\leq2L\left|\sum_{(i,j)\in\mathcal{E},i>j}1\right|\lesssim Ln^{2}p.
\]

Letting $V:=\frac{1}{L^{2}}\max\left\{ \left\Vert \sum_{(i,j)\in\mathcal{E}}\sum_{l=1}^{L}\mathbb{E}\left[\bm{z}_{i,j}^{(l)}\bm{z}_{i,j}^{(l)\top}\right]\right\Vert ,\left|\sum_{(i,j)\in\mathcal{E}}\sum_{l=1}^{L}\mathbb{E}\left[\bm{z}_{i,j}^{(l)\top}\bm{z}_{i,j}^{(l)}\right]\right|\right\} $
and $B:=\max_{i,j,l}\|\bm{z}_{i,j}^{(l)}\|$, we can invoke the matrix
Bernstein inequality \cite[Theorem 1.6]{tropp2012user} to reach 
\begin{align*}
\left\Vert \nabla\mathcal{L}_{\lambda}\left(\bm{\theta}^{*};\bm{y}\right)-\mathbb{E}\left[\nabla\mathcal{L}_{\lambda}\left(\bm{\theta}^{*};\bm{y}\right)\mid\mathcal{G}\right]\right\Vert _{2} & \lesssim\sqrt{V\log n}+B\log n\lesssim\sqrt{\frac{n^{2}p\log n}{L}}+\frac{\log n}{L}\\
 & \lesssim\sqrt{\frac{n^{2}p\log n}{L}}
\end{align*}
with probability at least $1-O(n^{-10})$. Combining this with the
identity $\mathbb{E}\left[\nabla\mathcal{L}_{\lambda}(\bm{\theta}^{*};\bm{y})\mid\mathcal{G}\right]=\lambda\bm{\theta}^{*}$
yields 
\begin{align*}
\left\Vert \nabla\mathcal{L}_{\lambda}\left(\bm{\theta}^{*};\bm{y}\right)\right\Vert _{2} & \leq\left\Vert \mathbb{E}\left[\nabla\mathcal{L}_{\lambda}\left(\bm{\theta}^{*};\bm{y}\right)\mid\mathcal{G}\right]\right\Vert _{2}+\left\Vert \nabla\mathcal{L}_{\lambda}\left(\bm{\theta}^{*};\bm{y}\right)-\mathbb{E}\left[\nabla\mathcal{L}_{\lambda}\left(\bm{\theta}^{*};\bm{y}\right)\mid\mathcal{G}\right]\right\Vert _{2}\\
 & \lesssim\lambda\left\Vert \bm{\theta}^{*}\right\Vert _{2}+\sqrt{\frac{n^{2}p\log n}{L}}\asymp\sqrt{\frac{n^{2}p\log n}{L}},
\end{align*}
where the last relation holds because of the facts that $\left\Vert \bm{\theta}^{*}\right\Vert _{2}\leq\sqrt{n}\log\kappa$,
$\lambda\asymp\frac{1}{\log\kappa}\sqrt{\frac{np\log n}{L}}$, and
\[
\lambda\left\Vert \bm{\theta}^{*}\right\Vert _{2}\lesssim\frac{\left\Vert \bm{\theta}^{*}\right\Vert _{2}}{\log\kappa}\sqrt{\frac{np\log n}{L}}\lesssim\sqrt{\frac{n^{2}p\log n}{L}}.
\]
This concludes the proof.

\subsection{Proof of Lemma \ref{lemma:strong-convexity-L}\label{subsec:Proof-of-Lemma-strong-convexity}}

It suffices to prove that 
\[
\min_{1\leq i,j\leq n}\frac{e^{\theta_{i}}e^{\theta_{j}}}{\left(e^{\theta_{i}}+e^{\theta_{j}}\right)^{2}}\geq\frac{1}{4\kappa e^{2C}}
\]
for all $\bm{\theta}\in\RR^{n}$ obeying $\left\Vert \bm{\theta}-\bm{\theta}^{*}\right\Vert _{\infty}\leq C$.
Without loss of generality, suppose $\theta_{i}\leq\theta_{j}$. One
can divide both the denominator and the numerator by $e^{2\theta_{j}}$
to obtain 
\[
\frac{e^{\theta_{i}}e^{\theta_{j}}}{\left(e^{\theta_{i}}+e^{\theta_{j}}\right)^{2}}=\frac{e^{\theta_{i}-\theta_{j}}}{\left(1+e^{\theta_{i}-\theta_{j}}\right)^{2}}=\frac{e^{-|\theta_{i}-\theta_{j}|}}{\left(1+e^{-|\theta_{i}-\theta_{j}|}\right)^{2}}\geq\frac{1}{4}e^{-|\theta_{i}-\theta_{j}|},
\]
where the last relation holds since $\left(1+e^{-x}\right)^{2}\leq4$
for all $x\geq0$. From our assumption, we see that for all $1\leq i,j\leq n$,
\[
\left|\theta_{i}-\theta_{j}\right|\leq\theta_{\max}^{*}-\theta_{\min}^{*}+2\left\Vert \bm{\theta}-\bm{\theta}^{*}\right\Vert _{\infty}\leq\log\kappa+2C,
\]
which relies on the fact that $\theta_{\max}^{*}-\theta_{\min}^{*}\leq\log\kappa$.
This allows one to justify that 
\[
\frac{e^{\theta_{i}}e^{\theta_{j}}}{\left(e^{\theta_{i}}+e^{\theta_{j}}\right)^{2}}\geq\frac{1}{4}e^{-|\theta_{i}-\theta_{j}|}\geq\frac{1}{4\kappa e^{2C}}.
\]

\subsection{Proof of Fact \ref{fact:mean}\label{subsec:Proof-of-fact:mean}}
By $\bm{\theta}^{0}=\bm{\theta}^{*}$ and $\bm{1}^{\top}\bm{\theta}^{*}=0$, the statement trivially holds
	true for $t=0$. Suppose it is true for some $t\geq0$. Then 
	\begin{align*}
	\bm{1}^{\top}\bm{\theta}^{t+1} & =\bm{1}^{\top}(\bm{\theta}^{t}-\eta_{t}\nabla\cL_{\lambda}(\bm{\theta}^{t};\bm{y}))\overset{\left(\text{i}\right)}{=}-\eta_{t}\bm{1}^{\top}\nabla\cL_{\lambda}(\bm{\theta}^{t};\bm{y})\\
	& =-\eta_{t}\bm{1}^{\top}\left(\nabla\cL(\bm{\theta}^{t};\bm{y})+\lambda\bm{\theta}^{t}\right)\overset{\left(\text{ii}\right)}{=}-\eta_{t}\bm{1}^{\top}\nabla\cL(\bm{\theta}^{t};\bm{y})\overset{\left(\text{iii}\right)}{=}0,
	\end{align*}
	where the equalities (i) and (ii) follow from the fact that $\bm{1}^{\top}\bm{\theta}^{t}=0$,
	whereas the last identity (iii) arises from the gradient expression
	(\ref{eq:gradient}) and the simple fact that $\bm{1}^{\top}(\bm{e}_{i}-\bm{e}_{j})=0$
	for any $i$ and $j$. This completes the whole proof.

\subsection{Proof of Lemma \ref{lemma:ell_2_rate}\label{subsec:Proof-of-Lemma-Coarse-ell-2}}

It follows from the optimality of $\bm{\theta}$ as well as the mean
value theorem that 
\begin{align*}
\mathcal{L}_{\lambda}\left(\bm{\theta}^{*};\bm{y}\right)  \geq\mathcal{L}_{\lambda}\left(\bm{\theta};\bm{y}\right)&=\mathcal{L}_{\lambda}\left(\bm{\theta}^{*};\bm{y}\right)+\left\langle \nabla\mathcal{L}_{\lambda}\left(\bm{\theta}^{*};\bm{y}\right),\bm{\theta}-\bm{\theta}^{*}\right\rangle \\
&\quad+\frac{1}{2}\left(\bm{\theta}-\bm{\theta}^{*}\right)^{\top}\nabla^{2}\mathcal{L}_{\lambda}\big(\tilde{\bm{\theta}};\bm{y}\big)\left(\bm{\theta}-\bm{\theta}^{*}\right),
\end{align*}
where $\tilde{\bm{\theta}}$ is between $\bm{\theta}$ and $\bm{\theta}^{*}$.
This together with the Cauchy-Schwarz inequality gives 
\begin{align*}
\frac{1}{2}\left(\bm{\theta}-\bm{\theta}^{*}\right)^{\top}\nabla^{2}\mathcal{L}_{\lambda}\big(\tilde{\bm{\theta}};\bm{y}\big)\left(\bm{\theta}-\bm{\theta}^{*}\right) & \leq-\left\langle \nabla\mathcal{L}_{\lambda}\left(\bm{\theta}^{*};\bm{y}\right),\bm{\theta}-\bm{\theta}^{*}\right\rangle \\
&\leq\left\Vert \nabla\mathcal{L}_{\lambda}\left(\bm{\theta}^{*};\bm{y}\right)\right\Vert _{2}\left\Vert \bm{\theta}-\bm{\theta}^{*}\right\Vert _{2}.
\end{align*}
The above inequality gives 
\begin{equation}
\left\Vert \bm{\theta}-\bm{\theta}^{*}\right\Vert _{2}\leq\frac{2\left\Vert \nabla\mathcal{L}_{\lambda}\left(\bm{\theta}^{*};\bm{y}\right)\right\Vert _{2}}{\lambda_{\min}\left(\nabla^{2}\mathcal{L}_{\lambda}\big(\tilde{\bm{\theta}};\bm{y}\big)\right)}.\label{eq:theta-thetas-dist}
\end{equation}
From the trivial lower bound $\lambda_{\min}\left(\nabla^{2}\mathcal{L}_{\lambda}\left(\tilde{\bm{\theta}};\bm{y}\right)\right)\geq\lambda$,
the preceding inequality gives 
\begin{equation}
\left\Vert \bm{\theta}-\bm{\theta}^{*}\right\Vert _{2}\leq\frac{2\left\Vert \nabla\mathcal{L}_{\lambda}\left(\bm{\theta}^{*};\bm{y}\right)\right\Vert _{2}}{\lambda}.\label{eq:theta-thetas-dist-2}
\end{equation}
On the event $\mathcal{A}_{2}=\left\{ \left\Vert \nabla\mathcal{L}_{\lambda}\left(\bm{\theta}^{*};\bm{y}\right)\right\Vert _{2}\lesssim\sqrt{\frac{n^{2}p\log n}{L}}\right\} $
and in the presence of the choice $\lambda\asymp\frac{1}{\log\kappa}\sqrt{\frac{np\log n}{L}}$,
we obtain $\left\Vert \bm{\theta}-\bm{\theta}^{*}\right\Vert _{2}\leq c_{2}\log\kappa\sqrt{n}$
for some constant $c_{2}>0$.

\subsection{Proof of Lemma \ref{lemma:consequence}\label{subsec:Proof-of-Lemma-consequence}}

In regard to the first consequence, one can apply the triangle inequality
to show 
\begin{align*}
\max_{1\leq m\leq n}\left\Vert \bm{\theta}^{t,\left(m\right)}-\bm{\theta}^{*}\right\Vert _{\infty} & \leq\max_{1\leq m\leq n}\left\Vert \bm{\theta}^{t}-\bm{\theta}^{t,\left(m\right)}\right\Vert _{2}+\left\Vert \bm{\theta}^{t}-\bm{\theta}^{*}\right\Vert _{\infty}\\
 & \leq C_{3}\kappa\sqrt{\frac{\log n}{npL}}+C_{4}\kappa^{2}\sqrt{\frac{\log n}{npL}}\\
 & \leq C_{5}\kappa^{2}\sqrt{\frac{\log n}{npL}},
\end{align*}
as long as $C_{5}\geq C_{4}+C_{3}$. Similarly, for the second one,
we have 
\begin{align*}
\max_{1\leq m\leq n}\left\Vert \bm{\theta}^{t,\left(m\right)}-\bm{\theta}^{*}\right\Vert _{2} & \leq\max_{1\leq m\leq n}\left\Vert \bm{\theta}^{t}-\bm{\theta}^{t,\left(m\right)}\right\Vert _{2}+\left\Vert \bm{\theta}^{t}-\bm{\theta}^{*}\right\Vert _{2}\\
 & \leq C_{3}\kappa\sqrt{\frac{\log n}{npL}}+C_{1}\kappa\sqrt{\frac{\log n}{pL}}\\
 & \leq C_{6}\kappa\sqrt{\frac{\log n}{pL}},
\end{align*}
as soon as $C_{6}\geq C_{1}+C_{3}$.

\subsection{Proof of Lemma \ref{lemma:ell_2_contraction}\label{subsec:Proof-of-Lemma-ell-2-contraction}}

%\yxc{Need to say that $1^{\top}\theta^t=0$ throughout. Maybe add it as another fact or lemma.}
In view of the gradient update rule (\ref{eq:gradient-update-original}),
we have 
\begin{align}
\bm{\theta}^{t+1}-\bm{\theta}^{*} & =\bm{\theta}^{t}-\eta\nabla\mathcal{L}_{\lambda}\left(\bm{\theta}^{t}\right)-\bm{\theta}^{*}\nonumber \\
 & =\bm{\theta}^{t}-\eta\nabla\mathcal{L}_{\lambda}\left(\bm{\theta}^{t}\right)-\left[\bm{\theta}^{*}-\eta\nabla\mathcal{L}_{\lambda}\left(\bm{\theta}^{*}\right)\right]-\eta\nabla\mathcal{L}_{\lambda}\left(\bm{\theta}^{*}\right)\nonumber \\
 & =\left\{ \bm{I}_{n}-\eta\int_{0}^{1}\nabla^{2}\mathcal{L}_{\lambda}\left(\bm{\theta}\left(\tau\right)\right)\mathrm{d}\tau\right\} \left(\bm{\theta}^{t}-\bm{\theta}^{*}\right)-\eta\nabla\mathcal{L}_{\lambda}\left(\bm{\theta}^{*}\right),\label{eq:theta-t+1}
\end{align}
where we denote $\bm{\theta}\left(\tau\right):=\bm{\theta}^{*}+\tau\left(\bm{\theta}^{t}-\bm{\theta}^{*}\right)$.
Here, the last identity results from the fundamental theorem of calculus
\cite[Chapter XIII, Theorem 4.2]{lang1993real}. Let $\theta_{\max}\left(\tau\right):=\max_{i}\theta_{i}(\tau)$
and $\theta_{\min}\left(\tau\right):=\min_{i}\theta_{i}(\tau)$. Combining
the induction hypothesis (\ref{eq:induction-infty}) with the definition
of $\bm{\theta}\left(\tau\right)$, one can see that for all $0\leq\tau\leq1$,
\begin{align*}
\theta_{\max}\left(\tau\right)-\theta_{\min}\left(\tau\right) & \leq\theta_{\max}^{*}-\theta_{\min}^{*}+2\left\Vert \bm{\theta}^{t}-\bm{\theta}^{*}\right\Vert _{\infty}\leq\log\kappa+\epsilon,
\end{align*}
for any sufficiently small $\epsilon>0$, as long as 
\[
2C_{4}\kappa^{2}\sqrt{\frac{\log n}{npL}}\leq\epsilon.
\]
This together with Lemma \ref{lem:smoothness-L} and Corollary \ref{coro:strong-convexity}
reveals that for any $0\leq\tau\leq1$, 
\begin{equation}
\frac{1}{10\kappa}np+\lambda\leq\frac{1}{8\kappa e^{\epsilon}}np+\lambda\leq\lambda_{\min,\perp}\left(\nabla^{2}\mathcal{L}_{\lambda}\left(\bm{\theta}\left(\tau\right)\right)\right)\leq\lambda_{\max}\left(\nabla^{2}\mathcal{L}_{\lambda}\left(\bm{\theta}\left(\tau\right)\right)\right)\leq\lambda+np,\label{subeq:spectral-property}
\end{equation}
where the first inequality holds as long as $\epsilon>0$ is small
enough. Denoting $\bm{A}=\int_{0}^{1}\nabla^{2}\mathcal{L}_{\lambda}\left(\bm{\theta}\left(\tau\right)\right)\mathrm{d}\tau$
and using the triangle inequality, we can derive from (\ref{eq:theta-t+1})
that 
\begin{equation}
\left\Vert \bm{\theta}^{t+1}-\bm{\theta}^{*}\right\Vert _{2}\leq\left\Vert \left(\bm{I}_{n}-\eta\bm{A}\right)\left(\bm{\theta}^{t}-\bm{\theta}^{*}\right)\right\Vert _{2}+\eta\left\Vert \nabla\mathcal{L}_{\lambda}\left(\bm{\theta}^{*}\right)\right\Vert _{2}.\label{eq:theta-t+1-triangle}
\end{equation}
Since $\mathbf{1}^{\top}(\bm{\theta}^{t}-\bm{\theta}^{*})=0$ %and $\eta\|\bm{A}\|_{2}\leq1$
, the first term on the right hand side of (\ref{eq:theta-t+1-triangle})
is controlled by 
\begin{align*}
\left\Vert \left(\bm{I}_{n}-\eta\bm{A}\right)\left(\bm{\theta}^{t}-\bm{\theta}^{*}\right)\right\Vert _{2} & \leq\max\left\{ \left|1-\eta\lambda_{\min,\perp}\left(\bm{A}\right)\right|,\left|1-\eta\lambda_{\max}\left(\bm{A}\right)\right|\right\} \left\Vert \bm{\theta}^{t}-\bm{\theta}^{*}\right\Vert _{2}.
\end{align*}
Substitute (\ref{subeq:spectral-property}) into the above inequality
to reach 
\begin{align}
\left\Vert \left(\bm{I}_{n}-\eta\bm{A}\right)\left(\bm{\theta}^{t}-\bm{\theta}^{*}\right)\right\Vert _{2} & \leq\left(1-\frac{1}{10\kappa}\eta np\right)\left\Vert \bm{\theta}^{t}-\bm{\theta}^{*}\right\Vert _{2}.\label{eq:theta-t+1-first-term}
\end{align}
Substitute (\ref{eq:theta-t+1-first-term}) back to (\ref{eq:theta-t+1-triangle})
and use the induction hypothesis (\ref{eq:induction-infty}) to conclude
that 
\begin{align*}
\left\Vert \bm{\theta}^{t+1}-\bm{\theta}^{*}\right\Vert _{2} & \leq\left(1-\frac{1}{10\kappa}\eta np\right)\left\Vert \bm{\theta}^{t}-\bm{\theta}^{*}\right\Vert _{2}+\eta\left\Vert \nabla\mathcal{L}_{\lambda}\left(\bm{\theta}^{*}\right)\right\Vert _{2}\\
 & \leq\left(1-\frac{1}{10\kappa}\eta np\right)C_{1}\kappa\sqrt{\frac{\log n}{pL}}+C\eta\sqrt{\frac{n^{2}p\log n}{L}}\\
 & \leq C_{1}\kappa\sqrt{\frac{\log n}{pL}}
\end{align*}
for some constants $C,C_{1}>0$. Here, the second inequality makes
use of the facts that $\left\Vert \nabla\mathcal{L}_{\lambda}\left(\bm{\theta}^{*}\right)\right\Vert _{2}\lesssim\sqrt{\frac{n^{2}p\log n}{L}}$
(see Lemma \ref{lemma:grad-L}). The last line holds with the proviso
that $C_{1}>0$ is sufficiently large.

\subsection{Proof of Lemma \ref{lemma:loo-m-entry-contraction}\label{subsec:Proof-of-Lemma-loo-m-entry-contraction}}

Consider any $m$ ($1\leq m\leq n$). According to the gradient update
rule (\ref{eq:gradient-update-loo}), one has 
\begin{align}
\theta_{m}^{t+1,\left(m\right)}-\theta_{m}^{*} & =\theta_{m}^{t,\left(m\right)}-\eta\left[\nabla\mathcal{L}_{\lambda}^{\left(m\right)}\left(\bm{\theta}^{t,\left(m\right)}\right)\right]_{m}-\theta_{m}^{*}\nonumber \\
 & =\theta_{m}^{t,\left(m\right)}-\theta_{m}^{*}-\eta\left[p\sum_{i:i\neq m}\left\{ \frac{e^{\theta_{i}^{*}}}{e^{\theta_{i}^{*}}+e^{\theta_{m}^{*}}}-\frac{e^{\theta_{i}^{t,\left(m\right)}}}{e^{\theta_{i}^{t,\left(m\right)}}+e^{\theta_{m}^{t,\left(m\right)}}}\right\} \right]-\eta\lambda\theta_{m}^{t,\left(m\right)},\label{eq:theta-t+1-m-m}
\end{align}
where the last line follows by the construction of $\mathcal{L}_{\lambda}^{\left(m\right)}$.
Apply the mean value theorem to obtain 
\begin{align}
\frac{e^{\theta_{i}^{*}}}{e^{\theta_{i}^{*}}+e^{\theta_{m}^{*}}}-\frac{e^{\theta_{i}^{t,\left(m\right)}}}{e^{\theta_{i}^{t,\left(m\right)}}+e^{\theta_{m}^{t,\left(m\right)}}} & =\frac{1}{1+e^{\theta_{m}^{*}-\theta_{i}^{*}}}-\frac{1}{1+e^{\theta_{m}^{t,\left(m\right)}-\theta_{i}^{t,\left(m\right)}}}\nonumber \\
 & =-\frac{e^{c_{i}}}{\left(1+e^{c_{i}}\right)^{2}}\left[\theta_{m}^{*}-\theta_{i}^{*}-\left(\theta_{m}^{t,\left(m\right)}-\theta_{i}^{t,\left(m\right)}\right)\right],\label{eq:mean-value}
\end{align}
where $c_{i}$ is some real number lying between $\theta_{m}^{*}-\theta_{i}^{*}$
and $\theta_{m}^{t,\left(m\right)}-\theta_{i}^{t,\left(m\right)}$.
Substituting (\ref{eq:mean-value}) back into (\ref{eq:theta-t+1-m-m})
and rearranging terms yield 
\begin{align*}
\theta_{m}^{t+1,\left(m\right)}-\theta_{m}^{*}&=\left(1-\eta\lambda-\eta p\sum_{i:i\neq m}\frac{e^{c_{i}}}{\left(1+e^{c_{i}}\right)^{2}}\right)\left(\theta_{m}^{t,\left(m\right)}-\theta_{m}^{*}\right) \\
&\quad+\eta p\sum_{i:i\neq m}\frac{e^{c_{i}}}{\left(1+e^{c_{i}}\right)^{2}}\left(\theta_{i}^{t,\left(m\right)}-\theta_{i}^{*}\right)-\eta\lambda\theta_{m}^{*}
\end{align*}
and hence 
\begin{align*}
\left|\theta_{m}^{t+1,\left(m\right)}-\theta_{m}^{*}\right| & \leq\left|1-\eta\lambda-\eta p\sum_{i:i\neq m}\frac{e^{c_{i}}}{\left(1+e^{c_{i}}\right)^{2}}\right|\left|\theta_{m}^{t,\left(m\right)}-\theta_{m}^{*}\right|\\
&\quad+\frac{\eta p}{4}\sum_{i:i\neq m}\left|\theta_{i}^{t,\left(m\right)}-\theta_{i}^{*}\right|+\eta\lambda\left\Vert \bm{\theta}^{*}\right\Vert _{\infty}\\
 & \leq\left|1-\eta\lambda-\eta p\sum_{i:i\neq m}\frac{e^{c_{i}}}{\left(1+e^{c_{i}}\right)^{2}}\right|\left|\theta_{m}^{t,\left(m\right)}-\theta_{m}^{*}\right|\\
 &\quad+\frac{\eta p}{4}\sqrt{n}\left\Vert \bm{\theta}^{t,\left(m\right)}-\bm{\theta}^{*}\right\Vert _{2}+\eta\lambda\left\Vert \bm{\theta}^{*}\right\Vert _{\infty}.
\end{align*}
Here, the first inequality comes from the triangle inequality and the elementary
inequality $\frac{e^{c}}{\left(1+e^{c}\right)^{2}}\leq\frac{1}{4}$
for any $c\in\RR$, whereas the second relation holds owing to the
Cauchy-Schwarz inequality, namely 
\[
\sum_{i:i\neq m}\left|\theta_{i}^{t,\left(m\right)}-\theta_{i}^{*}\right|\leq\sqrt{n}\left\Vert \bm{\theta}^{t,\left(m\right)}-\bm{\theta}^{*}\right\Vert _{2}.
\]
From $\frac{e^{c}}{(1+e^{c})^{2}}\leq\frac{1}{4}$ we also obtain
that 
\[
1-\eta\lambda-\eta p\sum_{i:i\neq m}\frac{e^{c_{i}}}{\left(1+e^{c_{i}}\right)^{2}}\geq1-\eta\lambda-\eta p\frac{n}{4}\geq1-\eta(np+\lambda)\geq0
\]
and 
\begin{align*}
\left|1-\eta\lambda-\eta p\sum_{i:i\neq m}\frac{e^{c_{i}}}{\left(1+e^{c_{i}}\right)^{2}}\right|&=1-\eta\lambda-\eta p\sum_{i:i\neq m}\frac{e^{c_{i}}}{\left(1+e^{c_{i}}\right)^{2}}\\%\leq1-\etap\sum_{i:i\neqm}\frac{e^{c_{i}}}{\left(1+e^{c_{i}}\right)^{2}}
&\leq1-\eta p(n-1)\min_{i:i\neq m}\frac{e^{c_{i}}}{\left(1+e^{c_{i}}\right)^{2}}.
\end{align*}
%The last inequality is due to $\PP ( d_{\min} \leq np/2  )= 1-O(n^{-10})$.

To further upper bound $\left|\theta_{m}^{t+1,\left(m\right)}-\theta_{m}^{*}\right|$,
it suffices to obtain a lower bound on $\frac{e^{c_{i}}}{\left(1+e^{c_{i}}\right)^{2}}$.
Toward this, it is easy to see from (\ref{eq:consequence-theta-t-m-infty})
that 
\begin{align*}
\max_{i:i\neq m}\left|c_{i}\right| & \leq\max_{i:i\neq m}\left|\theta_{m}^{*}-\theta_{i}^{*}\right|+\max_{i:i\neq m}\left|\theta_{m}^{*}-\theta_{i}^{*}-\left(\theta_{m}^{t,\left(m\right)}-\theta_{i}^{t,\left(m\right)}\right)\right|\\
 & \leq\log\kappa+2\left\Vert \bm{\theta}^{t,\left(m\right)}-\bm{\theta}^{*}\right\Vert _{\infty}\\
 & \leq\log\kappa+\epsilon,
\end{align*}
as long as 
\[
2C_{5}\kappa^{2}\sqrt{\frac{\log n}{npL}}\leq\epsilon.
\]
This further reveals that for $\epsilon>0$ small enough, one has
\begin{align*}
\frac{e^{c_{i}}}{\left(1+e^{c_{i}}\right)^{2}} & =\frac{e^{-|c_{i}|}}{\left(1+e^{-|c_{i}|}\right)^{2}}\geq\frac{e^{-|c_{i}|}}{4}\geq\frac{1}{4e^{\epsilon}\kappa}\geq\frac{1}{5\kappa}.
\end{align*}
Taking the previous bounds collectively, we arrive at 
\begin{align*}
\left|\theta_{m}^{t+1,\left(m\right)}-\theta_{m}^{*}\right| & \leq\left(1-\frac{1}{10\kappa}\eta np\right)\left|\theta_{m}^{t,\left(m\right)}-\theta_{m}^{*}\right|+\frac{\eta p\sqrt{n}}{4}\left\Vert \bm{\theta}^{t,\left(m\right)}-\bm{\theta}^{*}\right\Vert _{2}+\eta\lambda\left\Vert \bm{\theta}^{*}\right\Vert _{\infty}\\
 & \leq\left(1-\frac{1}{10\kappa}\eta np\right)C_{2}\kappa^{2}\sqrt{\frac{\log n}{npL}}+\frac{\eta p\sqrt{n}}{4}C_{6}\kappa\sqrt{\frac{\log n}{pL}}+c_{\lambda}\eta\sqrt{\frac{np\log n}{L}}\\
 & \leq C_{2}\kappa^{2}\sqrt{\frac{\log n}{npL}},
\end{align*}
as long as $C_{2}\gg\max\left\{ C_{6},c_{\lambda}\right\} $. Here
the second line comes from the setting of $\lambda$, namely, 
\[
\lambda\left\Vert \bm{\theta}^{*}\right\Vert _{\infty}=c_{\lambda}\frac{\left\Vert \bm{\theta}^{*}\right\Vert _{\infty}}{\log\kappa}\sqrt{\frac{np\log n}{L}}\leq c_{\lambda}\sqrt{\frac{np\log n}{L}}
\]
since $\left\Vert \bm{\theta}^{*}\right\Vert _{\infty}\leq\log\kappa$.

\subsection{Proof of Lemma \ref{lemma:loo-perturbation-contraction}\label{subsec:Proof-of-Lemma-loo-perturbation-contraction}}

Consider any $1\leq m\leq n$. Apply the update rules (\ref{eq:gradient-update-original})
and (\ref{eq:gradient-update-loo}) to obtain 
\begin{align*}
\bm{\theta}^{t+1}-\bm{\theta}^{t+1,\left(m\right)} & =\bm{\theta}^{t}-\eta\nabla\mathcal{L}_{\lambda}\left(\bm{\theta}^{t}\right)-\left[\bm{\theta}^{t,\left(m\right)}-\eta\nabla\mathcal{L}_{\lambda}^{\left(m\right)}\left(\bm{\theta}^{t,\left(m\right)}\right)\right]\\
 & =\bm{\theta}^{t}-\eta\nabla\mathcal{L}_{\lambda}\left(\bm{\theta}^{t}\right)-\left[\bm{\theta}^{t,\left(m\right)}-\eta\nabla\mathcal{L}_{\lambda}\left(\bm{\theta}^{t,\left(m\right)}\right)\right] \\
 &\quad-\eta\left(\nabla\mathcal{L}_{\lambda}\left(\bm{\theta}^{t,\left(m\right)}\right)-\nabla\mathcal{L}_{\lambda}^{\left(m\right)}\left(\bm{\theta}^{t,\left(m\right)}\right)\right)\\
 & =\underbrace{\left(\bm{I}_{n}-\eta\int_{0}^{1}\nabla^{2}\mathcal{L}_{\lambda}\left(\bm{\theta}\left(\tau\right)\right)\mathrm{d}\tau\right)\left(\bm{\theta}^{t}-\bm{\theta}^{t,\left(m\right)}\right)}_{:=\bm{v}_{1}}\\
 &\quad-\underbrace{\eta\left(\nabla\mathcal{L}_{\lambda}\left(\bm{\theta}^{t,\left(m\right)}\right)-\nabla\mathcal{L}_{\lambda}^{\left(m\right)}\left(\bm{\theta}^{t,\left(m\right)}\right)\right)}_{:=\bm{v}_{2}},
\end{align*}
where we abuse the notation and denote $\bm{\theta}\left(\tau\right)=\bm{\theta}^{t,\left(m\right)}+\tau\left(\bm{\theta}^{t}-\bm{\theta}^{t,\left(m\right)}\right)$,
and the last identity results from the fundamental theorem of calculus
\cite[Chapter XIII, Theorem 4.2]{lang1993real}. In what follows,
we control $\bm{v}_{1}$ and $\bm{v}_{2}$ separately. 
\begin{itemize}
\item Regarding the term $\bm{v}_{1}$, repeating the same argument as in
Appendix \ref{subsec:Proof-of-Lemma-ell-2-contraction} yields 
\[
\left\Vert \bm{v}_{1}\right\Vert _{2}\leq\left(1-\frac{1}{10\kappa}\eta np\right)\left\Vert \bm{\theta}^{t}-\bm{\theta}^{t,\left(m\right)}\right\Vert _{2}
\]
as long as $\eta\leq%\frac{1}{10\kappa}
\frac{1}{\lambda+np}$. 
\item When it comes to $\bm{v}_{2}$, one can use the gradient definitions
to reach 
\begin{align}
\frac{1}{\eta}\bm{v}_{2} & =\sum_{i:i\neq m}\left\{ \left(-y_{m,i}+\frac{e^{\theta_{i}^{t,\left(m\right)}}}{e^{\theta_{i}^{t,\left(m\right)}}+e^{\theta_{m}^{t,\left(m\right)}}}\right)\ind_{\left\{ \left(i,m\right)\in\mathcal{E}\right\} }-p\left(-y_{m,i}^{*}+\frac{e^{\theta_{i}^{t,\left(m\right)}}}{e^{\theta_{i}^{t,\left(m\right)}}+e^{\theta_{m}^{t,\left(m\right)}}}\right)\right\} \left(\bm{e}_{i}-\bm{e}_{m}\right)\nonumber \\
 & =\underset{:=\bm{u}^{m}}{\underbrace{\frac{1}{L}\sum_{i:(i,m)\in\mathcal{E}}\sum_{l=1}^{L}\left(-y_{m,i}^{(l)}+\frac{e^{\theta_{i}^{*}}}{e^{\theta_{i}^{*}}+e^{\theta_{m}^{*}}}\right)\left(\bm{e}_{i}-\bm{e}_{m}\right)}}\nonumber \\
 & \quad+\underset{:=\bm{v}^{m}}{\underbrace{\sum_{i:i\neq m}\left\{ \left(-\frac{e^{\theta_{i}^{*}}}{e^{\theta_{i}^{*}}+e^{\theta_{m}^{*}}}+\frac{e^{\theta_{i}^{t,\left(m\right)}}}{e^{\theta_{i}^{t,\left(m\right)}}+e^{\theta_{m}^{t,\left(m\right)}}}\right)\left(\ind_{\left\{ (i,m)\in\mathcal{E}\right\} }-p\right)\right\} \left(\bm{e}_{i}-\bm{e}_{m}\right)}}.\label{eq:nabla-L-first-line}
\end{align}
In the sequel, we control the two terms of (\ref{eq:nabla-L-first-line})
separately. 
\begin{itemize}
\item For the first term $\bm{u}^{m}$ in (\ref{eq:nabla-L-first-line}),
we make the observation that 
\[
u_{i}^{m}=\begin{cases}
\frac{1}{L}\sum_{l=1}^{L}\left(-y_{m,i}^{(l)}+\frac{e^{\theta_{i}^{*}}}{e^{\theta_{i}^{*}}+e^{\theta_{m}^{*}}}\right),\quad & \text{if }(i,m)\in\mathcal{E};\\
\frac{1}{L}\sum_{i:(i,m)\in\mathcal{E}}\sum_{l=1}^{L}\left(y_{m,i}^{(l)}-\frac{e^{\theta_{i}^{*}}}{e^{\theta_{i}^{*}}+e^{\theta_{m}^{*}}}\right),\quad & \text{if }i=m;\\
0, & \text{else}.
\end{cases}
\]
Since $\left|y_{m,i}^{(l)}-\frac{e^{\theta_{i}^{*}}}{e^{\theta_{i}^{*}}+e^{\theta_{m}^{*}}}\right|\leq1$
and $\mathrm{card}\left(\left\{ i:(i,m)\in\mathcal{E}\right\} \right)\asymp np$,
we can apply Hoeffding's inequality and union bounds to get for all
$1\leq m\leq n$, 
\begin{align*}
\left|u_{m}^{m}\right|\lesssim\sqrt{\frac{np\log n}{L}}\qquad\text{and}\qquad & \left|u_{i}^{m}\right|\lesssim\sqrt{\frac{\log n}{L}}\quad\forall i\text{ obeying }(i,m)\in\mathcal{E},
\end{align*}
which further gives 
\[
\|\bm{u}^{m}\|_{2}\leq\left|u_{m}^{m}\right|+\sqrt{\sum_{i:(i,m)\in\mathcal{E}}(u_{i}^{m})^{2}}\lesssim\sqrt{\frac{np\log n}{L}},~\forall1\leq m\leq n.
\]
\item We then turn to the second term $\bm{v}^{m}$ of (\ref{eq:nabla-L-first-line}).
This is a zero-mean random vector that satisfies 
\[
v_{i}^{m}=\begin{cases}
\xi_{i}\left(1-p\right),\quad & \text{if }(i,m)\in\mathcal{E},\\
-\sum_{i:i\neq m}\xi_{i}\left(\ind_{\left\{ (i,m)\in\mathcal{E}\right\} }-p\right),\quad & \text{if }i=m,\\
-\xi_{i}p, & \text{else},
\end{cases}
\]
where 
\[
\xi_{i}:=-\frac{e^{\theta_{i}^{*}}}{e^{\theta_{i}^{*}}+e^{\theta_{m}^{*}}}+\frac{e^{\theta_{i}^{t,\left(m\right)}}}{e^{\theta_{i}^{t,\left(m\right)}}+e^{\theta_{m}^{t,\left(m\right)}}}=-\frac{1}{1+e^{\theta_{m}^{*}-\theta_{i}^{*}}}+\frac{1}{1+e^{\theta_{m}^{t,\left(m\right)}-\theta_{i}^{t,\left(m\right)}}}.
\]
The first step is to bound the size of the coefficient $\xi_{i}$.
Define $g(x)=(1+e^{x})^{-1}$ for $x\in\mathbb{R}$. We have $|g'(x)|\leq1$
and thus 
\begin{align*}
|\xi_{i}|&=\left|g\left(\theta_{m}^{t,\left(m\right)}-\theta_{i}^{t,\left(m\right)}\right)-g\left(\theta_{m}^{*}-\theta_{i}^{*}\right)\right|\\
&\leq\left|\left(\theta_{m}^{t,\left(m\right)}-\theta_{i}^{t,\left(m\right)}\right)-\left(\theta_{m}^{*}-\theta_{i}^{*}\right)\right|\leq\left|\theta_{i}^{*}-\theta_{i}^{t,(m)}\right|+\left|\theta_{m}^{*}-\theta_{m}^{t,(m)}\right|.
\end{align*}
This indicates that 
\[
\left|\xi_{i}\right|\leq2\left\Vert \bm{\theta}^{t,(m)}-\bm{\theta}^{*}\right\Vert _{\infty}\qquad\text{and}\qquad\sum_{i=1}^{n}\xi_{i}^{2}\leq4n\left\Vert \bm{\theta}^{t,(m)}-\bm{\theta}^{*}\right\Vert _{\infty}^{2}.
\]
Applying the Bernstein inequality in Lemma \ref{lemma:bernstein}
we obtain 
\begin{align*}
\left|v_{m}^{m}\right|\text{ }&\lesssim\text{ }\sqrt{\left(p\sum_{i=1}^{n}\xi_{i}^{2}\right)\log n}+\max_{1\leq i\leq n}\left|\xi_{i}\right|\log n\text{ }\\
&\lesssim\text{ }\left(\sqrt{np\log n}+\log n\right)\left\Vert \bm{\theta}^{t,(m)}-\bm{\theta}^{*}\right\Vert _{\infty}
\end{align*}
with high probability. As a consequence, 
\begin{align*}
\|\bm{v}^{m}\|_{2} & \leq|v_{m}^{m}|+\sqrt{\sum_{i:(i,m)\in\mathcal{E}}(v_{i}^{m})^{2}}+\sqrt{\sum_{i:(i,m)\notin\mathcal{E}\text{ and }i\neq m}(v_{i}^{m})^{2}}\\
 & \lesssim\left(\sqrt{np\log n}+\log n\right)\left\Vert \bm{\theta}^{t,(m)}-\bm{\theta}^{*}\right\Vert _{\infty}\\
 &\quad+\sqrt{np}\left\Vert \bm{\theta}^{t,(m)}-\bm{\theta}^{*}\right\Vert _{\infty}+p\sqrt{n}\left\Vert \bm{\theta}^{t,(m)}-\bm{\theta}^{*}\right\Vert _{\infty}\\
 & \lesssim\left(\sqrt{np\log n}+\log n\right)\left\Vert \bm{\theta}^{t,(m)}-\bm{\theta}^{*}\right\Vert _{\infty}\lesssim\left(\sqrt{np\log n}\right)\left\Vert \bm{\theta}^{t,(m)}-\bm{\theta}^{*}\right\Vert _{\infty},
\end{align*}
as long as $np\gg\log n$. 
\end{itemize}
Putting the above results together, we see that

\[
\left\Vert \bm{v}_{2}\right\Vert _{2}\lesssim\eta\left(\sqrt{\frac{np\log n}{L}}+\sqrt{np\log n}\left\Vert \bm{\theta}^{t,\left(m\right)}-\bm{\theta}^{*}\right\Vert _{\infty}\right).
\]
\item Combine the above two bounds to deduce for some $C>0$ 
\begin{align*}
\left\Vert \bm{\theta}^{t+1}-\bm{\theta}^{t+1,\left(m\right)}\right\Vert _{2} & \leq\left(1-\frac{1}{10\kappa}\eta np\right)\left\Vert \bm{\theta}^{t}-\bm{\theta}^{t,\left(m\right)}\right\Vert _{2} \\
&\quad+C\eta\left(\sqrt{\frac{np\log n}{L}}+\sqrt{np\log n}\left\Vert \bm{\theta}^{t,\left(m\right)}-\bm{\theta}^{*}\right\Vert _{\infty}\right)\\
 & \leq\left(1-\frac{1}{10\kappa}\eta np\right)C_{3}\kappa\sqrt{\frac{\log n}{npL}} \\
 &\quad+C\eta\left(\sqrt{\frac{np\log n}{L}}+\sqrt{np\log n}C_{5}\kappa^{2}\sqrt{\frac{\log n}{npL}}\right)\\
 & \leq C_{3}\kappa\sqrt{\frac{\log n}{npL}},
\end{align*}
as soon as $C_{3}$ is sufficiently large and 
\[
\kappa^{2}\sqrt{\frac{\log n}{np}}\ll1.
\]
\end{itemize}

\subsection{Proof of Lemma \ref{lemma:error-t}\label{subsec:Proof-of-lemma:error-t}}
Consider any $m$ ($1\leq m\leq n$). It is easily seen from the triangle inequality
that 
\begin{align*}
\left|\theta_{m}^{t+1}-\theta_{m}^{*}\right| & \leq\left|\theta_{m}^{t+1}-\theta_{m}^{t+1,\left(m\right)}\right|+\left|\theta_{m}^{t+1,\left(m\right)}-\theta_{m}^{*}\right|\\
& \leq\big\Vert \bm{\theta}^{t+1}-\bm{\theta}^{t+1,\left(m\right)}\big\Vert _{2}+\left|\theta_{m}^{t+1,\left(m\right)}-\theta_{m}^{*}\right|\\
& \leq C_{3}\kappa\sqrt{\frac{\log n}{npL}}+C_{2}\kappa^{2}\sqrt{\frac{\log n}{npL}}\\
& \leq C_{4}\kappa^{2}\sqrt{\frac{\log n}{npL}},
\end{align*}
with the proviso that $C_{4}\geq C_{3}+C_{2}$.